\documentclass{article}

\usepackage[eandd, final]{neurips_2026} 

\usepackage[utf8]{inputenc} % allow utf-8 input
\usepackage[T1]{fontenc}    % use 8-bit T1 fonts
\usepackage{hyperref}       % hyperlinks
\usepackage{url}            % simple URL typesetting
\usepackage{booktabs}       % professional-quality tables
\usepackage{amsfonts}       % blackboard math symbols
\usepackage{nicefrac}       % compact symbols for 1/2, etc.
\usepackage{microtype}      % microtypography
\usepackage{xcolor}         % colors

\usepackage{algorithm}
\usepackage{algpseudocode} % from algorithmicx

\usepackage{wrapfig}
\usepackage{caption}
\usepackage{enumitem} 
\usepackage{subcaption}

\usepackage{amsmath}
\usepackage{amssymb}
\usepackage{mathtools}
\usepackage{amsthm}

\usepackage{booktabs}
\usepackage{multirow}
\usepackage{xcolor}
\usepackage{graphicx}
\usepackage[most]{tcolorbox}

\usepackage{titletoc}

\newtcolorbox{takeawaybox}{
  colback=gray!8,
  colframe=black!70,
  boxrule=0.6pt,
  arc=2pt,
  left=6pt,right=6pt,top=5pt,bottom=5pt
}
\theoremstyle{plain}
\newtheorem{theorem}{Theorem}[section]
\newtheorem*{theorem*}{Theorem}
\newtheorem{proposition}[theorem]{Proposition}
\newtheorem*{proposition*}{Proposition}
\newtheorem{lemma}[theorem]{Lemma}

\theoremstyle{definition}
\newtheorem{definition}[theorem]{Definition}

\theoremstyle{remark}
\newtheorem{remark}[theorem]{Remark}

\title{Auditing Conformal Prediction for Language Models Under Inference-Time Model Shift}

\author{%
  \begin{tabular}{c@{\hspace{3em}}c}
    \textbf{Achref Doula}\textsuperscript{1,2,}\thanks{%
        Corresponding author: \texttt{achref.doula@gmail.com} \\ \emph{Author preprint} \\ \emph{Submitted 01.05.2026, accepted 24.09.2026,}} &
    \textbf{Max Mühlhäuser}\textsuperscript{1,2}
  \end{tabular} \\[1ex]
  \textsuperscript{1}Tongji University, Shanghai, China \\
  \textsuperscript{2}Technical University of Darmstadt, Darmstadt, Germany
}

\begin{document}

\maketitle

\begin{abstract}
Conformal prediction (CP) for language models is typically evaluated under a \emph{matched} protocol, in which calibration and deployment use the same post-training inference configuration, such as quantization mode and inference temperature. In deployment, these choices often differ from those used at calibration to meet memory or throughput constraints or serving policies. We argue that this exposes a limitation in current CP evaluation protocols for language models. We formalize the underlying phenomenon as \emph{inference-time model shift}: a change in the conformal score distribution induced by inference configuration rather than by the input--label distribution, and theoretically characterize how it breaks split-conformal exchangeability even when the task distribution and trained model are unchanged. We then evaluate its consequences across $8$ language models, $4$ finite-label benchmarks, and $12$ inference configurations. We find that matched evaluation can substantially overstate deployment-time coverage, and that the effect is heterogeneous across models and datasets. We argue that inference configuration should be reported as a first-class experimental axis in CP evaluations for language-model systems.
\end{abstract}

\section{Introduction}

Conformal prediction (CP) is a distribution-free framework for uncertainty quantification: given a trained predictor and a held-out calibration set, it returns a prediction set guaranteed to contain the true label with a user-specified probability $1-\alpha$ in finite samples~\cite{vovk2005algorithmic,angelopoulos2021gentle}. The guarantee is obtained by calibrating a threshold on nonconformity scores, which are real-valued scores of disagreement between the predictor's output and a candidate label, and rests on an important assumption: that calibration and test scores are exchangeable. A substantial literature studies how this assumption fails when the data-generating distribution shifts, including covariate shift, label shift, online drift, and other forms of non-exchangeability~\cite{tibshirani2019covariate,gibbs2021adaptive,oliveira2024split}. In each case, the source of failure is a change in the joint distribution $P_{XY}$ between calibration and deployment.

For language models, calibration and test scores depend not only on $P_{XY}$ but also on the \emph{inference configuration} used to generate the predictive distribution from which scores are computed: post-training choices that shape the predictive distribution without altering the trained model itself, including numerical representation (e.g., quantization mode) and inference parameters (e.g., temperature). In deployed systems, the inference configuration is not always held fixed for reasons unrelated to the data or the model: quantization is changed to satisfy memory or throughput constraints, and inference temperature varies with system defaults, serving policies, or user-facing controls. Existing CP evaluations for language models, however, typically adopt a \emph{matched} protocol, in which calibration and test points are generated under the same inference configuration; reported coverage then estimates the reliability of one fixed score-generating procedure, a valid target only when the deployment configuration is fixed and known in advance~\cite{kumar2023conformal, quach2024clm,mohri2024factuality, cherian2024llmvalidity, su2024api, ulmer2024nonexchangeable, gui2024alignment, zhong2025quantized}.

When the calibration and deployment configurations differ, even with the trained model and the task distribution unchanged, calibration and test scores are generated by different mechanisms, and the exchangeability assumption underlying split-conformal validity need not hold. We refer to this phenomenon as \emph{inference-time model shift}: a change in the conformal score distribution induced by inference configuration rather than by the input--label distribution. It is distinct from classical dataset shift because $P_{XY}$ is fixed; the shift lives in the mapping from the input to the resulting predictive distribution, and hence in the induced score distribution.

\begin{wrapfigure}{r}{0.55\textwidth}
    \centering
    \vspace{-0.6em}
    \includegraphics[width=\linewidth]{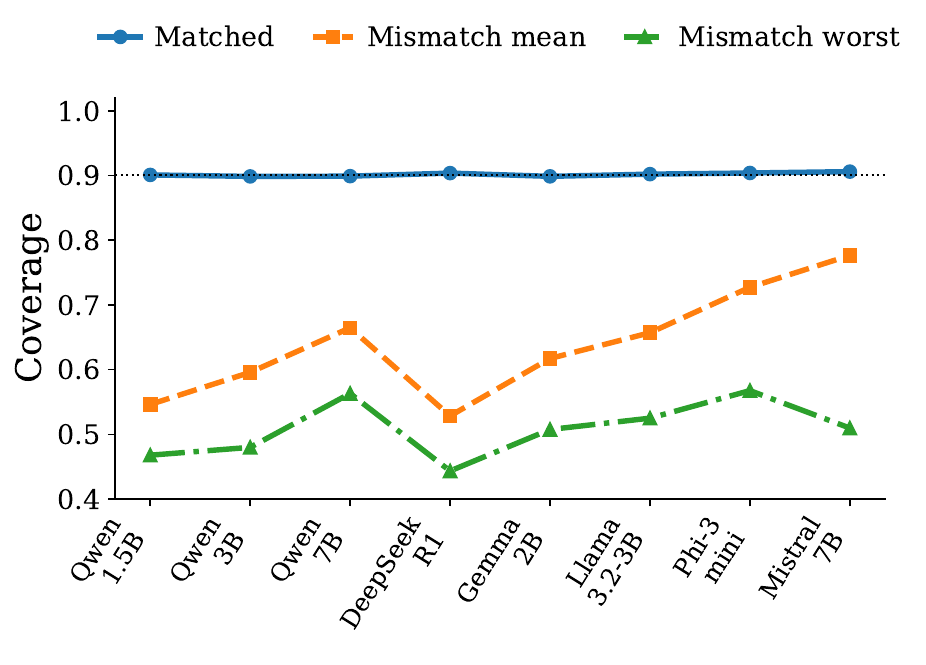}
    \caption{
    \textsc{MMLU-Pro} with deployment fixed to \(\texttt{nf4}@T{=}0.5\) and nominal target $0.9$. Matched calibration tracks target across eight model families; mismatched calibration, differing only in inference configuration between calibration and deployment, produces substantial coverage degradation.
    }
    \label{fig:motivating_case}
    \vspace{-0.8em}
\end{wrapfigure}

The empirical consequences of \emph{inference-time model shift} are nontrivial. Figure~\ref{fig:motivating_case} reports coverage on \textsc{MMLU-Pro} when deployment is fixed to a single configuration ($\texttt{nf4}$ quantization, temperature $T{=}0.5$) and calibration uses, in turn, each of the twelve configurations in a $3{\times}4$ grid of quantization modes and temperatures. Matched calibration tracks the nominal $0.90$ target across eight small language models, with mean coverage $0.902$. Mismatched calibration produces mean coverage $0.639$ and worst-case coverage $0.508$. The task distribution and the trained model are unchanged across all twelve calibrations; the degradation is attributable solely to the change in score-generating procedure between calibration and deployment. Matched evaluation, by construction, cannot surface this degradation, which exposes a limitation of matched-only CP evaluation protocols for language models.

We characterize this setting theoretically and audit it empirically. We formalize inference-time model shift as a property of the conformal score distribution, and we show that matched validity does not transport across inference configurations once score-level exchangeability is broken. The empirical audit covers $8$ language models on $4$ finite-label benchmarks (multiple-choice QA and classification) across a $12$-configuration grid. We restrict the empirical scope to discrete-label tasks because they admit a controlled comparison: candidate labels, predictive distributions, and conformal scores are directly comparable across inference configurations, isolating inference-configuration mismatch from confounds introduced by open-ended generation.

Two findings emerge. First, matched evaluation can substantially overstate deployment-time coverage, with the failures concentrated in low-temperature regimes and amplified, but not solely caused, by aggressive quantization. Second, and more consequential for evaluation practice, the effect is \emph{heterogeneous}: the same type of calibration--deployment mismatch produces severe undercoverage in one setting and mild degradation, or even overcoverage, in another. This implies that inference configuration is an experimental axis on the same footing as the dataset and the model.

\noindent \textbf{Contributions.} We make the following contributions.
\begin{itemize}[leftmargin=*]
    \item We identify a limitation in current CP evaluation protocols for language models and formalize the underlying phenomenon as \emph{inference-time model shift}, a distinct source of non-exchangeability characterized by a change in the score-generating procedure under a fixed task distribution, under which matched validity is not transportable across inference configurations.
    \item We provide an empirical audit across $8$ language models, $4$ finite-label benchmarks, and $12$ inference configurations, finding that mismatch effects are large but neither monotone in configuration nor consistent across models and datasets, and propose an evaluation protocol that reports matched validity alongside mismatch sensitivity and deployment-assumption-specific baselines.
\end{itemize}
\section{Related Work}

\paragraph{Conformal prediction under shift and non-exchangeability.}
CP provides distribution-free uncertainty quantification under exchangeability, with modern developments covering regression, classification, adaptive coverage, and broader risk-controlled prediction sets~\cite{doula2024conformal, lei2018distribution,sadinle2019least,romano2020classification,angelopoulos2021uncertainty,angelopoulos2024risk}. A substantial literature studies what happens when the matched i.i.d.\ setting fails, including covariate shift, label shift, online distribution drift, and other forms of non-exchangeability~\cite{doula2026conformal,tibshirani2019covariate,gibbs2021adaptive,gibbs2024online,podkopaev2021label,prinster2022jaws,oliveira2024split,kasa2025adapting}. These works are closest in spirit to ours because they analyze when standard conformal guarantees degrade outside matched calibration and test conditions. However, their source of failure is typically a change in the data-generating distribution or dependence structure. In contrast, we keep the task distribution fixed and study a different source of non-exchangeability: deployment-time changes in inference configuration. The same checkpoint and the same inputs can therefore induce different conformal score distributions under different temperatures or quantization modes.

\paragraph{Conformal prediction and uncertainty for language models.}
Recent work has brought CP to language and foundation models, including multiple-choice
question answering~\cite{kumar2023conformal}, conformal language modeling for generated
candidate sets~\cite{quach2024clm}, conformal factuality and claim filtering~\cite{mohri2024factuality,cherian2024llmvalidity},
black-box or API-only settings~\cite{su2024api}, correctness-coverage guarantees for open-ended
generation~\cite{wang2024conu}, non-exchangeable conformal language generation~\cite{ulmer2024nonexchangeable},
selective uncertainty testing~\cite{wang2025sconu}, and trust- or alignment-oriented selection for
foundation-model outputs~\cite{gui2024alignment}. These works establish CP as a promising tool
for language-model reliability, but their evaluation targets differ from ours: they typically calibrate
and evaluate under a fixed inference setup, or study other sources of uncertainty and
non-exchangeability. Our focus is the evaluation protocol itself: whether a conformal predictor
calibrated under one post-training inference configuration remains reliable when deployment uses
another. Appendix~\ref{app:prior-audit} provides a more detailed summary of this distinction.

\paragraph{Calibration, temperature, and quantized language models.}
A parallel literature studies calibration and uncertainty estimation for neural networks and language
models. Temperature scaling remains a standard post-hoc calibration baseline~\cite{guo2017calibration},
and calibration is also central for pretrained transformers and question-answering language models~\cite{desai2020calibration,jiang2021know}.
Recent work further studies uncertainty estimation for free-form or long-form LLM generations~\cite{ulmer2024generations,huang2024longform}.
On the systems side, quantization has become a standard deployment mechanism for LLM serving,
with work on low-bit inference and quantization-aware training~\cite{lin2024awq,liu2024llmqat};
recent work also shows that quantization can materially affect calibration itself~\cite{zhong2025quantized}.
This literature motivates temperature and quantization as practically relevant inference axes, but it
does not formulate calibration-to-deployment inference mismatch as a distinct conformal-validity
problem. Our contribution is to connect these deployment choices to exchangeability of conformal
scores and to treat inference configuration as part of the CP evaluation design.

%%%%%%%%%%%%%%%%%%%%%%%%%%%%%
%%%%%%%%%%%%%%%%%%%%%%%%%%%%

\section{A Theoretical Characterization of Inference-Time Model Shift}
\label{sec:shift}

Standard matched CP evaluation estimates reliability for a fixed inference pipeline: calibration and deployment use the same score-generating procedure. In deployed language-model systems, however, the inference configuration may change after calibration. We formalize this mismatch and show why matched validity need not describe deployment-time reliability even when the task distribution is unchanged.
We consider classification and multiple-choice tasks with finite candidate label set $\mathcal{Y}(x)$. Let $(X,Y)\sim P_{XY}$ denote the task distribution, assumed unchanged between calibration and deployment. At inference time, predictions are generated under a configuration $C\in\mathcal{C}$, where $\mathcal{C}=\mathcal{Q}\times\mathcal{T}$ is the product of a finite set of quantization modes $\mathcal{Q}$ and temperatures $\mathcal{T}$. For each $c\in\mathcal{C}$, the model induces a predictive distribution $\pi^{(c)}(x)\in\Delta^{|\mathcal{Y}(x)|}$. Given a nonconformity score $A:\Delta^{|\mathcal{Y}(x)|}\times\mathcal{Y}(x)\to\mathbb{R}$, the score induced under configuration $c$ is $V^{(c)}:=A(\pi^{(c)}(X),Y)$, and we write $F_c$ for the distribution of $V^{(c)}$.

\begin{definition}[Inference-time model shift]
Inference-time model shift occurs if there exist $c,c'\in\mathcal{C}$ such that $F_c\neq F_{c'}$. Equivalently, the conformal score distribution changes with inference configuration even though the task distribution $P_{XY}$ is fixed.
\end{definition}

This setting is distinct from classical dataset shift. The input distribution is fixed. The shift arises from the mechanism that maps an input $x$ to a predictive distribution $\pi^{(c)}(x)$, and hence to a conformal score. At calibration time, we assume access to tuples $(X_i,Y_i,C_i,\pi_i)$ with $\pi_i=\pi^{(C_i)}(X_i)$, so that the score-generating configuration is observed. At test time, the conformal layer receives only a predictive vector $\pi=\pi^{(C)}(X)$, while the underlying configuration $C$ may be latent.

To make the evaluation target explicit, let $\widehat S_{c_{\mathrm{cal}}}$ denote a conformal predictor calibrated under configuration $c_{\mathrm{cal}}$, and define
\[
\mathrm{Cov}(c_{\mathrm{cal}}\!\to\! c_{\mathrm{dep}})
:=\mathbb{P}\!\left(Y\in \widehat S_{c_{\mathrm{cal}}}(X;\pi^{(c_{\mathrm{dep}})}(X))\right).
\]
Matched evaluation estimates $\mathrm{Cov}(c\!\to\! c)$ for a fixed configuration $c$. Many deployments, however, care about reliability when $c_{\mathrm{dep}}$ may differ from $c_{\mathrm{cal}}$. 

The consequence for split conformal prediction is immediate. If calibration is performed under $c_{\mathrm{cal}}$ and deployment under $c_{\mathrm{dep}}$, then calibration scores have distribution $F_{c_{\mathrm{cal}}}$ while test scores have distribution $F_{c_{\mathrm{dep}}}$. Whenever $F_{c_{\mathrm{cal}}}\neq F_{c_{\mathrm{dep}}}$, calibration and test scores are generally not exchangeable, so the usual split-conformal proof no longer applies.

\begin{proposition}[Matched validity is not transportable across inference configurations. Proof in Appendix~\ref{app:proof_prop_exchangeability}.]
\label{prop:exchangeability_failure}
Consider split conformal prediction with score $A$. If calibration scores are generated under $c_{\mathrm{cal}}$ and test scores under $c_{\mathrm{dep}}$, and $F_{c_{\mathrm{cal}}}\neq F_{c_{\mathrm{dep}}}$, then calibration and test scores are, in general, not exchangeable, and the usual split-conformal coverage guarantee need not hold.
\end{proposition}

The source of failure is therefore not a change in the task distribution, but a change in the score-generating mechanism itself. Matched evaluation remains appropriate for the special case of a fixed known deployment configuration; the problem is that it need not answer the broader deployment question when inference configuration can vary after calibration. We present a more complete theoretical analysis in Appendix~\ref{app:theory}.
\section{Reliability Targets Under Deployment Uncertainty}
\label{sec:approaches}

Once inference configuration is allowed to vary after calibration, ``is CP valid?'' is no longer a single question. The relevant reliability target depends on the deployment assumption. If deployment configuration is fixed and known in advance, matched calibration is the correct target. If deployment follows a stable mixture over configurations, mixture-marginal validity is the relevant target; we include pooled calibration as a baseline for that setting and defer its formal treatment to the appendix. The main text focuses on the stronger regime in which deployment configuration is unobserved at test time and reliability should remain meaningful for individual latent deployment states.

We study two conservative baselines for that regime. \textsc{MaxThr} provides worst-case protection across configurations by hedging against the largest configuration-specific threshold. \textsc{Latent-Robust} uses the observed predictive vector to restrict that hedge to a conformal set of plausible latent configurations. These baselines are not intended as universally superior replacements for matched CP; they correspond to stricter deployment targets and, in our default construction, \textsc{Latent-Robust} can be even more conservative because the total error budget is split across two stages.

\subsection{Max-threshold calibration}
\label{subsec:maxthr}

A simple robust baseline is to calibrate separately for each configuration and then use the largest resulting threshold at deployment. Let $\mathcal{D}_{\mathrm{cal}}=\{(X_i,Y_i,C_i,\pi_i)\}_{i=1}^n$, and for each $c\in\mathcal{C}$ let
\[
\mathcal{D}_{\mathrm{cal}}^{(c)}=\{(X_i,Y_i,\pi_i)\in\mathcal{D}_{\mathrm{cal}}:C_i=c\}.
\]
From $\mathcal{D}_{\mathrm{cal}}^{(c)}$, compute a configuration-specific split-conformal threshold $\hat q_c$ at level $1-\alpha$ using the scores $V_i^{(c)}:=A(\pi_i,Y_i)$. The \textsc{MaxThr} predictor then uses the global threshold $\hat q_{\max}:=\max_{c\in\mathcal{C}}\hat q_c$ and returns
\[
\hat S_{\max}(x)=\{y\in\mathcal{Y}(x):A(\pi,y)\le \hat q_{\max}\}.
\]

This construction is deliberately conservative: it ignores any structure in the observed predictive vector and protects against deployment uncertainty by hedging against the largest configuration-specific threshold.

\begin{proposition}[Max-threshold validity, Proof in Appendix~\ref{app:proof_prop_maxthr}.]
\label{prop:maxthr}
Assume that, for each $c\in\mathcal{C}$, the calibration data in $\mathcal{D}_{\mathrm{cal}}^{(c)}$ and a test point generated under $c$ are exchangeable with respect to $A$, so that the configuration-specific predictor based on $\hat q_c$ is valid at level $1-\alpha$. Then
\[
\mathbb{P}(Y\in\hat S_{\max}(X)\mid C=c)\ge 1-\alpha
\qquad\text{for every } c\in\mathcal{C}.
\]
Consequently, $\hat S_{\max}$ is also valid marginally for any deployment distribution on $C$.
\end{proposition}

\subsection{Latent-robust calibration}
\label{subsec:latent}

\textsc{MaxThr} protects against latent deployment variation with a single global hedge. Our second baseline retains configuration-specific calibration while using the observed predictive vector to infer which deployment configurations are plausible for the current test point.

We split the calibration data into two disjoint parts, $\mathcal{D}_{\mathrm{cal}}=\mathcal{D}_1\cup\mathcal{D}_2$. In Stage 1, we perform configuration-conditional calibration: for each $c\in\mathcal{C}$, let
\[
\mathcal{D}_1^{(c)}=\{(X_i,Y_i,\pi_i)\in\mathcal{D}_1:C_i=c\}
\]
and compute a threshold $\hat q_c$ at level $1-\alpha_{\mathrm{stage1}}$ from the scores $V_i^{(c)}:=A(\pi_i,Y_i)$. This yields configuration-specific prediction sets
\[
\hat S_c(x)=\{y\in\mathcal{Y}(x):A(\pi,y)\le \hat q_c\}.
\]

In Stage 2, the deployment configuration is treated as latent. We therefore construct from the observed predictive vector a label-free signature that is informative about configuration identity. Our primary signature is
\[
Z(x)=\big(|\hat S_c(x)|\big)_{c\in\mathcal{C}},
\]
the vector of hypothetical set sizes induced by the Stage-1 thresholds. Using $\mathcal{D}_2$, we build a conformal set predictor $\hat\Gamma(Z(x))\subseteq\mathcal{C}$ such that the true deployment configuration is contained in $\hat\Gamma(Z(x))$ with probability at least $1-\beta_{\mathrm{stage2}}$.

The final threshold is then chosen adaptively as
\[
\hat q_{\Gamma}(x)=\max_{c\in\hat\Gamma(Z(x))}\hat q_c,
\]
with fallback $\max_{c\in\mathcal{C}}\hat q_c$ if $\hat\Gamma(Z(x))=\varnothing$, and the final prediction set is
\[
\hat S_{\Gamma}(x)=\{y\in\mathcal{Y}(x):A(\pi,y)\le \hat q_{\Gamma}(x)\}.
\]

\begin{theorem}[Latent-robust validity. Proof in Appendix~\ref{app:proof_thm_latent}.]
\label{thm:latent}
Assume (i) for each $c\in\mathcal{C}$, the Stage-1 calibration data in $\mathcal{D}_1^{(c)}$ and a test point generated under $c$ are exchangeable with respect to $A$, so that $\hat S_c$ is valid at level $1-\alpha_{\mathrm{stage1}}$; and (ii) the Stage-2 calibration tuples in $\mathcal{D}_2$ and the test tuple are exchangeable with respect to $(C,Z)$, so that $\mathbb{P}(C\in\hat\Gamma(Z(X)))\ge 1-\beta_{\mathrm{stage2}}$. Then
\[
\mathbb{P}(Y\in\hat S_{\Gamma}(X))\ge 1-\alpha_{\mathrm{stage1}}-\beta_{\mathrm{stage2}}.
\]
\end{theorem}

The key distinction from \textsc{MaxThr} is that the hedge is no longer global. However, this adaptivity should not be confused with reduced conservativeness. Under our default two-stage construction, part of the total error budget is allocated to latent-state identification and part to label coverage, so \textsc{Latent-Robust} can produce larger prediction sets than \textsc{MaxThr}. We therefore treat it as a stricter reference baseline for latent deployment uncertainty rather than as a uniformly better method.
\section{Evaluation}

We study inference-time model shift as an \emph{evaluation problem} through three questions. First, can matched evaluation materially misstate deployment-time coverage under calibration--deployment mismatch? Second, how much of this effect is attributable to temperature and to quantization, respectively? Third, how do calibration procedures aligned with different deployment assumptions compare in coverage under these regimes?

\subsection{Experimental Setup}

\noindent \textbf{Tasks and models.} We evaluate on four discrete-label benchmarks: \textsc{CommonsenseQA}~\cite{talmor2019commonsenseqa}, \textsc{DBPedia-14}~\cite{zhang2015charcnn}, \textsc{MMLU-Pro}~\cite{wang2024mmlupro}, and \textsc{MobileActions}~\cite{google2025mobileactions}. These datasets span multiple-choice and classification settings and differ substantially in difficulty, label structure, and output uncertainty. We study eight small language models: \textsc{Qwen2.5-1.5B}, \textsc{Qwen2.5-3B}, \textsc{Qwen2.5-7B}~\cite{qwen2024qwen25}, \textsc{DeepSeek-R1-Distill-Qwen-7B}~\cite{deepseek2025r1}, \textsc{Gemma-2-2B}~\cite{gu2024gemma2}, \textsc{Llama-3.2-3B}~\cite{grattafiori2024llama3}, \textsc{Phi-3-Mini}~\cite{abdin2024phi3}, and \textsc{Mistral-7B}~\cite{jiang2023mistral7b}. 

\noindent \textbf{Inference configurations.} An inference configuration is a pair \(c=(q,t)\in\mathcal{C}=\mathcal{Q}\times\mathcal{T}\), where $\mathcal{Q}$ is the set of quantization modes and $\mathcal{T}$ is the set of token-logit
temperature values used before candidate-label scoring. In the main benchmark, we evaluate the grid $\mathcal{Q}=\{\texttt{fp16},\texttt{int8},\texttt{nf4}\},
\; \mathcal{T}=\{0.5,1.0,1.5,2.0\}$, yielding \(|\mathcal{C}|=12\) inference configurations. Appendix~\ref{app:quant-temp} gives the implementation details for the quantization and temperature procedures.

\noindent \textbf{CP Setup.} In the main experiments, we use the \textsc{LAC} nonconformity score with \(\alpha=0.1\). Results for additional nonconformity scores, RAPS~\cite{angelopoulos2021uncertainty} and APS~\cite{romano2020classification}, are reported in Appendix~\ref{app:score_robustness}.

\noindent \textbf{Data splits.} For each dataset--model--configuration combination, we use disjoint calibration and evaluation sets; exact split sizes are reported in the appendix. All results are averaged over 3 independent runs with different random splits. For \textsc{Latent-Robust}, the calibration set is further split into Stage~1 and Stage~2 portions of 50\% and 50\%. In the main setting, we use \(\alpha_{\mathrm{stage1}}=\alpha/2\) and \(\beta_{\mathrm{stage2}}=\alpha/2\), so that \(\alpha_{\mathrm{stage1}}+\beta_{\mathrm{stage2}}=\alpha\). Appendix~\ref{app:runs} gives the implementation details of the Stage-2 latent-configuration predictor.

\noindent \textbf{Methods compared.} We compare six procedures:
\begin{itemize}[leftmargin=*]
    \item \textsc{Matched}: standard split conformal prediction calibrated under the true deployment configuration;
    \item \textsc{Mismatch Mean}: standard split conformal prediction, averaged over all off-diagonal calibration configurations for a fixed deployment configuration;
    \item \textsc{Mismatch Worst}: the minimum coverage over all off-diagonal calibration configurations for a fixed deployment configuration;
    \item \textsc{Pooled}: calibration on the pooled configuration mixture;
    \item \textsc{MaxThr}: the maximum of configuration-specific thresholds;
    \item \textsc{Latent-Robust}: the two-stage conservative construction from Section~\ref{sec:approaches}.
\end{itemize}

\paragraph{Metrics.}
Our primary metric in the main paper is empirical coverage, which measures reliability relative
to the nominal target $1-\alpha$. We defer average prediction-set size and other efficiency analyses
to Appendices~\ref{app:exp_details} and~\ref{app:main_results}. For a fixed deployment configuration \(c_{\mathrm{dep}}\), let \(\mathrm{Cov}(c_{\mathrm{cal}}\!\rightarrow\!c_{\mathrm{dep}})\) denote empirical coverage when calibration uses \(c_{\mathrm{cal}}\) and deployment uses \(c_{\mathrm{dep}}\). We summarize mismatch by
\[
\mathrm{MismatchMean}(c_{\mathrm{dep}})
=
\frac{1}{|\mathcal{C}|-1}
\sum_{c_{\mathrm{cal}}\neq c_{\mathrm{dep}}}
\mathrm{Cov}(c_{\mathrm{cal}}\!\rightarrow\!c_{\mathrm{dep}})
\]
and
\[
\mathrm{MismatchWorst}(c_{\mathrm{dep}})
=
\min_{c_{\mathrm{cal}}\neq c_{\mathrm{dep}}}
\mathrm{Cov}(c_{\mathrm{cal}}\!\rightarrow\!c_{\mathrm{dep}}).
\]
Full benchmark configuration, including dataset split sizes, model versions, score-function hyperparameters, and latent-split settings, is given in Appendix~\ref{app:exp_details}. Full results including set size and standard deviations are reported in Appendix~\ref{app:main_results}. Results with different nonconformity scores (APS, LAC, RAPS) are reported in Appendix~\ref{app:score_robustness}, and results with different $\alpha$ values are reported in Appendix~\ref{app:alpha_robustness}.

\subsection{Experiment 1: Reliability Under Calibration--Deployment Mismatch}
\label{subsec:exp1}

We fix deployment to quantization \(\texttt{nf4}\) and temperature \(T=0.5\). Figure~\ref{fig:exp1_nf4_t05_bars} reports coverage separately for each dataset--model pair, averaged over three runs. Across all four datasets, \textsc{Matched} tracks the nominal $0.90$ target tightly, while \textsc{Mismatch Worst} falls below target for almost every model. The robust baselines, \textsc{MaxThr} and \textsc{Latent-Robust}, restore target coverage in nearly all cells. As expected, this coverage recovery comes with increased prediction-set size; detailed efficiency results are reported in Appendix~\ref{app:main_results}.

\textsc{MMLU-Pro} exhibits the strongest degradation. All eight models fall below target under \textsc{Mismatch Mean}, \textsc{Mismatch Worst}, and \textsc{Pooled}; descriptive averages over models drop from $0.902$ under \textsc{Matched} to $0.639$ under \textsc{Mismatch Mean} and $0.508$ under \textsc{Mismatch Worst}, with \textsc{Pooled} at $0.686$. \textsc{CommonsenseQA} follows the same qualitative pattern more mildly: all eight models fall below target under the three mismatch summaries, but the corresponding average drops are smaller (matched $0.912$; mismatch mean $0.827$; mismatch worst $0.789$; pooled $0.818$).

\textsc{DBPedia-14} differs in kind rather than degree. The descriptive averages are comparatively close to target ($0.901 / 0.863 / 0.827 / 0.867$), but \textsc{Mismatch Mean} spans a wide model-level range ($0.795$ to $0.946$): the same summary statistic conceals substantial degradation for some model families and near-matched behavior for others. \textsc{MobileActions} is the most stable dataset, with average mismatch effects within a few points of target, but it is not invariant: \textsc{Mismatch Worst} still falls below target for 6 of 8 models, indicating that even comparatively stable settings retain adverse calibration--deployment pairs. Full per-model tables are in Appendix~\ref{app:exp1}.

Two observations cut across the four datasets. First, the gap between \textsc{Mismatch Mean} and \textsc{Mismatch Worst} is consistently large, so any summary that averages over off-diagonal calibration choices understates the worst-case reliability cost. Second, the heterogeneity is structural: \textsc{MMLU-Pro} and \textsc{CommonsenseQA} fail uniformly across models while \textsc{DBPedia-14} fails selectively, and a single reporting protocol that fixes one calibration choice cannot capture both regimes.

\begin{figure*}[t]
    \centering
    \includegraphics[width=\linewidth]{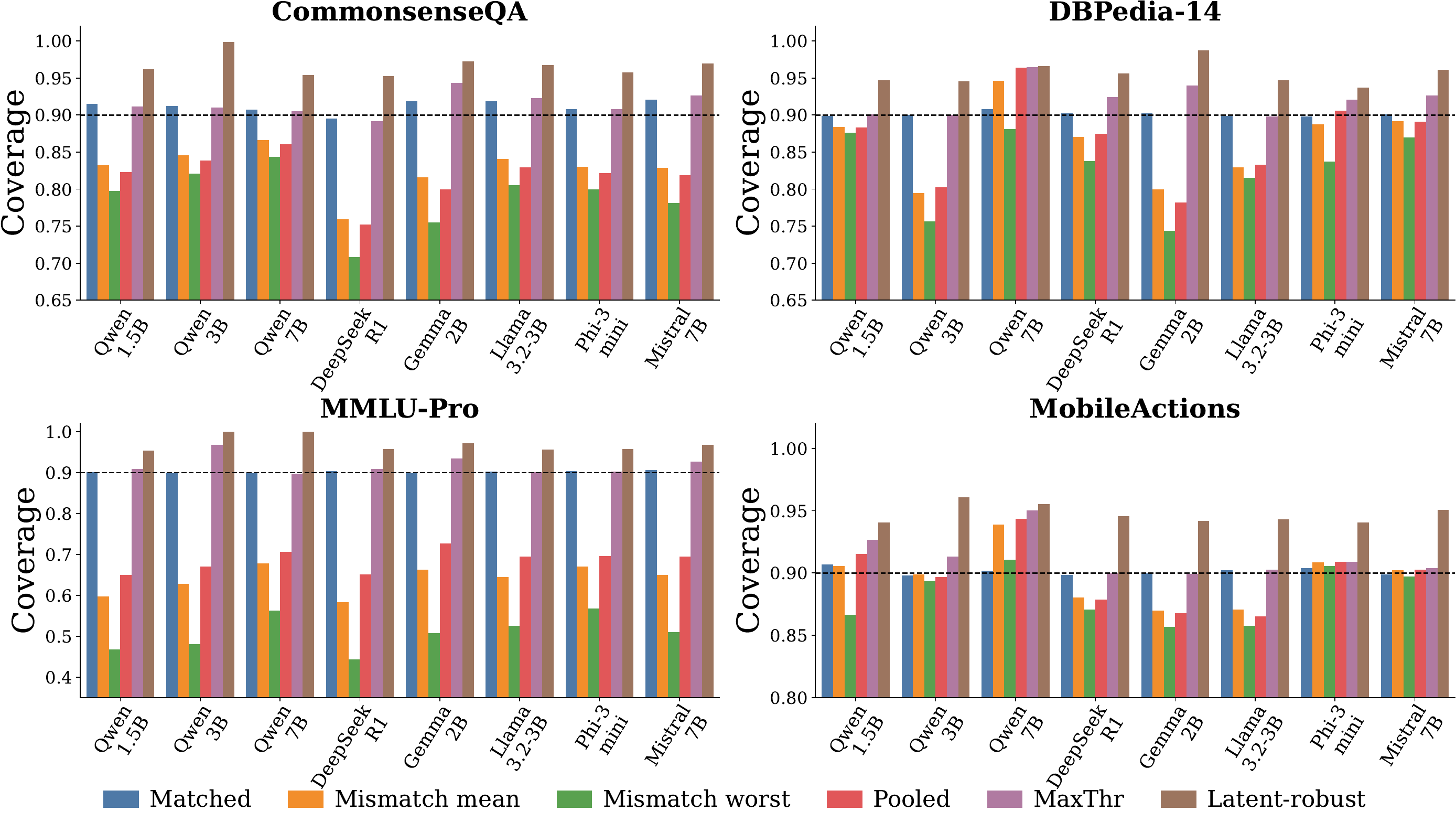}
    \caption{
    \textbf{Coverage with deployment fixed to \(\texttt{nf4}@T=0.5\).}
    Coverage is shown for 6 CP approaches across eight models and four datasets. 
    }
    \label{fig:exp1_nf4_t05_bars}
\end{figure*}

\subsection{Experiment 2: Quantization Effects at Fixed Deployment Temperature}
\label{subsec:exp2}

We next isolate the role of quantization by fixing deployment temperature to \(T=0.5\) and averaging over deployment quantization modes. Figure~\ref{fig:exp2_t05_avg_quant} reports coverage separately for each dataset--model pair. For each model and run, coverage is first averaged over deployment quantization modes at fixed \(T=0.5\), and the plotted bar is then averaged over runs.

The mismatch effect persists. The qualitative pattern from Experiment~1 is preserved across all four datasets: \textsc{Matched} tracks target, \textsc{Mismatch Worst} falls below target for all 8 models on every dataset, and \textsc{MaxThr} and \textsc{Latent-Robust} restore coverage. \textsc{MMLU-Pro} again shows the largest drop, with descriptive averages over models of $0.906$ (\textsc{Matched}), $0.648$ (\textsc{Mismatch Mean}), $0.520$ (\textsc{Mismatch Worst}), and $0.695$ (\textsc{Pooled})---essentially unchanged from the single-configuration values reported in Experiment~1. \textsc{CommonsenseQA} likewise remains uniformly mismatch-sensitive, with averages $0.912 / 0.833 / 0.796 / 0.824$.

Heterogeneity also persists. \textsc{DBPedia-14} continues to span a wide model-level range under \textsc{Mismatch Mean} ($0.777$ to $0.952$), so averaging over deployment quantization does not collapse the cross-model variation. \textsc{MobileActions} remains the most stable dataset on average, but \textsc{Mismatch Worst} still falls below target for all 8 models, indicating that adverse calibration--deployment pairs survive the within-run averaging. We present full per-model tables in Appendix~\ref{app:exp2}.

The results indicate that the degradation observed is not specific to a particular quantization mode. Low-temperature deployment alone already defines a regime in which standard CP is substantially more sensitive to calibration--deployment mismatch, regardless of the deployment quantization mode.

\begin{figure*}[t]
    \centering
    \includegraphics[width=\linewidth]{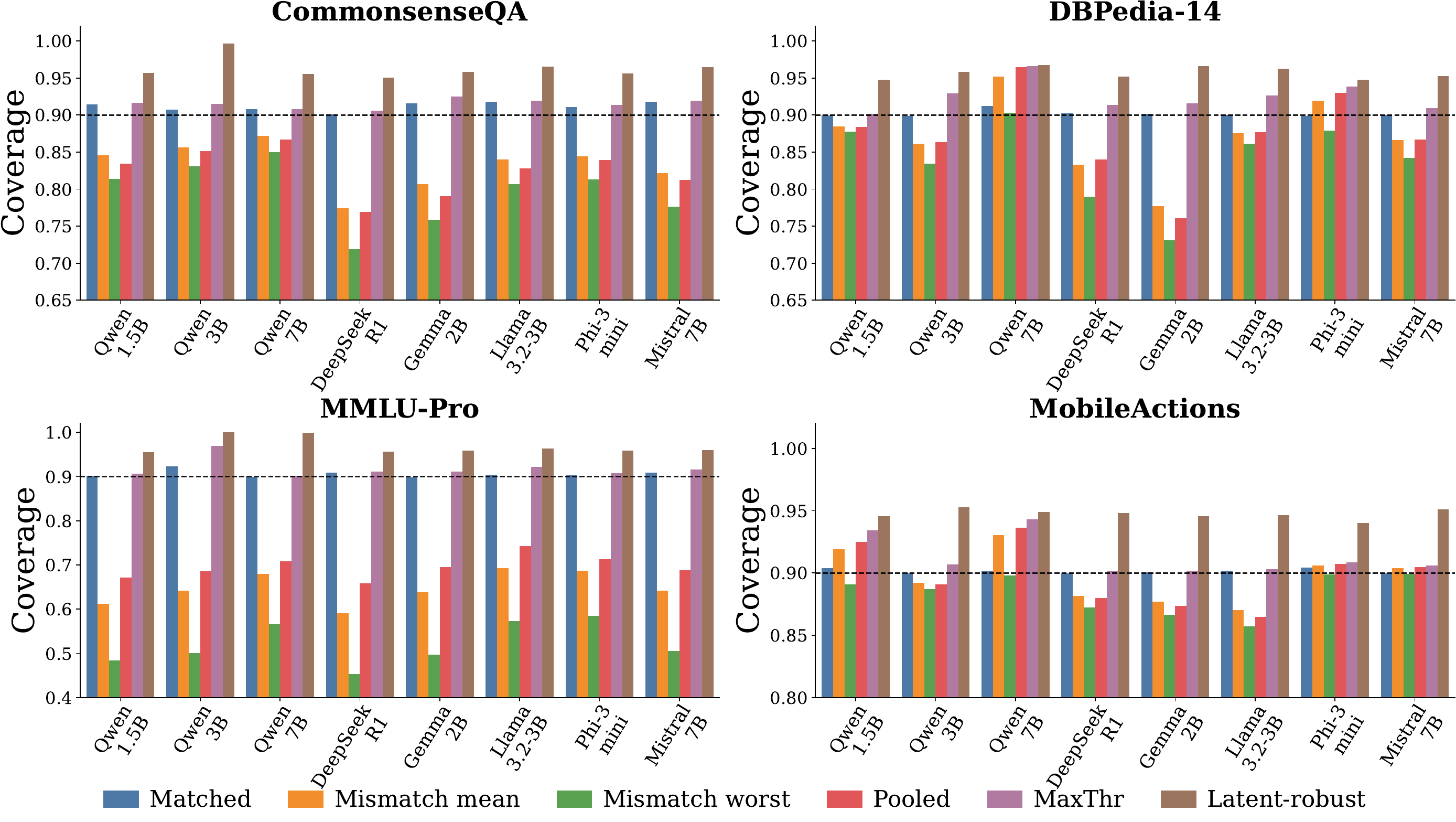}
    \caption{
    \textbf{Coverage with deployment temperature fixed to \(T=0.5\), averaged over deployment quantization modes.}
    Coverage is shown for all six procedures across eight models and four datasets. }
    \label{fig:exp2_t05_avg_quant}
\end{figure*}

\subsection{Experiment 3: Temperature Effects at Fixed Deployment Quantization}
\label{subsec:exp3}

Finally, we isolate the role of temperature by fixing deployment quantization to \(\texttt{nf4}\) and averaging over deployment temperatures. This complements Experiment~2 by asking whether the mismatch effect persists when quantization is held fixed and only the inference regime varies. Figure~\ref{fig:exp3_nf4_avg_temp} reports coverage separately for each dataset--model pair. For each model and run, coverage is first averaged over deployment temperatures at fixed \(\texttt{nf4}\), and the plotted bar is then averaged over runs.

The mismatch effect persists, but more selectively than in Experiments~1 and~2. The clearest sign is that \textsc{Mismatch Worst} still falls below target for all 8 of 8 models on every dataset, whereas \textsc{Mismatch Mean} now falls below target on only some datasets: 8 of 8 on \textsc{MMLU-Pro}, 1 of 8 on \textsc{CommonsenseQA}, and an intermediate count on \textsc{DBPedia-14} and \textsc{MobileActions}. Averaging over temperatures attenuates the average mismatch effect for many model--dataset pairs but does not eliminate adverse calibration--deployment pairs.

\textsc{MMLU-Pro} again shows the strongest degradation: descriptive averages over models drop from $0.901$ (\textsc{Matched}) to $0.842$ (\textsc{Mismatch Mean}) and $0.719$ (\textsc{Mismatch Worst}). \textsc{DBPedia-14} again shows the strongest cross-model heterogeneity, with \textsc{Mismatch Mean} spanning $0.567$ to $0.922$ and \textsc{Mismatch Worst} spanning $0.220$ to $0.890$, a range so wide that some model families are nearly invariant under the averaging while others collapse. This rules out a clear monotone ordering of inference configurations along which mismatch severity could be predicted.

With Experiment~2, the results show that neither temperature nor quantization alone accounts for the observed degradation. Each axis is sufficient to induce mismatch effects when the other is held fixed. We present further results including full per-model tables in Appendix~\ref{app:exp3}.
\begin{figure*}[t]
    \centering
    \includegraphics[width=\linewidth]{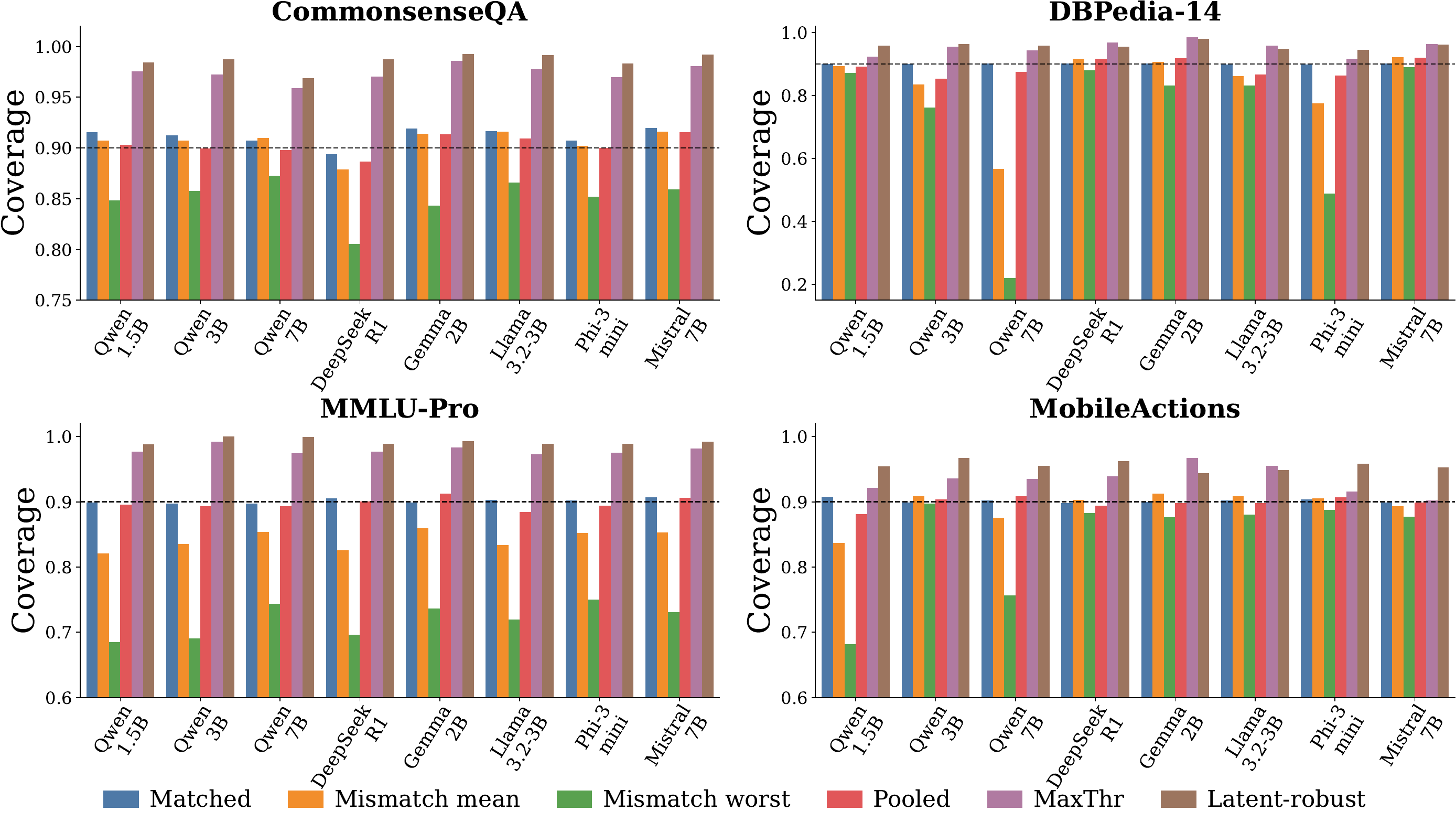}
    \caption{
    \textbf{Coverage with deployment quantization fixed to \(\texttt{nf4}\), averaged over deployment temperatures.}
    Coverage is shown for all six procedures across eight models and four datasets. 
    }
    \label{fig:exp3_nf4_avg_temp}
\end{figure*}
\section{Discussion}
\label{sec:discussion}

The central claim of this paper is evaluative: for deployed language-model systems, inference configuration is part of the validity problem itself, not an implementation detail to be held fixed or averaged away. Standard CP evaluations ask whether the guarantee holds when calibration and deployment share an inference pipeline. Our results show this is insufficient. Even with the task distribution and trained model fixed, changing the inference configuration between calibration and deployment can move the conformal score distribution enough to invalidate split-conformal validity, and matched-only evaluation cannot, by construction, surface this failure.

The empirical picture shows the challenging nature of inference-time model shift. In our experiments, low-temperature deployment is consistently the most mismatch-sensitive regime, and aggressive quantization can amplify the effect, but the failures admit no monotone ordering over inference configurations. Their magnitude varies with the dataset and model family, and the direction of the effect can change across settings. 

\noindent \textbf{Implications for evaluation practice.} Our findings motivate a concrete revision of reporting practice. We argue that evaluations of CP for language models should treat inference configuration as a first-class experimental axis, on the same footing as the dataset and the model. Concretely, this entails three reporting steps: (i) state the deployment assumption being targeted: fixed configuration, stable mixture, or unobserved at test time, since each implies a different reliability target; (ii) report sensitivity to calibration--deployment mismatch along the inference axes that may vary at deployment, rather than matched coverage alone; and (iii) where the deployment configuration is not known in advance, report a baseline aligned with the targeted deployment regime. In the setting studied here, these reporting steps require no new theoretical machinery, but they do require treating inference configuration as part of the experimental design. We view this as a necessary step toward conformal guarantees that remain meaningful when language-model systems are actually deployed.

\noindent \textbf{Limitations.} Our analysis is primarily a characterization rather than a mechanistic account: we identify when mismatch matters but do not provide a predictive theory of when a given configuration change will induce undercoverage, overcoverage, or only mild effects. The audit is carried out on a finite deployment grid, which captures many realistic serving choices but does not address richer deployment variation, including continuous temperatures or temporal drift in serving policies.
\section{Conclusion}

We identify \emph{inference-time model shift} as a distinct failure mode for conformal prediction in language models: even under a fixed task distribution, changing deployment-time inference configuration can alter the conformal score distribution and break the assumptions behind matched calibration. Across eight small language models, four datasets, and a grid of quantization and temperature settings, this effect leads matched-only evaluation to substantially overstate deployment-time reliability. These results motivate a simple change in practice: conformal evaluations for language models should treat inference configuration as part of the experimental design, not as a fixed implementation detail.

\bibliographystyle{plain}
\bibliography{refs}
%%%%%%%%%%%%%%%%%%%%%%%%%%%%%%%%%%%%%%%%%%%%%%%%%%%%%%%%%%%%

\clearpage
\appendix
\section*{Appendix}
\startcontents[sections]
\printcontents[sections]{l}{1}{\setcounter{tocdepth}{2}}

% \section*{Appendix}
% \addcontentsline{toc}{section}{Appendix}

% % Show only the appendix in this TOC
% \etocsettocstyle{\section*{Contents of this appendix}}{}
% \localtableofcontents
%\clearpage

\section{Additional Theory and Full Statements}
\label{app:theory}

This appendix completes the technical development from the main text. We first fix notation and recall the split-conformal threshold used throughout. We then formalize the reliability targets induced by different deployment assumptions and provide full proofs for the statements in Sections~\ref{sec:shift} and~\ref{sec:approaches}. The purpose of these results is not to characterize all effects of inference-configuration mismatch, but to make precise which guarantees remain valid under which deployment regimes.

\subsection{Setup and split-conformal preliminaries}
\label{app:setup}

We use the notation from Sections~\ref{sec:shift} and~\ref{sec:approaches}. For a calibration sample
\[
\mathcal D=\{(X_i,Y_i,\pi_i)\}_{i=1}^m
\]
and nonconformity score \(A:\Delta^{|\mathcal Y(x)|}\times \mathcal Y(x)\to\mathbb R\), define the calibration scores
\[
V_i := A(\pi_i,Y_i), \qquad i=1,\dots,m.
\]
Let \(V_{(1)}\le \cdots \le V_{(m)}\) denote the order statistics. For target miscoverage \(\alpha\in(0,1)\), define
\[
k_\alpha(m):=\left\lceil (m+1)(1-\alpha)\right\rceil
\]
and the split-conformal threshold
\[
\widehat q_\alpha(\mathcal D)
:=
\begin{cases}
V_{(k_\alpha(m))}, & \text{if } k_\alpha(m)\le m,\\[0.4em]
+\infty, & \text{if } k_\alpha(m)=m+1.
\end{cases}
\]
Given a predictive vector \(\pi\) for a test input \(x\), the associated prediction set is
\[
\widehat S_\alpha(x;\pi,\mathcal D)
:=
\{y\in\mathcal Y(x): A(\pi,y)\le \widehat q_\alpha(\mathcal D)\}.
\]

We use the following standard split-conformal fact repeatedly.

\begin{lemma}[Standard split-conformal validity]
\label{lem:standard_split_cp}
Let \(\{(X_i,Y_i,\pi_i)\}_{i=1}^{m}\) be a calibration sample and let \((X_{m+1},Y_{m+1},\pi_{m+1})\) be a test tuple. If the \(m+1\) tuples are exchangeable and \(V_i=A(\pi_i,Y_i)\), then
\[
\mathbb P\!\left(
Y_{m+1}\in \widehat S_\alpha(X_{m+1};\pi_{m+1},\mathcal D)
\right)
=
\mathbb P\!\left(
V_{m+1}\le \widehat q_\alpha(\mathcal D)
\right)
\ge 1-\alpha.
\]
\end{lemma}

We also use the following monotonicity property.

\begin{lemma}[Monotonicity in the threshold]
\label{lem:monotonicity}
If \(q'\ge q\), then for every pair \((x,\pi)\),
\[
\{y\in\mathcal Y(x):A(\pi,y)\le q\}
\subseteq
\{y\in\mathcal Y(x):A(\pi,y)\le q'\}.
\]
\end{lemma}

\begin{proof}
Immediate from the definition of the prediction set.
\end{proof}

\subsection{Proof of Proposition~\ref{prop:exchangeability_failure}}
\label{app:proof_prop_exchangeability}

We restate the proposition for convenience.

\begin{proposition*}[Matched validity is not transportable across inference configurations]
Consider split conformal prediction with score \(A\). If calibration scores are generated under \(c_{\mathrm{cal}}\) and test scores under \(c_{\mathrm{dep}}\), and \(F_{c_{\mathrm{cal}}}\neq F_{c_{\mathrm{dep}}}\), then calibration and test scores are, in general, not exchangeable, and the usual split-conformal coverage guarantee need not hold.
\end{proposition*}

\begin{proof}
Let \(V_1,\dots,V_n\) denote the calibration scores generated under \(c_{\mathrm{cal}}\), and let \(V_{n+1}\) denote the test score generated under \(c_{\mathrm{dep}}\). By construction,
\[
V_i \sim F_{c_{\mathrm{cal}}}
\quad \text{for } i=1,\dots,n,
\qquad
V_{n+1}\sim F_{c_{\mathrm{dep}}}.
\]
If the vector \((V_1,\dots,V_n,V_{n+1})\) were exchangeable, then all coordinates would have the same marginal distribution. Since \(F_{c_{\mathrm{cal}}}\neq F_{c_{\mathrm{dep}}}\), this is impossible in general. Hence calibration and test scores are not exchangeable in general.

The standard split-conformal proof applies by ranking the test score among the calibration scores, and therefore requires exchangeability of the \(n+1\) scores. Once exchangeability is absent, that proof no longer applies.

To see that the coverage guarantee can in fact fail, consider the simple score-level counterexample in which
\[
V_1=\cdots=V_n=0 \quad \text{a.s.},
\qquad
V_{n+1}=1 \quad \text{a.s.}
\]
Then \(F_{c_{\mathrm{cal}}}\neq F_{c_{\mathrm{dep}}}\). For any \(\alpha<1\), the split-conformal threshold computed from the calibration scores is \(\widehat q_\alpha=0\), while the test score equals \(1\) almost surely. Therefore
\[
\mathbb P(V_{n+1}\le \widehat q_\alpha)=0,
\]
and equivalently the resulting conformal predictor has coverage \(0\). Thus the usual split-conformal coverage guarantee need not hold under calibration--deployment mismatch in the score-generating configuration.
\end{proof}

\begin{remark}
Proposition~\ref{prop:exchangeability_failure} is deliberately minimal. It does not claim that every configuration mismatch causes undercoverage, nor that mismatch effects are monotone across configurations. It identifies the point at which the standard split-conformal proof breaks: when calibration and test scores are no longer exchangeable because the score-generating mechanism has changed.
\end{remark}

\subsection{Reliability targets under deployment uncertainty}
\label{app:targets}

Once inference configuration may vary after calibration, there is no single universal notion of deployment-time validity. The relevant target depends on the deployment assumption.

\begin{definition}[Fixed-configuration validity]
\label{def:fixed_config_validity}
Let \(c\in\mathcal C\) be a fixed deployment configuration. A predictor \(\widehat S\) is valid at level \(1-\alpha\) for configuration \(c\) if
\[
\mathbb P\!\left(Y\in \widehat S(X)\mid C=c\right)\ge 1-\alpha.
\]
\end{definition}

\begin{definition}[Mixture-marginal validity]
\label{def:mixture_validity}
Let \(G\) be a deployment distribution on \(\mathcal C\). A predictor \(\widehat S\) is mixture-marginally valid at level \(1-\alpha\) under \(G\) if
\[
\mathbb P\!\left(Y\in \widehat S(X)\right)\ge 1-\alpha
\]
when \(C\sim G\).
\end{definition}

\begin{definition}[Configuration-conditional validity]
\label{def:conditional_validity}
A predictor \(\widehat S\) is configuration-conditionally valid at level \(1-\alpha\) if
\[
\mathbb P\!\left(Y\in \widehat S(X)\mid C=c\right)\ge 1-\alpha
\qquad\text{for every } c\in\mathcal C.
\]
\end{definition}

These targets are distinct. Fixed-configuration validity refers to one known deployment state. Mixture-marginal validity averages over a deployment mixture. Configuration-conditional validity is stronger: it requires validity for every configuration separately.

The procedures studied in the main text align with these targets as follows:
\begin{itemize}
    \item \textsc{Matched} targets fixed-configuration validity.
    \item \textsc{Pooled} targets mixture-marginal validity under a stable deployment mixture.
    \item \textsc{MaxThr} targets configuration-conditional validity under the assumptions of Proposition~\ref{prop:maxthr}.
    \item \textsc{Latent-Robust}, as analyzed in Theorem~\ref{thm:latent}, combines configuration-specific Stage-1 calibration with a Stage-2 conformal localization step and yields a marginal coverage guarantee at level \(1-\alpha_{\mathrm{stage1}}-\beta_{\mathrm{stage2}}\).
\end{itemize}

The last point is important: although \textsc{Latent-Robust} is motivated by latent deployment uncertainty, the theorem in the main text is an unconditional coverage statement. It does \emph{not} by itself imply the same configuration-conditional guarantee as Proposition~\ref{prop:maxthr}; obtaining such a result would require a stronger, configuration-conditional form of Stage-2 control.

The following elementary implication will also be useful.

\begin{lemma}[Configuration-conditional validity implies mixture-marginal validity]
\label{lem:conditional_implies_marginal}
If \(\widehat S\) satisfies
\[
\mathbb P\!\left(Y\in \widehat S(X)\mid C=c\right)\ge 1-\alpha
\qquad\text{for every } c\in\mathcal C,
\]
then \(\widehat S\) is mixture-marginally valid at level \(1-\alpha\) for any deployment distribution on \(\mathcal C\).
\end{lemma}

\begin{proof}
By the law of total probability,
\[
\mathbb P(Y\in \widehat S(X))
=
\sum_{c\in\mathcal C}
\mathbb P(C=c)\,
\mathbb P(Y\in \widehat S(X)\mid C=c)
\ge
\sum_{c\in\mathcal C}
\mathbb P(C=c)\,(1-\alpha)
=
1-\alpha.
\]
\end{proof}

\subsection{Pooled calibration and mixture-marginal validity}
\label{app:pooled}

We now state formally the guarantee targeted by pooled calibration.

\begin{proposition}[Pooled calibration under a stable deployment mixture]
\label{prop:pooled}
Let \(G\) be a deployment distribution on \(\mathcal C\). Suppose the calibration tuples
\[
\{(X_i,Y_i,C_i,\pi_i)\}_{i=1}^{m}
\]
and the test tuple \((X_{m+1},Y_{m+1},C_{m+1},\pi_{m+1})\) are exchangeable under the joint law induced by \(P_{XY}\), the mixture \(G\), and the configuration-dependent predictive mechanism \(c\mapsto \pi^{(c)}(\cdot)\). Let \(\widehat S_{\mathrm{pool}}\) be the split-conformal predictor obtained by pooling all calibration tuples and ignoring configuration labels when computing the threshold. Then
\[
\mathbb P\!\left(
Y_{m+1}\in \widehat S_{\mathrm{pool}}(X_{m+1})
\right)\ge 1-\alpha.
\]
That is, pooled calibration is mixture-marginally valid under the stable deployment mixture \(G\).
\end{proposition}

\begin{proof}
Let \(V_i:=A(\pi_i,Y_i)\) for \(i=1,\dots,m+1\). Since the tuples are exchangeable and \(V_i\) is a measurable function of the \(i\)-th tuple, the pooled scores \(V_1,\dots,V_{m+1}\) are exchangeable. Applying Lemma~\ref{lem:standard_split_cp} to the pooled calibration sample yields
\[
\mathbb P\!\left(
Y_{m+1}\in \widehat S_{\mathrm{pool}}(X_{m+1})
\right)
\ge 1-\alpha.
\]
This is exactly the claimed mixture-marginal validity statement.
\end{proof}

\begin{remark}
Proposition~\ref{prop:pooled} is a statement about the mixture induced by \(G\). It does not imply validity within each configuration \(c\in\mathcal C\). Therefore pooled calibration may be appropriate when average reliability over a stable deployment mixture is the target, but can be inadequate when the objective is reliability for individual latent deployment states.
\end{remark}

\subsection{Proof of Proposition~\ref{prop:maxthr}}
\label{app:proof_prop_maxthr}

We restate the proposition for convenience.

\begin{proposition*}[Max-threshold validity]
Assume that, for each \(c\in\mathcal C\), the calibration data in \(\mathcal D_{\mathrm{cal}}^{(c)}\) and a test point generated under \(c\) are exchangeable with respect to \(A\), so that the configuration-specific predictor based on \(\widehat q_c\) is valid at level \(1-\alpha\). Then
\[
\mathbb P(Y\in \widehat S_{\max}(X)\mid C=c)\ge 1-\alpha
\qquad\text{for every } c\in\mathcal C.
\]
Consequently, \(\widehat S_{\max}\) is also valid marginally for any deployment distribution on \(C\).
\end{proposition*}

\begin{proof}
Fix \(c\in\mathcal C\). By assumption, the configuration-specific split-conformal predictor
\[
\widehat S_c(x):=\{y\in \mathcal Y(x):A(\pi,y)\le \widehat q_c\}
\]
satisfies
\[
\mathbb P(Y\in \widehat S_c(X)\mid C=c)\ge 1-\alpha.
\]
By construction,
\[
\widehat q_{\max}=\max_{c'\in\mathcal C}\widehat q_{c'}\ge \widehat q_c.
\]
Hence, by Lemma~\ref{lem:monotonicity},
\[
\widehat S_c(x)\subseteq \widehat S_{\max}(x)
\qquad\text{for every } x.
\]
Therefore
\[
\mathbb P(Y\in \widehat S_{\max}(X)\mid C=c)
\ge
\mathbb P(Y\in \widehat S_c(X)\mid C=c)
\ge 1-\alpha.
\]
Since \(c\in\mathcal C\) was arbitrary, the configuration-conditional statement holds for every configuration.

The marginal validity statement now follows from Lemma~\ref{lem:conditional_implies_marginal}.
\end{proof}

\subsection{Proof of Theorem~\ref{thm:latent}}
\label{app:proof_thm_latent}

We restate the theorem for convenience.

\begin{theorem*}[Latent-robust validity]
Assume:
\begin{enumerate}
    \item for each \(c\in\mathcal C\), the Stage-1 calibration data in \(\mathcal D_1^{(c)}\) and a test point generated under \(c\) are exchangeable with respect to \(A\), so that \(\widehat S_c\) is valid at level \(1-\alpha_{\mathrm{stage1}}\); and
    \item the Stage-2 calibration tuples in \(\mathcal D_2\) and the test tuple are exchangeable with respect to \((C,Z)\), so that
    \[
    \mathbb P\!\left(C\in \widehat\Gamma(Z(X))\right)\ge 1-\beta_{\mathrm{stage2}}.
    \]
\end{enumerate}
Then
\[
\mathbb P\!\left(Y\in \widehat S_{\Gamma}(X)\right)\ge 1-\alpha_{\mathrm{stage1}}-\beta_{\mathrm{stage2}}.
\]
\end{theorem*}

\begin{proof}
Let
\[
E:=\{C\in \widehat\Gamma(Z(X))\}
\]
denote the event that the Stage-2 conformal set of plausible configurations contains the true deployment configuration.

Define the configuration-specific Stage-1 predictor
\[
\widehat S_C(X)
:=
\{y\in\mathcal Y(X): A(\pi,y)\le \widehat q_C\},
\]
where \(\widehat q_C\) denotes the Stage-1 threshold corresponding to the realized configuration \(C\).

On the event \(E\), the realized configuration \(C\) belongs to \(\widehat\Gamma(Z(X))\). Therefore, by definition of \(\widehat q_\Gamma(X)\),
\[
\widehat q_\Gamma(X)
=
\max_{c\in \widehat\Gamma(Z(X))}\widehat q_c
\ge \widehat q_C.
\]
Hence, by Lemma~\ref{lem:monotonicity},
\[
\widehat S_C(X)\subseteq \widehat S_\Gamma(X)
\qquad\text{on } E.
\]
Equivalently,
\[
\{Y\notin \widehat S_\Gamma(X)\}\cap E
\subseteq
\{Y\notin \widehat S_C(X)\}.
\]
It follows that
\begin{align*}
\mathbb P(Y\notin \widehat S_\Gamma(X))
&=
\mathbb P\bigl(\{Y\notin \widehat S_\Gamma(X)\}\cap E\bigr)
+
\mathbb P\bigl(\{Y\notin \widehat S_\Gamma(X)\}\cap E^c\bigr) \\
&\le
\mathbb P(Y\notin \widehat S_C(X))
+
\mathbb P(E^c).
\end{align*}

We now bound the two terms separately.

For the first term, use the law of total probability and the Stage-1 validity assumption:
\begin{align*}
\mathbb P(Y\notin \widehat S_C(X))
&=
\sum_{c\in\mathcal C}
\mathbb P(C=c)\,
\mathbb P(Y\notin \widehat S_c(X)\mid C=c) \\
&\le
\sum_{c\in\mathcal C}
\mathbb P(C=c)\,\alpha_{\mathrm{stage1}}
=
\alpha_{\mathrm{stage1}}.
\end{align*}

For the second term, assumption (ii) gives
\[
\mathbb P(E^c)
=
\mathbb P\!\left(C\notin \widehat\Gamma(Z(X))\right)
\le \beta_{\mathrm{stage2}}.
\]

Combining the two bounds yields
\[
\mathbb P(Y\notin \widehat S_\Gamma(X))
\le
\alpha_{\mathrm{stage1}}+\beta_{\mathrm{stage2}}.
\]
Equivalently,
\[
\mathbb P(Y\in \widehat S_\Gamma(X))
\ge
1-\alpha_{\mathrm{stage1}}-\beta_{\mathrm{stage2}},
\]
which proves the theorem.
\end{proof}

\begin{remark}
Theorem~\ref{thm:latent} guarantees marginal coverage at level \(1-\alpha_{\mathrm{stage1}}-\beta_{\mathrm{stage2}}\), but it does not imply that \textsc{Latent-Robust} is uniformly less conservative than \textsc{MaxThr}. In our default construction, the total error budget is split across two stages, so the Stage-1 thresholds are computed at level \(1-\alpha_{\mathrm{stage1}}\) with \(\alpha_{\mathrm{stage1}}<\alpha\). This can inflate prediction sets relative to \textsc{MaxThr}, even though \(\widehat q_\Gamma(X)\) is adaptive.
\end{remark}

\begin{remark}
If one strengthened assumption (ii) to the configuration-conditional statement
\[
\mathbb P\!\left(C\in \widehat\Gamma(Z(X))\mid C=c\right)\ge 1-\beta_{\mathrm{stage2}}
\qquad\text{for every } c\in\mathcal C,
\]
then the same argument would yield
\[
\mathbb P\!\left(Y\in \widehat S_\Gamma(X)\mid C=c\right)
\ge
1-\alpha_{\mathrm{stage1}}-\beta_{\mathrm{stage2}}
\qquad\text{for every } c\in\mathcal C.
\]
We do not claim this stronger guarantee in the main paper.
\end{remark}

\subsection{Remarks and non-claims}
\label{app:remarks}

We close with several clarifications that are useful when interpreting the main results.

\begin{remark}[No universal monotone ordering over configurations]
The theory above does not assume that one inference configuration is uniformly ``harder'' or ``worse'' than another. In particular, nothing in the formal setup implies a monotone ordering of coverage across temperatures or quantization modes.
\end{remark}

\begin{remark}[Mismatch need not imply undercoverage]
The failure of exchangeability under inference-configuration mismatch means that the standard split-conformal guarantee is no longer justified. It does not imply that mismatch must always reduce coverage. Depending on the direction of the score-distribution change, mismatch may produce undercoverage, overcoverage, or only mild effects.
\end{remark}

\begin{remark}[Different baselines target different deployment questions]
\textsc{Matched}, \textsc{Pooled}, \textsc{MaxThr}, and \textsc{Latent-Robust} should not be interpreted as points on a single efficiency frontier. They correspond to different deployment assumptions and therefore to different reliability targets.
\end{remark}

\begin{remark}[Scope of the theory]
The theory in this appendix is intentionally modest. Its role is to make precise how deployment-time changes in inference configuration affect the assumptions behind standard split conformal prediction, and to clarify which guarantees remain valid under which deployment regimes. It does not attempt to provide a predictive theory of when a particular configuration change will induce undercoverage, overcoverage, or only mild effects.
\end{remark}
%%%%%%%%%%%%%%%%%%%%%%%%%%%%%%%%%

\section{More on related work}
\label{app:prior-audit}

Table~\ref{tab:prior-audit} summarizes how representative prior work relates to the
calibration--deployment inference-configuration mismatch studied in this paper. The review focuses
on whether inference configuration is treated as an experimental axis for conformal validity, rather
than on whether a paper reports every low-level generation parameter. Matched calibration is the
right target when the deployment configuration is fixed and known. Our setting instead asks whether
coverage transports when calibration and deployment differ only in inference configuration.

\begin{table}[t]
\centering
\small
\setlength{\tabcolsep}{3pt}
\renewcommand{\arraystretch}{1.12}
\caption{Representative prior work and its relation to calibration--deployment inference-configuration mismatch.}
\label{tab:prior-audit}
\begin{tabular}{p{0.22\linewidth}p{0.30\linewidth}p{0.36\linewidth}}
\toprule
Work & Primary focus & Relation to our setting \\
\midrule
Kumar et al.~\cite{kumar2023conformal}
& CP for multiple-choice QA with LLMs
& Finite-label LLM setting, but not calibration--deployment mismatch across inference configurations. \\

Quach et al.~\cite{quach2024clm}
& Conformal language modeling for generated candidate sets
& CP for generation under a fixed score-generating setup. \\

Mohri and Hashimoto~\cite{mohri2024factuality}
& Conformal factuality guarantees
& Correctness/factuality guarantees, not inference-configuration-induced score shift. \\

Cherian et al.~\cite{cherian2024llmvalidity}
& Enhanced conformal validity for LLM outputs
& Stronger validity notions for LLM outputs, but not quantization/temperature mismatch. \\

Su et al.~\cite{su2024api}
& API-only CP without logit access
& Black-box uncertainty scores with the inference interface treated as fixed. \\

Wang et al.~\cite{wang2024conu}
& Correctness coverage for open-ended generation
& Open-ended correctness coverage, not coverage transport across inference configurations. \\

Ulmer et al.~\cite{ulmer2024nonexchangeable}
& Non-exchangeable conformal language generation
& Addresses non-exchangeability in generation, but not inference configuration as the shift mechanism. \\

Wang et al.~\cite{wang2025sconu}
& Selective conformal uncertainty and outlier testing
& Tests uncertainty outliers and exchangeability violations, but not deployment-time quantization or temperature mismatch. \\

Gui et al.~\cite{gui2024alignment}
& Conformal alignment and trust selection
& Trustworthy output selection, not CP evaluation under inference-configuration mismatch. \\

Zhong et al.~\cite{zhong2025quantized}
& Calibration of quantized LLMs
& Shows quantization affects calibration, motivating our setting, but does not study split-conformal transport across calibration/deployment configurations. \\
\bottomrule
\end{tabular}
\end{table}

%%%%%%%%%%%%%%%%%%%%%%%%%%%%%%%
\section{Supplementary Experimental Details}
\label{app:exp_details}

This section records the benchmark configuration used to generate the supplementary results.

\paragraph{Reporting convention.}
All primary supplementary tables report results separately for each dataset--model pair. Every tabulated performance entry is reported as mean $\pm$ standard deviation across the three random runs. We do not average across models in the primary appendix evidence, since averaging across model families would obscure an important part of the phenomenon under study. When deployment axes are averaged, such as averaging over deployment quantization at fixed temperature or averaging over deployment temperature at fixed quantization, that aggregation is part of the experimental design and is performed within run before summarization across runs.

\subsection{Datasets and split sizes}
\label{app:datasets}

\begin{table}[t]
\centering
\small
\begin{tabular}{lccc}
\toprule
Dataset & $n_{\mathrm{cal}}$ & $n_{\mathrm{test}}$ & $n_{\mathrm{total}}$ \\
\midrule
CommonsenseQA & 6000 & 6000 & 12000 \\
DBPedia-14 & 8000 & 50000 & 58000 \\
MMLU-Pro & 2500 & 8000 & 10500 \\
MobileActions & 2500 & 5000 & 7500 \\
\bottomrule
\end{tabular}
\caption{Dataset sizes used throughout the benchmark.}
\label{tab:dataset_overview}
\end{table}

Table~\ref{tab:dataset_overview} reports the calibration and test sizes used for the four benchmark datasets. These split sizes are fixed throughout the reported experiments.

\subsection{Models, score functions, and deployment grid}
\label{app:models}

The benchmark covers eight small language models:
Qwen2.5-1.5B-Instruct,
Qwen2.5-3B-Instruct,
Qwen2.5-7B-Instruct,
DeepSeek-R1-Distill-Qwen-7B,
Gemma-2-2B-it,
Llama-3.2-3B-Instruct,
Phi-3-mini-4k-instruct,
and Mistral-7B-Instruct-v0.3.

We evaluate three nonconformity scores: APS, LAC, and RAPS. The RAPS hyperparameters are fixed to
\[
\lambda = 0.01,
\qquad
k_{\mathrm{reg}} = 1.
\]

The inference-configuration grid is the Cartesian product of
\[
\mathcal{Q}=\{\mathrm{fp16},\mathrm{int8},\mathrm{nf4}\}
\qquad\text{and}\qquad
\mathcal{T}=\{0.5,1.0,1.5,2.0\},
\]
yielding 12 total deployment configurations.

\subsection{Latent split implementation and setting}
\label{app:runs}

The main-paper latent setting uses a 50/50 Stage-1/Stage-2 split of the calibration data:
\[
d_1 = 0.5.
\]
For the default setting at $\alpha=0.1$, the results use
\[
\alpha_{\mathrm{stage1}} = 0.05,
\qquad
\beta_{\mathrm{stage2}} = 0.05.
\]
Unless otherwise stated, the main appendix tables use the default setting above.

\paragraph{Stage-2 latent-configuration predictor.}
In our implementation, the Stage-2 conformal predictor $\hat{\Gamma}$ is built using a
multinomial logistic-regression router. After the Stage-1 thresholds
$\{\hat q_c\}_{c\in\mathcal{C}}$ are fixed, we compute the set-size signature
\[
    Z(x)=\bigl(|\hat S_c(x)|\bigr)_{c\in\mathcal{C}},
\]
where
\[
    \hat S_c(x)=\{y\in\mathcal{Y}(x):A(\pi,y)\leq \hat q_c\}
\]
is the hypothetical prediction set obtained by applying the Stage-1 threshold for configuration
$c$. The router takes $Z(x)$ as input and predicts the inference configuration
$C\in\mathcal{C}$. To preserve the split-conformal interpretation of the Stage-2 configuration
set, we split the Stage-2 data into a router-training split $D_{2,\mathrm{train}}$ and a
router-calibration split $D_{2,\mathrm{cal}}$. The router is fit on
$D_{2,\mathrm{train}}$ and then held fixed when computing configuration nonconformity
scores on $D_{2,\mathrm{cal}}$. Let $\hat p(c\mid Z(x))$ denote the router's predicted
probability for configuration $c$. We use the configuration nonconformity score
\[
    B(Z(x),C)=1-\hat p(C\mid Z(x)),
\]
so that configurations assigned higher router probability have smaller nonconformity. The
Stage-2 threshold $\hat r_{\beta}$ is the standard split-conformal quantile of
\[
    \{B(Z_i,C_i):(X_i,C_i)\in D_{2,\mathrm{cal}}\}
\]
at miscoverage level $\beta_{\mathrm{stage2}}$. At test time, the plausible latent-configuration
set is
\[
    \hat{\Gamma}(Z(x))
    =
    \{c\in\mathcal{C}:1-\hat p(c\mid Z(x))\leq \hat r_{\beta}\}.
\]
The LATENT-ROBUST threshold is then
\[
    \hat q_{\Gamma}(x)=\max_{c\in\hat{\Gamma}(Z(x))}\hat q_c,
\]
with fallback to $\max_{c\in\mathcal{C}}\hat q_c$ if
$\hat{\Gamma}(Z(x))=\emptyset$.
\subsection{Quantization and temperature implementation.} \label{app:quant-temp}

For each model, we instantiate the same pretrained checkpoint under each quantization mode
$\mathcal{Q} \in \{\mathrm{fp16}, \mathrm{int8}, \mathrm{nf4}\}$. The fp16 setting uses half-precision
inference, while the int8 and nf4 settings use post-training low-bit inference following standard
LLM quantization practice \citep{dettmers2022gpt3,dettmers2023qlora}. In all cases,
quantization changes only the numerical representation used at inference time: the model checkpoint,
training procedure, prompt template, candidate labels, calibration split, and evaluation split are held
fixed. No model weights are fine-tuned or otherwise updated between inference configurations.

Temperature is applied before candidate-label scoring, at the token-logit level. Let
$\ell_{q,k}(v \mid x, y_{<k})$ denote the pre-softmax logit assigned, under quantization mode $q$,
to vocabulary token $v$ at position $k$ when scoring candidate label $y=(y_1,\ldots,y_m)$
conditioned on input $x$ and the candidate prefix $y_{<k}$. For temperature $t>0$, we form the
temperature-scaled token distribution
\[
    p_{q,t}(v \mid x, y_{<k})
    =
    \frac{\exp\!\left(\ell_{q,k}(v \mid x, y_{<k})/t\right)}
    {\sum_{v' \in \mathcal{V}}
    \exp\!\left(\ell_{q,k}(v' \mid x, y_{<k})/t\right)} ,
\]
where $\mathcal{V}$ is the model vocabulary. We then compute the candidate-label score from the
temperature-scaled token probabilities as the length-normalized log-likelihood
\[
    s_{q,t}(y \mid x)
    =
    \frac{1}{m}
    \sum_{k=1}^{m}
    \log p_{q,t}(y_k \mid x, y_{<k}) .
\]
The finite-label predictive distribution used by conformal prediction is obtained by exponentiating
and normalizing these candidate-label log-scores over the candidate set,
\[
    \pi_{q,t}(y \mid x)
    =
    \frac{\exp\!\left(s_{q,t}(y \mid x)\right)}
    {\sum_{y' \in \mathcal{Y}(x)}
    \exp\!\left(s_{q,t}(y' \mid x)\right)} ,
    \qquad y \in \mathcal{Y}(x).
\]
Thus, changing $t$ changes the token-level predictive distribution before candidate-label likelihoods
are computed, while changing $q$ changes the numerical representation used to compute the logits.
The conformal score is then computed from $\pi_{q,t}(\cdot \mid x)$ and the true label. Therefore,
changing $(q,t)$ changes the inference-time score-generating procedure while leaving the task
distribution, trained checkpoint, prompt template, candidate labels, and data splits unchanged.

\subsection{Compute resources}
All experiments were run on a local workstation equipped with an AMD Ryzen Threadripper PRO CPU and a single NVIDIA RTX A6000 GPU with 48GB of GPU memory. Model inference, including fp16, int8, and nf4 configurations, was executed on the GPU. CPU resources were used for data loading, preprocessing, conformal calibration, the logistic-regression router used by LATENT-ROBUST, and result aggregation. The experiments were run sequentially over the dataset--model--configuration grid.

\section{Additional Main-Benchmark Results}
\label{app:main_results}

Unless otherwise stated, all tables in this section use LAC with $\alpha=0.1$ and the main-paper latent setting
\[
\alpha_{\mathrm{stage1}}=\beta_{\mathrm{stage2}}=\alpha/2,
\qquad
d_1=0.5.
\]
Each performance entry reports mean $\pm$ standard deviation across runs for coverage and set size.

\subsection{Deployment fixed to \texorpdfstring{$\mathrm{nf4}@T=0.5$}{nf4@T=0.5}}
\label{app:exp1}

\begin{table*}[t]
\centering
\centerline{
\resizebox{1.4\textwidth}{!}{%
\begin{tabular}{lcccccccccccc}
\toprule
Model & \multicolumn{2}{c}{Matched} & \multicolumn{2}{c}{Mismatch mean} & \multicolumn{2}{c}{Mismatch worst} & \multicolumn{2}{c}{Pooled} & \multicolumn{2}{c}{MaxThr} & \multicolumn{2}{c}{Latent-robust} \\
\cmidrule(lr){2-3} \cmidrule(lr){4-5} \cmidrule(lr){6-7} \cmidrule(lr){8-9} \cmidrule(lr){10-11} \cmidrule(lr){12-13}
 & Cov. & Size & Cov. & Size & Cov. & Size & Cov. & Size & Cov. & Size & Cov. & Size \\
\midrule
Qwen2.5-1.5B & 0.915 $\pm$ 0.023 & 1.876 $\pm$ 0.149 & 0.832 $\pm$ 0.014 & 1.339 $\pm$ 0.044 & 0.797 $\pm$ 0.016 & 1.166 $\pm$ 0.020 & 0.823 $\pm$ 0.006 & 1.255 $\pm$ 0.010 & 0.911 $\pm$ 0.024 & 1.850 $\pm$ 0.149 & 0.962 $\pm$ 0.023 & 2.652 $\pm$ 0.402 \\
Qwen2.5-3B & 0.913 $\pm$ 0.013 & 1.489 $\pm$ 0.049 & 0.846 $\pm$ 0.011 & 1.187 $\pm$ 0.011 & 0.821 $\pm$ 0.008 & 1.086 $\pm$ 0.008 & 0.838 $\pm$ 0.015 & 1.154 $\pm$ 0.019 & 0.910 $\pm$ 0.014 & 1.474 $\pm$ 0.048 & 0.998 $\pm$ 0.002 & 4.957 $\pm$ 0.012 \\
Qwen2.5-7B & 0.907 $\pm$ 0.005 & 1.369 $\pm$ 0.010 & 0.866 $\pm$ 0.012 & 1.153 $\pm$ 0.037 & 0.844 $\pm$ 0.009 & 1.074 $\pm$ 0.020 & 0.861 $\pm$ 0.011 & 1.120 $\pm$ 0.041 & 0.905 $\pm$ 0.015 & 1.353 $\pm$ 0.072 & 0.954 $\pm$ 0.015 & 2.007 $\pm$ 0.198 \\
DeepSeek-R1-7B & 0.895 $\pm$ 0.003 & 2.637 $\pm$ 0.029 & 0.759 $\pm$ 0.008 & 1.675 $\pm$ 0.036 & 0.708 $\pm$ 0.007 & 1.374 $\pm$ 0.022 & 0.752 $\pm$ 0.008 & 1.591 $\pm$ 0.005 & 0.892 $\pm$ 0.004 & 2.616 $\pm$ 0.027 & 0.952 $\pm$ 0.013 & 3.404 $\pm$ 0.198 \\
Gemma-2-2B & 0.919 $\pm$ 0.017 & 2.544 $\pm$ 0.193 & 0.816 $\pm$ 0.013 & 1.655 $\pm$ 0.092 & 0.755 $\pm$ 0.011 & 1.208 $\pm$ 0.037 & 0.799 $\pm$ 0.014 & 1.430 $\pm$ 0.080 & 0.943 $\pm$ 0.014 & 2.956 $\pm$ 0.196 & 0.973 $\pm$ 0.008 & 3.573 $\pm$ 0.139 \\
Llama-3.2-3B & 0.919 $\pm$ 0.012 & 1.993 $\pm$ 0.045 & 0.841 $\pm$ 0.014 & 1.470 $\pm$ 0.020 & 0.805 $\pm$ 0.014 & 1.252 $\pm$ 0.003 & 0.830 $\pm$ 0.013 & 1.372 $\pm$ 0.027 & 0.923 $\pm$ 0.016 & 2.064 $\pm$ 0.116 & 0.967 $\pm$ 0.013 & 2.826 $\pm$ 0.223 \\
Phi-3-mini & 0.908 $\pm$ 0.010 & 1.885 $\pm$ 0.061 & 0.830 $\pm$ 0.007 & 1.406 $\pm$ 0.024 & 0.800 $\pm$ 0.007 & 1.244 $\pm$ 0.029 & 0.822 $\pm$ 0.003 & 1.355 $\pm$ 0.021 & 0.908 $\pm$ 0.010 & 1.887 $\pm$ 0.060 & 0.957 $\pm$ 0.016 & 2.533 $\pm$ 0.273 \\
Mistral-7B & 0.921 $\pm$ 0.007 & 2.504 $\pm$ 0.046 & 0.829 $\pm$ 0.001 & 1.619 $\pm$ 0.015 & 0.781 $\pm$ 0.003 & 1.262 $\pm$ 0.006 & 0.819 $\pm$ 0.009 & 1.475 $\pm$ 0.076 & 0.926 $\pm$ 0.008 & 2.557 $\pm$ 0.009 & 0.970 $\pm$ 0.007 & 3.429 $\pm$ 0.047 \\
\bottomrule
\end{tabular}%
}
}
\caption{Coverage and average prediction-set size for CommonsenseQA with deployment fixed to $\mathrm{nf4}@T=0.5$, using LAC with $\alpha=0.1$. Each entry reports mean $\pm$ standard deviation across runs.}
\label{tab:exp1_commonsense_qa}
\end{table*}
\begin{table*}[t]
\centering
\centerline{
\resizebox{1.4\textwidth}{!}{%
\begin{tabular}{lcccccccccccc}
\toprule
Model & \multicolumn{2}{c}{Matched} & \multicolumn{2}{c}{Mismatch mean} & \multicolumn{2}{c}{Mismatch worst} & \multicolumn{2}{c}{Pooled} & \multicolumn{2}{c}{MaxThr} & \multicolumn{2}{c}{Latent-robust} \\
\cmidrule(lr){2-3} \cmidrule(lr){4-5} \cmidrule(lr){6-7} \cmidrule(lr){8-9} \cmidrule(lr){10-11} \cmidrule(lr){12-13}
 & Cov. & Size & Cov. & Size & Cov. & Size & Cov. & Size & Cov. & Size & Cov. & Size \\
\midrule
Qwen2.5-1.5B & 0.899 $\pm$ 0.006 & 1.231 $\pm$ 0.042 & 0.884 $\pm$ 0.002 & 1.139 $\pm$ 0.013 & 0.876 $\pm$ 0.002 & 1.098 $\pm$ 0.007 & 0.883 $\pm$ 0.001 & 1.130 $\pm$ 0.004 & 0.901 $\pm$ 0.004 & 1.254 $\pm$ 0.033 & 0.947 $\pm$ 0.008 & 2.505 $\pm$ 0.319 \\
Qwen2.5-3B & 0.900 $\pm$ 0.002 & 1.659 $\pm$ 0.019 & 0.795 $\pm$ 0.003 & 1.168 $\pm$ 0.007 & 0.757 $\pm$ 0.001 & 1.061 $\pm$ 0.003 & 0.803 $\pm$ 0.006 & 1.185 $\pm$ 0.016 & 0.900 $\pm$ 0.002 & 1.658 $\pm$ 0.018 & 0.946 $\pm$ 0.002 & 2.342 $\pm$ 0.034 \\
Qwen2.5-7B & 0.908 $\pm$ 0.000 & 0.916 $\pm$ 0.000 & 0.946 $\pm$ 0.001 & 0.966 $\pm$ 0.000 & 0.881 $\pm$ 0.001 & 0.888 $\pm$ 0.000 & 0.964 $\pm$ 0.000 & 0.988 $\pm$ 0.000 & 0.965 $\pm$ 0.000 & 0.991 $\pm$ 0.000 & 0.966 $\pm$ 0.000 & 0.993 $\pm$ 0.000 \\
DeepSeek-R1-7B & 0.902 $\pm$ 0.004 & 1.701 $\pm$ 0.036 & 0.871 $\pm$ 0.002 & 1.472 $\pm$ 0.012 & 0.838 $\pm$ 0.002 & 1.251 $\pm$ 0.004 & 0.875 $\pm$ 0.001 & 1.452 $\pm$ 0.009 & 0.924 $\pm$ 0.002 & 2.024 $\pm$ 0.025 & 0.956 $\pm$ 0.003 & 3.026 $\pm$ 0.200 \\
Gemma-2-2B & 0.902 $\pm$ 0.004 & 3.485 $\pm$ 0.134 & 0.799 $\pm$ 0.001 & 2.251 $\pm$ 0.046 & 0.744 $\pm$ 0.000 & 1.232 $\pm$ 0.002 & 0.782 $\pm$ 0.000 & 1.611 $\pm$ 0.003 & 0.940 $\pm$ 0.004 & 5.953 $\pm$ 0.237 & 0.987 $\pm$ 0.005 & 10.787 $\pm$ 0.672 \\
Llama-3.2-3B & 0.899 $\pm$ 0.002 & 1.750 $\pm$ 0.024 & 0.829 $\pm$ 0.001 & 1.216 $\pm$ 0.005 & 0.815 $\pm$ 0.000 & 1.152 $\pm$ 0.002 & 0.833 $\pm$ 0.000 & 1.226 $\pm$ 0.002 & 0.898 $\pm$ 0.002 & 1.748 $\pm$ 0.026 & 0.947 $\pm$ 0.004 & 3.576 $\pm$ 0.266 \\
Phi-3-mini & 0.898 $\pm$ 0.001 & 0.970 $\pm$ 0.002 & 0.888 $\pm$ 0.001 & 0.955 $\pm$ 0.001 & 0.837 $\pm$ 0.003 & 0.878 $\pm$ 0.004 & 0.906 $\pm$ 0.000 & 0.983 $\pm$ 0.000 & 0.921 $\pm$ 0.000 & 1.015 $\pm$ 0.000 & 0.937 $\pm$ 0.004 & 1.060 $\pm$ 0.013 \\
Mistral-7B & 0.901 $\pm$ 0.004 & 1.133 $\pm$ 0.012 & 0.892 $\pm$ 0.002 & 1.115 $\pm$ 0.007 & 0.870 $\pm$ 0.003 & 1.045 $\pm$ 0.006 & 0.891 $\pm$ 0.003 & 1.103 $\pm$ 0.005 & 0.926 $\pm$ 0.004 & 1.264 $\pm$ 0.023 & 0.961 $\pm$ 0.004 & 2.088 $\pm$ 0.186 \\
\bottomrule
\end{tabular}%
}
}
\caption{Coverage and average prediction-set size for DBPedia-14 with deployment fixed to $\mathrm{nf4}@T=0.5$, using LAC with $\alpha=0.1$. Each entry reports mean $\pm$ standard deviation across runs.}
\label{tab:exp1_dbpedia_14}
\end{table*}
\begin{table*}[t]
\centering
%\small
\centerline{
\resizebox{1.4\textwidth}{!}{%
\begin{tabular}{lcccccccccccc}
\toprule
Model & \multicolumn{2}{c}{Matched} & \multicolumn{2}{c}{Mismatch mean} & \multicolumn{2}{c}{Mismatch worst} & \multicolumn{2}{c}{Pooled} & \multicolumn{2}{c}{MaxThr} & \multicolumn{2}{c}{Latent-robust} \\
\cmidrule(lr){2-3} \cmidrule(lr){4-5} \cmidrule(lr){6-7} \cmidrule(lr){8-9} \cmidrule(lr){10-11} \cmidrule(lr){12-13}
 & Cov. & Size & Cov. & Size & Cov. & Size & Cov. & Size & Cov. & Size & Cov. & Size \\
\midrule
Qwen2.5-1.5B & 0.901 $\pm$ 0.009 & 7.765 $\pm$ 0.072 & 0.597 $\pm$ 0.005 & 4.293 $\pm$ 0.019 & 0.468 $\pm$ 0.003 & 2.876 $\pm$ 0.004 & 0.650 $\pm$ 0.025 & 4.829 $\pm$ 0.236 & 0.908 $\pm$ 0.007 & 7.860 $\pm$ 0.043 & 0.954 $\pm$ 0.003 & 8.533 $\pm$ 0.022 \\
Qwen2.5-3B & 0.899 $\pm$ 0.011 & 6.963 $\pm$ 0.139 & 0.628 $\pm$ 0.009 & 3.870 $\pm$ 0.175 & 0.480 $\pm$ 0.004 & 2.261 $\pm$ 0.033 & 0.670 $\pm$ 0.004 & 4.095 $\pm$ 0.063 & 0.967 $\pm$ 0.056 & 8.655 $\pm$ 1.406 & 1.000 $\pm$ 0.000 & 9.464 $\pm$ 0.017 \\
Qwen2.5-7B & 0.899 $\pm$ 0.007 & 6.167 $\pm$ 0.092 & 0.678 $\pm$ 0.002 & 3.454 $\pm$ 0.021 & 0.563 $\pm$ 0.002 & 2.286 $\pm$ 0.005 & 0.706 $\pm$ 0.009 & 3.586 $\pm$ 0.082 & 0.898 $\pm$ 0.008 & 6.148 $\pm$ 0.092 & 1.000 $\pm$ 0.001 & 9.461 $\pm$ 0.018 \\
DeepSeek-R1-7B & 0.904 $\pm$ 0.006 & 7.643 $\pm$ 0.050 & 0.583 $\pm$ 0.007 & 3.949 $\pm$ 0.056 & 0.443 $\pm$ 0.005 & 2.508 $\pm$ 0.032 & 0.651 $\pm$ 0.017 & 4.462 $\pm$ 0.165 & 0.908 $\pm$ 0.009 & 7.719 $\pm$ 0.099 & 0.958 $\pm$ 0.007 & 8.491 $\pm$ 0.095 \\
Gemma-2-2B & 0.899 $\pm$ 0.003 & 7.495 $\pm$ 0.062 & 0.663 $\pm$ 0.002 & 4.717 $\pm$ 0.028 & 0.507 $\pm$ 0.003 & 3.037 $\pm$ 0.020 & 0.727 $\pm$ 0.001 & 5.293 $\pm$ 0.020 & 0.935 $\pm$ 0.005 & 8.028 $\pm$ 0.074 & 0.972 $\pm$ 0.001 & 8.656 $\pm$ 0.033 \\
Llama-3.2-3B & 0.902 $\pm$ 0.003 & 7.239 $\pm$ 0.037 & 0.644 $\pm$ 0.004 & 4.213 $\pm$ 0.048 & 0.525 $\pm$ 0.006 & 2.997 $\pm$ 0.050 & 0.694 $\pm$ 0.009 & 4.629 $\pm$ 0.100 & 0.902 $\pm$ 0.003 & 7.232 $\pm$ 0.043 & 0.957 $\pm$ 0.004 & 8.152 $\pm$ 0.075 \\
Phi-3-mini & 0.904 $\pm$ 0.003 & 6.374 $\pm$ 0.027 & 0.671 $\pm$ 0.004 & 3.601 $\pm$ 0.034 & 0.568 $\pm$ 0.005 & 2.499 $\pm$ 0.028 & 0.696 $\pm$ 0.006 & 3.745 $\pm$ 0.058 & 0.903 $\pm$ 0.003 & 6.362 $\pm$ 0.032 & 0.957 $\pm$ 0.004 & 7.425 $\pm$ 0.112 \\
Mistral-7B & 0.906 $\pm$ 0.008 & 7.217 $\pm$ 0.087 & 0.650 $\pm$ 0.004 & 3.992 $\pm$ 0.012 & 0.510 $\pm$ 0.006 & 2.421 $\pm$ 0.012 & 0.694 $\pm$ 0.005 & 4.305 $\pm$ 0.045 & 0.927 $\pm$ 0.004 & 7.605 $\pm$ 0.023 & 0.968 $\pm$ 0.001 & 8.425 $\pm$ 0.061 \\
\bottomrule
\end{tabular}%
}
}
\caption{Coverage and average prediction-set size for MMLU-Pro with deployment fixed to $\mathrm{nf4}@T=0.5$, using LAC with $\alpha=0.1$. Each entry reports mean $\pm$ standard deviation across runs.}
\label{tab:exp1_mmlu_pro}
\end{table*}
\begin{table*}[t]
\centering
\centerline{
\resizebox{1.4\textwidth}{!}{%
\begin{tabular}{lcccccccccccc}
\toprule
Model & \multicolumn{2}{c}{Matched} & \multicolumn{2}{c}{Mismatch mean} & \multicolumn{2}{c}{Mismatch worst} & \multicolumn{2}{c}{Pooled} & \multicolumn{2}{c}{MaxThr} & \multicolumn{2}{c}{Latent-robust} \\
\cmidrule(lr){2-3} \cmidrule(lr){4-5} \cmidrule(lr){6-7} \cmidrule(lr){8-9} \cmidrule(lr){10-11} \cmidrule(lr){12-13}
 & Cov. & Size & Cov. & Size & Cov. & Size & Cov. & Size & Cov. & Size & Cov. & Size \\
\midrule
Qwen2.5-1.5B & 0.907 $\pm$ 0.000 & 0.975 $\pm$ 0.002 & 0.905 $\pm$ 0.002 & 0.974 $\pm$ 0.001 & 0.866 $\pm$ 0.003 & 0.909 $\pm$ 0.005 & 0.915 $\pm$ 0.001 & 0.990 $\pm$ 0.004 & 0.927 $\pm$ 0.000 & 1.012 $\pm$ 0.002 & 0.940 $\pm$ 0.010 & 1.049 $\pm$ 0.029 \\
Qwen2.5-3B & 0.898 $\pm$ 0.006 & 1.017 $\pm$ 0.010 & 0.899 $\pm$ 0.002 & 1.019 $\pm$ 0.003 & 0.893 $\pm$ 0.001 & 1.006 $\pm$ 0.001 & 0.897 $\pm$ 0.002 & 1.014 $\pm$ 0.003 & 0.913 $\pm$ 0.004 & 1.052 $\pm$ 0.008 & 0.961 $\pm$ 0.001 & 1.265 $\pm$ 0.008 \\
Qwen2.5-7B & 0.902 $\pm$ 0.003 & 0.919 $\pm$ 0.003 & 0.939 $\pm$ 0.002 & 0.969 $\pm$ 0.003 & 0.911 $\pm$ 0.009 & 0.930 $\pm$ 0.012 & 0.944 $\pm$ 0.001 & 0.976 $\pm$ 0.000 & 0.950 $\pm$ 0.001 & 0.986 $\pm$ 0.002 & 0.955 $\pm$ 0.001 & 0.998 $\pm$ 0.000 \\
DeepSeek-R1-7B & 0.899 $\pm$ 0.003 & 1.119 $\pm$ 0.003 & 0.880 $\pm$ 0.000 & 1.072 $\pm$ 0.005 & 0.871 $\pm$ 0.001 & 1.049 $\pm$ 0.003 & 0.879 $\pm$ 0.002 & 1.067 $\pm$ 0.001 & 0.900 $\pm$ 0.001 & 1.121 $\pm$ 0.003 & 0.946 $\pm$ 0.006 & 1.277 $\pm$ 0.035 \\
Gemma-2-2B & 0.900 $\pm$ 0.001 & 1.202 $\pm$ 0.009 & 0.870 $\pm$ 0.002 & 1.070 $\pm$ 0.002 & 0.857 $\pm$ 0.002 & 1.030 $\pm$ 0.002 & 0.868 $\pm$ 0.003 & 1.053 $\pm$ 0.005 & 0.899 $\pm$ 0.001 & 1.195 $\pm$ 0.008 & 0.942 $\pm$ 0.003 & 1.674 $\pm$ 0.014 \\
Llama-3.2-3B & 0.902 $\pm$ 0.001 & 1.169 $\pm$ 0.007 & 0.871 $\pm$ 0.001 & 1.081 $\pm$ 0.007 & 0.858 $\pm$ 0.001 & 1.048 $\pm$ 0.008 & 0.865 $\pm$ 0.001 & 1.066 $\pm$ 0.004 & 0.903 $\pm$ 0.001 & 1.170 $\pm$ 0.006 & 0.943 $\pm$ 0.007 & 1.451 $\pm$ 0.084 \\
Phi-3-mini & 0.904 $\pm$ 0.006 & 0.999 $\pm$ 0.008 & 0.909 $\pm$ 0.003 & 1.009 $\pm$ 0.001 & 0.906 $\pm$ 0.007 & 1.002 $\pm$ 0.010 & 0.909 $\pm$ 0.002 & 1.009 $\pm$ 0.000 & 0.909 $\pm$ 0.002 & 1.009 $\pm$ 0.000 & 0.940 $\pm$ 0.007 & 1.086 $\pm$ 0.015 \\
Mistral-7B & 0.899 $\pm$ 0.005 & 0.986 $\pm$ 0.007 & 0.902 $\pm$ 0.003 & 0.993 $\pm$ 0.002 & 0.897 $\pm$ 0.003 & 0.983 $\pm$ 0.003 & 0.903 $\pm$ 0.002 & 0.994 $\pm$ 0.004 & 0.904 $\pm$ 0.003 & 0.997 $\pm$ 0.001 & 0.951 $\pm$ 0.005 & 1.166 $\pm$ 0.016 \\
\bottomrule
\end{tabular}%
}
}
\caption{Coverage and average prediction-set size for MobileActions with deployment fixed to $\mathrm{nf4}@T=0.5$, using LAC with $\alpha=0.1$. Each entry reports mean $\pm$ standard deviation across runs.}
\label{tab:exp1_mobile_actions}
\end{table*}

Tables~\ref{tab:exp1_commonsense_qa}--\ref{tab:exp1_mobile_actions} correspond to the main deployment setting emphasized in the paper. Several points are visible at the per-model level. First, the matched baseline generally remains close to target, consistent with the standard split-conformal setting. Second, the degradation under calibration--deployment mismatch is often substantial, but its magnitude varies sharply across model families. Third, the strongest failures remain concentrated in MMLU-Pro, where Table~\ref{tab:exp1_mmlu_pro} shows that multiple models exhibit large drops from \textsc{Matched} to \textsc{Mismatch Mean} and especially to \textsc{Mismatch Worst}. By contrast, Table~\ref{tab:exp1_mobile_actions} shows that MobileActions is more stable for many models, though worst-case mismatch can still fall below target.

The merged presentation in Tables~\ref{tab:exp1_commonsense_qa}--\ref{tab:exp1_mobile_actions} also makes the reliability--efficiency tradeoff explicit. In settings where \textsc{MaxThr} and \textsc{Latent-Robust} recover coverage, they generally do so by increasing average prediction-set size. This is especially visible on the harder tasks, where robust coverage recovery is paired with a nontrivial efficiency cost. The attached standard deviations are typically small relative to the method gaps, indicating that the qualitative conclusions are not driven by a single random split.

\subsection{Deployment temperature fixed to \texorpdfstring{$T=0.5$}{T=0.5}}
\label{app:exp2}

\begin{table*}[t]
\centering
\centerline{
\resizebox{1.4\textwidth}{!}{%
\begin{tabular}{lcccccccccccc}
\toprule
Model & \multicolumn{2}{c}{Matched} & \multicolumn{2}{c}{Mismatch mean} & \multicolumn{2}{c}{Mismatch worst} & \multicolumn{2}{c}{Pooled} & \multicolumn{2}{c}{MaxThr} & \multicolumn{2}{c}{Latent-robust} \\
\cmidrule(lr){2-3} \cmidrule(lr){4-5} \cmidrule(lr){6-7} \cmidrule(lr){8-9} \cmidrule(lr){10-11} \cmidrule(lr){12-13}
 & Cov. & Size & Cov. & Size & Cov. & Size & Cov. & Size & Cov. & Size & Cov. & Size \\
\midrule
Qwen2.5-1.5B & 0.914 $\pm$ 0.020 & 1.765 $\pm$ 0.129 & 0.846 $\pm$ 0.011 & 1.344 $\pm$ 0.044 & 0.814 $\pm$ 0.010 & 1.181 $\pm$ 0.015 & 0.834 $\pm$ 0.007 & 1.269 $\pm$ 0.017 & 0.916 $\pm$ 0.021 & 1.788 $\pm$ 0.150 & 0.957 $\pm$ 0.020 & 2.456 $\pm$ 0.353 \\
Qwen2.5-3B & 0.907 $\pm$ 0.015 & 1.422 $\pm$ 0.047 & 0.857 $\pm$ 0.009 & 1.184 $\pm$ 0.011 & 0.831 $\pm$ 0.008 & 1.079 $\pm$ 0.008 & 0.851 $\pm$ 0.012 & 1.150 $\pm$ 0.020 & 0.915 $\pm$ 0.013 & 1.464 $\pm$ 0.032 & 0.997 $\pm$ 0.001 & 4.957 $\pm$ 0.017 \\
Qwen2.5-7B & 0.908 $\pm$ 0.014 & 1.303 $\pm$ 0.051 & 0.872 $\pm$ 0.013 & 1.138 $\pm$ 0.029 & 0.850 $\pm$ 0.009 & 1.066 $\pm$ 0.016 & 0.867 $\pm$ 0.017 & 1.109 $\pm$ 0.036 & 0.908 $\pm$ 0.016 & 1.307 $\pm$ 0.065 & 0.956 $\pm$ 0.013 & 1.907 $\pm$ 0.194 \\
DeepSeek-R1-7B & 0.901 $\pm$ 0.010 & 2.657 $\pm$ 0.064 & 0.774 $\pm$ 0.013 & 1.758 $\pm$ 0.042 & 0.719 $\pm$ 0.011 & 1.418 $\pm$ 0.031 & 0.769 $\pm$ 0.012 & 1.664 $\pm$ 0.008 & 0.906 $\pm$ 0.004 & 2.744 $\pm$ 0.038 & 0.950 $\pm$ 0.009 & 3.521 $\pm$ 0.200 \\
Gemma-2-2B & 0.916 $\pm$ 0.013 & 2.445 $\pm$ 0.186 & 0.807 $\pm$ 0.012 & 1.516 $\pm$ 0.073 & 0.759 $\pm$ 0.011 & 1.176 $\pm$ 0.030 & 0.790 $\pm$ 0.011 & 1.346 $\pm$ 0.064 & 0.925 $\pm$ 0.013 & 2.618 $\pm$ 0.187 & 0.958 $\pm$ 0.014 & 3.233 $\pm$ 0.176 \\
Llama-3.2-3B & 0.918 $\pm$ 0.013 & 1.938 $\pm$ 0.045 & 0.840 $\pm$ 0.015 & 1.432 $\pm$ 0.021 & 0.807 $\pm$ 0.016 & 1.240 $\pm$ 0.013 & 0.828 $\pm$ 0.013 & 1.342 $\pm$ 0.028 & 0.919 $\pm$ 0.018 & 1.967 $\pm$ 0.107 & 0.965 $\pm$ 0.011 & 2.661 $\pm$ 0.191 \\
Phi-3-mini & 0.911 $\pm$ 0.009 & 1.779 $\pm$ 0.039 & 0.844 $\pm$ 0.009 & 1.377 $\pm$ 0.025 & 0.813 $\pm$ 0.008 & 1.220 $\pm$ 0.024 & 0.839 $\pm$ 0.006 & 1.331 $\pm$ 0.020 & 0.913 $\pm$ 0.010 & 1.810 $\pm$ 0.051 & 0.956 $\pm$ 0.014 & 2.474 $\pm$ 0.270 \\
Mistral-7B & 0.918 $\pm$ 0.005 & 2.498 $\pm$ 0.039 & 0.822 $\pm$ 0.002 & 1.597 $\pm$ 0.016 & 0.776 $\pm$ 0.003 & 1.266 $\pm$ 0.012 & 0.813 $\pm$ 0.009 & 1.454 $\pm$ 0.058 & 0.919 $\pm$ 0.005 & 2.504 $\pm$ 0.015 & 0.964 $\pm$ 0.006 & 3.329 $\pm$ 0.031 \\
\bottomrule
\end{tabular}%
}
}
\caption{Coverage and average prediction-set size for CommonsenseQA with deployment temperature fixed to $T=0.5$ and deployment quantization averaged within run, using LAC with $\alpha=0.1$. Each entry reports mean $\pm$ standard deviation across runs.}
\label{tab:exp2_commonsense_qa}
\end{table*}
\begin{table*}[t]
\centering
\centerline{
\resizebox{1.4\textwidth}{!}{%
\begin{tabular}{lcccccccccccc}
\toprule
Model & \multicolumn{2}{c}{Matched} & \multicolumn{2}{c}{Mismatch mean} & \multicolumn{2}{c}{Mismatch worst} & \multicolumn{2}{c}{Pooled} & \multicolumn{2}{c}{MaxThr} & \multicolumn{2}{c}{Latent-robust} \\
\cmidrule(lr){2-3} \cmidrule(lr){4-5} \cmidrule(lr){6-7} \cmidrule(lr){8-9} \cmidrule(lr){10-11} \cmidrule(lr){12-13}
 & Cov. & Size & Cov. & Size & Cov. & Size & Cov. & Size & Cov. & Size & Cov. & Size \\
\midrule
Qwen2.5-1.5B & 0.899 $\pm$ 0.005 & 1.196 $\pm$ 0.027 & 0.885 $\pm$ 0.003 & 1.121 $\pm$ 0.010 & 0.878 $\pm$ 0.002 & 1.087 $\pm$ 0.006 & 0.884 $\pm$ 0.001 & 1.114 $\pm$ 0.004 & 0.902 $\pm$ 0.004 & 1.214 $\pm$ 0.025 & 0.948 $\pm$ 0.006 & 2.066 $\pm$ 0.207 \\
Qwen2.5-3B & 0.899 $\pm$ 0.002 & 1.306 $\pm$ 0.011 & 0.861 $\pm$ 0.002 & 1.140 $\pm$ 0.006 & 0.834 $\pm$ 0.001 & 1.047 $\pm$ 0.003 & 0.864 $\pm$ 0.004 & 1.135 $\pm$ 0.012 & 0.929 $\pm$ 0.001 & 1.494 $\pm$ 0.014 & 0.958 $\pm$ 0.001 & 2.023 $\pm$ 0.030 \\
Qwen2.5-7B & 0.912 $\pm$ 0.001 & 0.921 $\pm$ 0.000 & 0.952 $\pm$ 0.001 & 0.971 $\pm$ 0.000 & 0.903 $\pm$ 0.001 & 0.911 $\pm$ 0.000 & 0.965 $\pm$ 0.000 & 0.988 $\pm$ 0.000 & 0.966 $\pm$ 0.000 & 0.990 $\pm$ 0.000 & 0.967 $\pm$ 0.000 & 0.993 $\pm$ 0.000 \\
DeepSeek-R1-7B & 0.903 $\pm$ 0.003 & 1.868 $\pm$ 0.022 & 0.833 $\pm$ 0.003 & 1.482 $\pm$ 0.012 & 0.790 $\pm$ 0.003 & 1.285 $\pm$ 0.006 & 0.840 $\pm$ 0.001 & 1.481 $\pm$ 0.010 & 0.913 $\pm$ 0.002 & 2.014 $\pm$ 0.024 & 0.952 $\pm$ 0.004 & 3.132 $\pm$ 0.237 \\
Gemma-2-2B & 0.902 $\pm$ 0.003 & 3.926 $\pm$ 0.157 & 0.777 $\pm$ 0.002 & 1.827 $\pm$ 0.033 & 0.731 $\pm$ 0.001 & 1.179 $\pm$ 0.001 & 0.760 $\pm$ 0.001 & 1.400 $\pm$ 0.002 & 0.916 $\pm$ 0.003 & 4.822 $\pm$ 0.196 & 0.966 $\pm$ 0.008 & 9.508 $\pm$ 0.733 \\
Llama-3.2-3B & 0.900 $\pm$ 0.001 & 1.367 $\pm$ 0.002 & 0.876 $\pm$ 0.000 & 1.201 $\pm$ 0.002 & 0.861 $\pm$ 0.001 & 1.118 $\pm$ 0.003 & 0.877 $\pm$ 0.000 & 1.184 $\pm$ 0.003 & 0.926 $\pm$ 0.001 & 1.650 $\pm$ 0.023 & 0.963 $\pm$ 0.003 & 3.282 $\pm$ 0.250 \\
Phi-3-mini & 0.899 $\pm$ 0.002 & 0.945 $\pm$ 0.002 & 0.920 $\pm$ 0.001 & 0.975 $\pm$ 0.001 & 0.879 $\pm$ 0.003 & 0.915 $\pm$ 0.004 & 0.930 $\pm$ 0.000 & 0.990 $\pm$ 0.001 & 0.938 $\pm$ 0.000 & 1.008 $\pm$ 0.000 & 0.948 $\pm$ 0.002 & 1.035 $\pm$ 0.007 \\
Mistral-7B & 0.900 $\pm$ 0.004 & 1.259 $\pm$ 0.017 & 0.866 $\pm$ 0.003 & 1.122 $\pm$ 0.007 & 0.842 $\pm$ 0.003 & 1.049 $\pm$ 0.005 & 0.867 $\pm$ 0.003 & 1.115 $\pm$ 0.005 & 0.909 $\pm$ 0.004 & 1.306 $\pm$ 0.025 & 0.953 $\pm$ 0.004 & 2.141 $\pm$ 0.174 \\
\bottomrule
\end{tabular}%
}
}
\caption{Coverage and average prediction-set size for DBPedia-14 with deployment temperature fixed to $T=0.5$ and deployment quantization averaged within run, using LAC with $\alpha=0.1$. Each entry reports mean $\pm$ standard deviation across runs.}
\label{tab:exp2_dbpedia_14}
\end{table*}
\begin{table*}[t]
\centering
\centerline{
\resizebox{1.4\textwidth}{!}{%
\begin{tabular}{lcccccccccccc}
\toprule
Model & \multicolumn{2}{c}{Matched} & \multicolumn{2}{c}{Mismatch mean} & \multicolumn{2}{c}{Mismatch worst} & \multicolumn{2}{c}{Pooled} & \multicolumn{2}{c}{MaxThr} & \multicolumn{2}{c}{Latent-robust} \\
\cmidrule(lr){2-3} \cmidrule(lr){4-5} \cmidrule(lr){6-7} \cmidrule(lr){8-9} \cmidrule(lr){10-11} \cmidrule(lr){12-13}
 & Cov. & Size & Cov. & Size & Cov. & Size & Cov. & Size & Cov. & Size & Cov. & Size \\
\midrule
Qwen2.5-1.5B & 0.902 $\pm$ 0.006 & 7.548 $\pm$ 0.037 & 0.612 $\pm$ 0.006 & 4.179 $\pm$ 0.023 & 0.484 $\pm$ 0.006 & 2.770 $\pm$ 0.013 & 0.671 $\pm$ 0.023 & 4.761 $\pm$ 0.223 & 0.907 $\pm$ 0.006 & 7.629 $\pm$ 0.046 & 0.955 $\pm$ 0.003 & 8.385 $\pm$ 0.023 \\
Qwen2.5-3B & 0.923 $\pm$ 0.026 & 7.447 $\pm$ 0.577 & 0.642 $\pm$ 0.007 & 3.866 $\pm$ 0.140 & 0.501 $\pm$ 0.003 & 2.299 $\pm$ 0.036 & 0.686 $\pm$ 0.004 & 4.156 $\pm$ 0.059 & 0.969 $\pm$ 0.053 & 8.657 $\pm$ 1.404 & 1.000 $\pm$ 0.000 & 9.464 $\pm$ 0.017 \\
Qwen2.5-7B & 0.899 $\pm$ 0.005 & 6.070 $\pm$ 0.069 & 0.680 $\pm$ 0.003 & 3.389 $\pm$ 0.022 & 0.565 $\pm$ 0.002 & 2.204 $\pm$ 0.005 & 0.709 $\pm$ 0.009 & 3.510 $\pm$ 0.081 & 0.902 $\pm$ 0.007 & 6.114 $\pm$ 0.090 & 1.000 $\pm$ 0.001 & 9.459 $\pm$ 0.020 \\
DeepSeek-R1-7B & 0.908 $\pm$ 0.009 & 7.682 $\pm$ 0.080 & 0.591 $\pm$ 0.007 & 4.016 $\pm$ 0.049 & 0.454 $\pm$ 0.004 & 2.580 $\pm$ 0.027 & 0.659 $\pm$ 0.016 & 4.549 $\pm$ 0.161 & 0.911 $\pm$ 0.010 & 7.729 $\pm$ 0.092 & 0.957 $\pm$ 0.008 & 8.461 $\pm$ 0.087 \\
Gemma-2-2B & 0.899 $\pm$ 0.006 & 7.492 $\pm$ 0.081 & 0.638 $\pm$ 0.002 & 4.359 $\pm$ 0.023 & 0.497 $\pm$ 0.002 & 2.833 $\pm$ 0.016 & 0.695 $\pm$ 0.001 & 4.843 $\pm$ 0.021 & 0.911 $\pm$ 0.007 & 7.673 $\pm$ 0.081 & 0.959 $\pm$ 0.002 & 8.405 $\pm$ 0.036 \\
Llama-3.2-3B & 0.903 $\pm$ 0.003 & 7.075 $\pm$ 0.028 & 0.692 $\pm$ 0.003 & 4.536 $\pm$ 0.036 & 0.573 $\pm$ 0.005 & 3.234 $\pm$ 0.046 & 0.743 $\pm$ 0.008 & 4.995 $\pm$ 0.096 & 0.921 $\pm$ 0.002 & 7.392 $\pm$ 0.034 & 0.964 $\pm$ 0.004 & 8.195 $\pm$ 0.060 \\
Phi-3-mini & 0.903 $\pm$ 0.004 & 6.226 $\pm$ 0.068 & 0.687 $\pm$ 0.003 & 3.629 $\pm$ 0.033 & 0.585 $\pm$ 0.004 & 2.535 $\pm$ 0.028 & 0.713 $\pm$ 0.005 & 3.776 $\pm$ 0.055 & 0.908 $\pm$ 0.003 & 6.327 $\pm$ 0.038 & 0.959 $\pm$ 0.004 & 7.354 $\pm$ 0.108 \\
Mistral-7B & 0.908 $\pm$ 0.005 & 7.272 $\pm$ 0.029 & 0.641 $\pm$ 0.004 & 3.840 $\pm$ 0.013 & 0.506 $\pm$ 0.006 & 2.346 $\pm$ 0.013 & 0.688 $\pm$ 0.005 & 4.141 $\pm$ 0.035 & 0.916 $\pm$ 0.004 & 7.408 $\pm$ 0.021 & 0.960 $\pm$ 0.001 & 8.266 $\pm$ 0.065 \\
\bottomrule
\end{tabular}%
}
}
\caption{Coverage and average prediction-set size for MMLU-Pro with deployment temperature fixed to $T=0.5$ and deployment quantization averaged within run, using LAC with $\alpha=0.1$. Each entry reports mean $\pm$ standard deviation across runs.}
\label{tab:exp2_mmlu_pro}
\end{table*}
\begin{table*}[t]
\centering
\centerline{
\resizebox{1.4\textwidth}{!}{%
\begin{tabular}{lcccccccccccc}
\toprule
Model & \multicolumn{2}{c}{Matched} & \multicolumn{2}{c}{Mismatch mean} & \multicolumn{2}{c}{Mismatch worst} & \multicolumn{2}{c}{Pooled} & \multicolumn{2}{c}{MaxThr} & \multicolumn{2}{c}{Latent-robust} \\
\cmidrule(lr){2-3} \cmidrule(lr){4-5} \cmidrule(lr){6-7} \cmidrule(lr){8-9} \cmidrule(lr){10-11} \cmidrule(lr){12-13}
 & Cov. & Size & Cov. & Size & Cov. & Size & Cov. & Size & Cov. & Size & Cov. & Size \\
\midrule
Qwen2.5-1.5B & 0.904 $\pm$ 0.000 & 0.954 $\pm$ 0.001 & 0.919 $\pm$ 0.001 & 0.981 $\pm$ 0.001 & 0.891 $\pm$ 0.001 & 0.932 $\pm$ 0.001 & 0.925 $\pm$ 0.001 & 0.992 $\pm$ 0.002 & 0.934 $\pm$ 0.000 & 1.010 $\pm$ 0.001 & 0.946 $\pm$ 0.008 & 1.039 $\pm$ 0.021 \\
Qwen2.5-3B & 0.900 $\pm$ 0.004 & 1.037 $\pm$ 0.006 & 0.892 $\pm$ 0.002 & 1.017 $\pm$ 0.003 & 0.887 $\pm$ 0.001 & 1.006 $\pm$ 0.000 & 0.891 $\pm$ 0.001 & 1.014 $\pm$ 0.002 & 0.907 $\pm$ 0.004 & 1.054 $\pm$ 0.008 & 0.953 $\pm$ 0.001 & 1.266 $\pm$ 0.010 \\
Qwen2.5-7B & 0.902 $\pm$ 0.006 & 0.927 $\pm$ 0.007 & 0.930 $\pm$ 0.002 & 0.967 $\pm$ 0.003 & 0.898 $\pm$ 0.006 & 0.922 $\pm$ 0.007 & 0.936 $\pm$ 0.001 & 0.975 $\pm$ 0.000 & 0.943 $\pm$ 0.002 & 0.986 $\pm$ 0.002 & 0.949 $\pm$ 0.001 & 0.998 $\pm$ 0.000 \\
DeepSeek-R1-7B & 0.900 $\pm$ 0.001 & 1.124 $\pm$ 0.008 & 0.882 $\pm$ 0.001 & 1.076 $\pm$ 0.006 & 0.872 $\pm$ 0.001 & 1.053 $\pm$ 0.004 & 0.880 $\pm$ 0.001 & 1.070 $\pm$ 0.003 & 0.902 $\pm$ 0.001 & 1.129 $\pm$ 0.003 & 0.948 $\pm$ 0.008 & 1.289 $\pm$ 0.039 \\
Gemma-2-2B & 0.900 $\pm$ 0.001 & 1.142 $\pm$ 0.005 & 0.877 $\pm$ 0.001 & 1.056 $\pm$ 0.002 & 0.866 $\pm$ 0.001 & 1.025 $\pm$ 0.002 & 0.874 $\pm$ 0.002 & 1.041 $\pm$ 0.003 & 0.902 $\pm$ 0.001 & 1.147 $\pm$ 0.006 & 0.946 $\pm$ 0.001 & 1.547 $\pm$ 0.005 \\
Llama-3.2-3B & 0.902 $\pm$ 0.002 & 1.158 $\pm$ 0.009 & 0.870 $\pm$ 0.001 & 1.078 $\pm$ 0.004 & 0.857 $\pm$ 0.001 & 1.046 $\pm$ 0.005 & 0.865 $\pm$ 0.001 & 1.064 $\pm$ 0.002 & 0.903 $\pm$ 0.001 & 1.162 $\pm$ 0.005 & 0.946 $\pm$ 0.007 & 1.407 $\pm$ 0.069 \\
Phi-3-mini & 0.904 $\pm$ 0.006 & 1.005 $\pm$ 0.007 & 0.906 $\pm$ 0.004 & 1.009 $\pm$ 0.003 & 0.899 $\pm$ 0.008 & 0.994 $\pm$ 0.011 & 0.907 $\pm$ 0.003 & 1.011 $\pm$ 0.002 & 0.908 $\pm$ 0.003 & 1.014 $\pm$ 0.001 & 0.940 $\pm$ 0.006 & 1.094 $\pm$ 0.018 \\
Mistral-7B & 0.900 $\pm$ 0.002 & 0.985 $\pm$ 0.005 & 0.904 $\pm$ 0.002 & 0.994 $\pm$ 0.002 & 0.899 $\pm$ 0.002 & 0.984 $\pm$ 0.003 & 0.905 $\pm$ 0.001 & 0.995 $\pm$ 0.003 & 0.906 $\pm$ 0.002 & 0.998 $\pm$ 0.001 & 0.951 $\pm$ 0.004 & 1.166 $\pm$ 0.017 \\
\bottomrule
\end{tabular}%
}
}
\caption{Coverage and average prediction-set size for MobileActions with deployment temperature fixed to $T=0.5$ and deployment quantization averaged within run, using LAC with $\alpha=0.1$. Each entry reports mean $\pm$ standard deviation across runs.}
\label{tab:exp2_mobile_actions}
\end{table*}

Tables~\ref{tab:exp2_commonsense_qa}--\ref{tab:exp2_mobile_actions} report the experiment in which deployment temperature is fixed to $T=0.5$ and the deployment quantization mode is averaged within run by design. The resulting tables preserve the same qualitative picture seen in the main setting. In particular, Tables~\ref{tab:exp2_commonsense_qa} and \ref{tab:exp2_mmlu_pro} show that several model--dataset pairs on CommonsenseQA and MMLU-Pro remain strongly mismatch-sensitive, indicating that the phenomenon is not attributable only to one particular deployment quantization mode. This supports the interpretation that low-temperature deployment itself defines a regime of elevated sensitivity to calibration--deployment mismatch.

At the same time, the per-model view in Tables~\ref{tab:exp2_commonsense_qa}--\ref{tab:exp2_mobile_actions} continues to show that the effect is not uniform across model families. Some models remain comparatively stable under the same averaging regime in which others degrade sharply. This reinforces one of the paper's central points: there is no single monotone ordering of configurations that cleanly explains all failures.

\subsection{Deployment quantization fixed to \texorpdfstring{$\mathrm{nf4}$}{nf4}}
\label{app:exp3}

\begin{table*}[t]
\centering
\centerline{
\resizebox{1.4\textwidth}{!}{%
\begin{tabular}{lcccccccccccc}
\toprule
Model & \multicolumn{2}{c}{Matched} & \multicolumn{2}{c}{Mismatch mean} & \multicolumn{2}{c}{Mismatch worst} & \multicolumn{2}{c}{Pooled} & \multicolumn{2}{c}{MaxThr} & \multicolumn{2}{c}{Latent-robust} \\
\cmidrule(lr){2-3} \cmidrule(lr){4-5} \cmidrule(lr){6-7} \cmidrule(lr){8-9} \cmidrule(lr){10-11} \cmidrule(lr){12-13}
 & Cov. & Size & Cov. & Size & Cov. & Size & Cov. & Size & Cov. & Size & Cov. & Size \\
\midrule
Qwen2.5-1.5B & 0.916 $\pm$ 0.021 & 1.866 $\pm$ 0.107 & 0.907 $\pm$ 0.012 & 2.377 $\pm$ 0.122 & 0.848 $\pm$ 0.014 & 1.419 $\pm$ 0.049 & 0.903 $\pm$ 0.008 & 1.924 $\pm$ 0.068 & 0.976 $\pm$ 0.007 & 3.914 $\pm$ 0.155 & 0.984 $\pm$ 0.013 & 4.075 $\pm$ 0.532 \\
Qwen2.5-3B & 0.913 $\pm$ 0.012 & 1.478 $\pm$ 0.042 & 0.907 $\pm$ 0.010 & 2.021 $\pm$ 0.103 & 0.858 $\pm$ 0.014 & 1.227 $\pm$ 0.025 & 0.900 $\pm$ 0.018 & 1.543 $\pm$ 0.175 & 0.973 $\pm$ 0.005 & 3.635 $\pm$ 0.123 & 0.987 $\pm$ 0.007 & 4.087 $\pm$ 0.349 \\
Qwen2.5-7B & 0.907 $\pm$ 0.007 & 1.364 $\pm$ 0.012 & 0.910 $\pm$ 0.013 & 1.699 $\pm$ 0.168 & 0.872 $\pm$ 0.013 & 1.185 $\pm$ 0.051 & 0.898 $\pm$ 0.017 & 1.356 $\pm$ 0.127 & 0.959 $\pm$ 0.012 & 2.732 $\pm$ 0.455 & 0.969 $\pm$ 0.013 & 2.780 $\pm$ 0.370 \\
DeepSeek-R1-7B & 0.894 $\pm$ 0.003 & 2.618 $\pm$ 0.028 & 0.879 $\pm$ 0.006 & 2.908 $\pm$ 0.039 & 0.805 $\pm$ 0.010 & 1.959 $\pm$ 0.048 & 0.887 $\pm$ 0.004 & 2.835 $\pm$ 0.075 & 0.970 $\pm$ 0.001 & 4.279 $\pm$ 0.018 & 0.987 $\pm$ 0.005 & 4.542 $\pm$ 0.126 \\
Gemma-2-2B & 0.919 $\pm$ 0.021 & 2.480 $\pm$ 0.216 & 0.914 $\pm$ 0.009 & 3.069 $\pm$ 0.111 & 0.843 $\pm$ 0.015 & 1.844 $\pm$ 0.123 & 0.914 $\pm$ 0.013 & 2.855 $\pm$ 0.237 & 0.986 $\pm$ 0.003 & 4.480 $\pm$ 0.051 & 0.993 $\pm$ 0.002 & 4.620 $\pm$ 0.038 \\
Llama-3.2-3B & 0.916 $\pm$ 0.013 & 1.983 $\pm$ 0.058 & 0.916 $\pm$ 0.011 & 2.425 $\pm$ 0.072 & 0.866 $\pm$ 0.016 & 1.605 $\pm$ 0.041 & 0.909 $\pm$ 0.011 & 2.044 $\pm$ 0.073 & 0.977 $\pm$ 0.005 & 3.752 $\pm$ 0.153 & 0.992 $\pm$ 0.003 & 4.313 $\pm$ 0.118 \\
Phi-3-mini & 0.907 $\pm$ 0.011 & 1.892 $\pm$ 0.052 & 0.902 $\pm$ 0.007 & 2.209 $\pm$ 0.061 & 0.852 $\pm$ 0.010 & 1.519 $\pm$ 0.054 & 0.900 $\pm$ 0.005 & 1.928 $\pm$ 0.083 & 0.970 $\pm$ 0.004 & 3.484 $\pm$ 0.108 & 0.983 $\pm$ 0.009 & 3.890 $\pm$ 0.471 \\
Mistral-7B & 0.920 $\pm$ 0.006 & 2.505 $\pm$ 0.028 & 0.916 $\pm$ 0.003 & 2.918 $\pm$ 0.030 & 0.859 $\pm$ 0.003 & 1.864 $\pm$ 0.023 & 0.916 $\pm$ 0.013 & 2.707 $\pm$ 0.248 & 0.981 $\pm$ 0.002 & 4.286 $\pm$ 0.006 & 0.992 $\pm$ 0.002 & 4.596 $\pm$ 0.013 \\
\bottomrule
\end{tabular}%
}
}
\caption{Coverage and average prediction-set size for CommonsenseQA with deployment quantization fixed to $\mathrm{nf4}$ and deployment temperature averaged within run, using LAC with $\alpha=0.1$. Each entry reports mean $\pm$ standard deviation across runs.}
\label{tab:exp3_commonsense_qa}
\end{table*}
\begin{table*}[t]
\centering
\centerline{
\resizebox{1.4\textwidth}{!}{%
\begin{tabular}{lcccccccccccc}
\toprule
Model & \multicolumn{2}{c}{Matched} & \multicolumn{2}{c}{Mismatch mean} & \multicolumn{2}{c}{Mismatch worst} & \multicolumn{2}{c}{Pooled} & \multicolumn{2}{c}{MaxThr} & \multicolumn{2}{c}{Latent-robust} \\
\cmidrule(lr){2-3} \cmidrule(lr){4-5} \cmidrule(lr){6-7} \cmidrule(lr){8-9} \cmidrule(lr){10-11} \cmidrule(lr){12-13}
 & Cov. & Size & Cov. & Size & Cov. & Size & Cov. & Size & Cov. & Size & Cov. & Size \\
\midrule
Qwen2.5-1.5B & 0.899 $\pm$ 0.005 & 1.254 $\pm$ 0.035 & 0.893 $\pm$ 0.006 & 1.281 $\pm$ 0.068 & 0.872 $\pm$ 0.009 & 1.103 $\pm$ 0.030 & 0.892 $\pm$ 0.004 & 1.199 $\pm$ 0.026 & 0.924 $\pm$ 0.007 & 1.736 $\pm$ 0.188 & 0.958 $\pm$ 0.004 & 3.125 $\pm$ 0.276 \\
Qwen2.5-3B & 0.900 $\pm$ 0.003 & 1.648 $\pm$ 0.033 & 0.834 $\pm$ 0.005 & 1.462 $\pm$ 0.026 & 0.762 $\pm$ 0.007 & 1.055 $\pm$ 0.014 & 0.853 $\pm$ 0.008 & 1.356 $\pm$ 0.040 & 0.956 $\pm$ 0.002 & 3.220 $\pm$ 0.103 & 0.964 $\pm$ 0.002 & 3.176 $\pm$ 0.056 \\
Qwen2.5-7B & 0.902 $\pm$ 0.002 & 0.909 $\pm$ 0.001 & 0.567 $\pm$ 0.003 & 0.577 $\pm$ 0.003 & 0.220 $\pm$ 0.000 & 0.222 $\pm$ 0.000 & 0.874 $\pm$ 0.003 & 0.889 $\pm$ 0.003 & 0.943 $\pm$ 0.001 & 0.960 $\pm$ 0.001 & 0.959 $\pm$ 0.001 & 0.980 $\pm$ 0.001 \\
DeepSeek-R1-7B & 0.902 $\pm$ 0.005 & 1.681 $\pm$ 0.038 & 0.917 $\pm$ 0.002 & 2.511 $\pm$ 0.078 & 0.881 $\pm$ 0.004 & 1.520 $\pm$ 0.026 & 0.916 $\pm$ 0.001 & 1.968 $\pm$ 0.028 & 0.968 $\pm$ 0.001 & 6.025 $\pm$ 0.328 & 0.955 $\pm$ 0.002 & 2.982 $\pm$ 0.130 \\
Gemma-2-2B & 0.902 $\pm$ 0.005 & 3.508 $\pm$ 0.153 & 0.907 $\pm$ 0.001 & 7.144 $\pm$ 0.085 & 0.832 $\pm$ 0.002 & 2.550 $\pm$ 0.069 & 0.918 $\pm$ 0.001 & 7.484 $\pm$ 0.082 & 0.985 $\pm$ 0.001 & 11.988 $\pm$ 0.059 & 0.980 $\pm$ 0.002 & 10.233 $\pm$ 0.185 \\
Llama-3.2-3B & 0.899 $\pm$ 0.002 & 1.771 $\pm$ 0.024 & 0.862 $\pm$ 0.001 & 1.896 $\pm$ 0.022 & 0.832 $\pm$ 0.002 & 1.242 $\pm$ 0.013 & 0.866 $\pm$ 0.002 & 1.459 $\pm$ 0.015 & 0.959 $\pm$ 0.002 & 6.831 $\pm$ 0.334 & 0.948 $\pm$ 0.003 & 3.570 $\pm$ 0.139 \\
Phi-3-mini & 0.898 $\pm$ 0.002 & 0.969 $\pm$ 0.003 & 0.775 $\pm$ 0.003 & 0.822 $\pm$ 0.004 & 0.489 $\pm$ 0.005 & 0.506 $\pm$ 0.005 & 0.863 $\pm$ 0.002 & 0.919 $\pm$ 0.003 & 0.916 $\pm$ 0.002 & 1.006 $\pm$ 0.006 & 0.945 $\pm$ 0.003 & 1.096 $\pm$ 0.013 \\
Mistral-7B & 0.901 $\pm$ 0.004 & 1.132 $\pm$ 0.012 & 0.922 $\pm$ 0.003 & 1.580 $\pm$ 0.096 & 0.890 $\pm$ 0.004 & 1.100 $\pm$ 0.010 & 0.920 $\pm$ 0.002 & 1.267 $\pm$ 0.019 & 0.963 $\pm$ 0.004 & 3.215 $\pm$ 0.537 & 0.962 $\pm$ 0.004 & 2.120 $\pm$ 0.185 \\
\bottomrule
\end{tabular}%
}
}
\caption{Coverage and average prediction-set size for DBPedia-14 with deployment quantization fixed to $\mathrm{nf4}$ and deployment temperature averaged within run, using LAC with $\alpha=0.1$. Each entry reports mean $\pm$ standard deviation across runs.}
\label{tab:exp3_dbpedia_14}
\end{table*}
\begin{table*}[t]
\centering
\centerline{
\resizebox{1.4\textwidth}{!}{%
\begin{tabular}{lcccccccccccc}
\toprule
Model & \multicolumn{2}{c}{Matched} & \multicolumn{2}{c}{Mismatch mean} & \multicolumn{2}{c}{Mismatch worst} & \multicolumn{2}{c}{Pooled} & \multicolumn{2}{c}{MaxThr} & \multicolumn{2}{c}{Latent-robust} \\
\cmidrule(lr){2-3} \cmidrule(lr){4-5} \cmidrule(lr){6-7} \cmidrule(lr){8-9} \cmidrule(lr){10-11} \cmidrule(lr){12-13}
 & Cov. & Size & Cov. & Size & Cov. & Size & Cov. & Size & Cov. & Size & Cov. & Size \\
\midrule
Qwen2.5-1.5B & 0.899 $\pm$ 0.010 & 7.783 $\pm$ 0.100 & 0.821 $\pm$ 0.004 & 7.044 $\pm$ 0.019 & 0.685 $\pm$ 0.006 & 5.321 $\pm$ 0.026 & 0.895 $\pm$ 0.011 & 7.964 $\pm$ 0.120 & 0.977 $\pm$ 0.002 & 9.053 $\pm$ 0.004 & 0.989 $\pm$ 0.001 & 9.232 $\pm$ 0.012 \\
Qwen2.5-3B & 0.898 $\pm$ 0.007 & 6.964 $\pm$ 0.074 & 0.835 $\pm$ 0.005 & 6.636 $\pm$ 0.088 & 0.691 $\pm$ 0.006 & 4.438 $\pm$ 0.080 & 0.893 $\pm$ 0.002 & 7.437 $\pm$ 0.049 & 0.992 $\pm$ 0.014 & 9.259 $\pm$ 0.361 & 1.000 $\pm$ 0.000 & 9.463 $\pm$ 0.017 \\
Qwen2.5-7B & 0.898 $\pm$ 0.006 & 6.159 $\pm$ 0.068 & 0.853 $\pm$ 0.002 & 6.206 $\pm$ 0.025 & 0.743 $\pm$ 0.003 & 4.172 $\pm$ 0.018 & 0.893 $\pm$ 0.006 & 6.764 $\pm$ 0.101 & 0.974 $\pm$ 0.002 & 8.596 $\pm$ 0.029 & 1.000 $\pm$ 0.000 & 9.456 $\pm$ 0.023 \\
DeepSeek-R1-7B & 0.906 $\pm$ 0.006 & 7.641 $\pm$ 0.041 & 0.826 $\pm$ 0.006 & 6.943 $\pm$ 0.050 & 0.697 $\pm$ 0.008 & 5.152 $\pm$ 0.075 & 0.900 $\pm$ 0.007 & 7.927 $\pm$ 0.083 & 0.977 $\pm$ 0.002 & 9.026 $\pm$ 0.019 & 0.989 $\pm$ 0.002 & 9.217 $\pm$ 0.014 \\
Gemma-2-2B & 0.898 $\pm$ 0.002 & 7.432 $\pm$ 0.039 & 0.860 $\pm$ 0.001 & 7.241 $\pm$ 0.017 & 0.736 $\pm$ 0.003 & 5.504 $\pm$ 0.026 & 0.912 $\pm$ 0.001 & 7.922 $\pm$ 0.019 & 0.983 $\pm$ 0.001 & 9.078 $\pm$ 0.022 & 0.993 $\pm$ 0.000 & 9.259 $\pm$ 0.019 \\
Llama-3.2-3B & 0.903 $\pm$ 0.001 & 7.284 $\pm$ 0.009 & 0.834 $\pm$ 0.003 & 6.642 $\pm$ 0.049 & 0.720 $\pm$ 0.007 & 5.011 $\pm$ 0.086 & 0.884 $\pm$ 0.006 & 7.248 $\pm$ 0.091 & 0.973 $\pm$ 0.001 & 8.778 $\pm$ 0.012 & 0.989 $\pm$ 0.001 & 9.104 $\pm$ 0.014 \\
Phi-3-mini & 0.902 $\pm$ 0.004 & 6.419 $\pm$ 0.033 & 0.852 $\pm$ 0.003 & 6.302 $\pm$ 0.040 & 0.750 $\pm$ 0.005 & 4.491 $\pm$ 0.062 & 0.894 $\pm$ 0.004 & 6.884 $\pm$ 0.063 & 0.975 $\pm$ 0.001 & 8.597 $\pm$ 0.006 & 0.989 $\pm$ 0.001 & 8.940 $\pm$ 0.043 \\
Mistral-7B & 0.907 $\pm$ 0.008 & 7.212 $\pm$ 0.085 & 0.853 $\pm$ 0.002 & 6.881 $\pm$ 0.010 & 0.731 $\pm$ 0.005 & 4.925 $\pm$ 0.021 & 0.906 $\pm$ 0.002 & 7.668 $\pm$ 0.027 & 0.982 $\pm$ 0.001 & 8.996 $\pm$ 0.006 & 0.992 $\pm$ 0.000 & 9.205 $\pm$ 0.022 \\
\bottomrule
\end{tabular}%
}
}
\caption{Coverage and average prediction-set size for MMLU-Pro with deployment quantization fixed to $\mathrm{nf4}$ and deployment temperature averaged within run, using LAC with $\alpha=0.1$. Each entry reports mean $\pm$ standard deviation across runs.}
\label{tab:exp3_mmlu_pro}
\end{table*}
\begin{table*}[t]
\centering
\centerline{
\resizebox{1.4\textwidth}{!}{%
\begin{tabular}{lcccccccccccc}
\toprule
Model & \multicolumn{2}{c}{Matched} & \multicolumn{2}{c}{Mismatch mean} & \multicolumn{2}{c}{Mismatch worst} & \multicolumn{2}{c}{Pooled} & \multicolumn{2}{c}{MaxThr} & \multicolumn{2}{c}{Latent-robust} \\
\cmidrule(lr){2-3} \cmidrule(lr){4-5} \cmidrule(lr){6-7} \cmidrule(lr){8-9} \cmidrule(lr){10-11} \cmidrule(lr){12-13}
 & Cov. & Size & Cov. & Size & Cov. & Size & Cov. & Size & Cov. & Size & Cov. & Size \\
\midrule
Qwen2.5-1.5B & 0.908 $\pm$ 0.002 & 0.976 $\pm$ 0.002 & 0.837 $\pm$ 0.003 & 0.882 $\pm$ 0.003 & 0.682 $\pm$ 0.020 & 0.697 $\pm$ 0.021 & 0.881 $\pm$ 0.003 & 0.936 $\pm$ 0.004 & 0.921 $\pm$ 0.001 & 1.002 $\pm$ 0.003 & 0.954 $\pm$ 0.014 & 1.112 $\pm$ 0.077 \\
Qwen2.5-3B & 0.899 $\pm$ 0.005 & 1.017 $\pm$ 0.010 & 0.909 $\pm$ 0.004 & 1.045 $\pm$ 0.009 & 0.897 $\pm$ 0.003 & 1.013 $\pm$ 0.006 & 0.904 $\pm$ 0.003 & 1.028 $\pm$ 0.004 & 0.936 $\pm$ 0.006 & 1.137 $\pm$ 0.023 & 0.967 $\pm$ 0.002 & 1.465 $\pm$ 0.019 \\
Qwen2.5-7B & 0.902 $\pm$ 0.002 & 0.920 $\pm$ 0.003 & 0.875 $\pm$ 0.009 & 0.894 $\pm$ 0.010 & 0.757 $\pm$ 0.011 & 0.763 $\pm$ 0.012 & 0.909 $\pm$ 0.003 & 0.930 $\pm$ 0.004 & 0.935 $\pm$ 0.003 & 0.963 $\pm$ 0.004 & 0.955 $\pm$ 0.002 & 0.996 $\pm$ 0.004 \\
DeepSeek-R1-7B & 0.898 $\pm$ 0.002 & 1.117 $\pm$ 0.006 & 0.903 $\pm$ 0.002 & 1.142 $\pm$ 0.011 & 0.883 $\pm$ 0.004 & 1.078 $\pm$ 0.013 & 0.894 $\pm$ 0.002 & 1.107 $\pm$ 0.007 & 0.939 $\pm$ 0.001 & 1.272 $\pm$ 0.002 & 0.963 $\pm$ 0.006 & 1.491 $\pm$ 0.077 \\
Gemma-2-2B & 0.900 $\pm$ 0.002 & 1.195 $\pm$ 0.007 & 0.912 $\pm$ 0.001 & 1.989 $\pm$ 0.026 & 0.877 $\pm$ 0.001 & 1.089 $\pm$ 0.003 & 0.898 $\pm$ 0.005 & 1.216 $\pm$ 0.031 & 0.967 $\pm$ 0.000 & 4.165 $\pm$ 0.059 & 0.944 $\pm$ 0.004 & 2.053 $\pm$ 0.132 \\
Llama-3.2-3B & 0.902 $\pm$ 0.000 & 1.165 $\pm$ 0.006 & 0.909 $\pm$ 0.002 & 1.435 $\pm$ 0.046 & 0.880 $\pm$ 0.003 & 1.104 $\pm$ 0.012 & 0.898 $\pm$ 0.001 & 1.163 $\pm$ 0.009 & 0.955 $\pm$ 0.001 & 2.278 $\pm$ 0.042 & 0.949 $\pm$ 0.006 & 1.702 $\pm$ 0.152 \\
Phi-3-mini & 0.904 $\pm$ 0.006 & 1.000 $\pm$ 0.008 & 0.905 $\pm$ 0.006 & 1.003 $\pm$ 0.008 & 0.888 $\pm$ 0.019 & 0.970 $\pm$ 0.030 & 0.907 $\pm$ 0.004 & 1.006 $\pm$ 0.005 & 0.916 $\pm$ 0.002 & 1.025 $\pm$ 0.001 & 0.959 $\pm$ 0.008 & 1.152 $\pm$ 0.032 \\
Mistral-7B & 0.899 $\pm$ 0.005 & 0.986 $\pm$ 0.007 & 0.894 $\pm$ 0.004 & 0.978 $\pm$ 0.005 & 0.877 $\pm$ 0.004 & 0.952 $\pm$ 0.005 & 0.899 $\pm$ 0.002 & 0.986 $\pm$ 0.006 & 0.902 $\pm$ 0.003 & 0.992 $\pm$ 0.004 & 0.952 $\pm$ 0.004 & 1.244 $\pm$ 0.030 \\
\bottomrule
\end{tabular}%
}
}
\caption{Coverage and average prediction-set size for MobileActions with deployment quantization fixed to $\mathrm{nf4}$ and deployment temperature averaged within run, using LAC with $\alpha=0.1$. Each entry reports mean $\pm$ standard deviation across runs.}
\label{tab:exp3_mobile_actions}
\end{table*}

Tables~\ref{tab:exp3_commonsense_qa}--\ref{tab:exp3_mobile_actions} report the experiment in which deployment quantization is fixed to $\mathrm{nf4}$ and deployment temperature is averaged within run. These tables show that substantial mismatch effects remain visible for multiple models even after fixing quantization, which indicates that temperature is itself a primary driver of inference-time model shift. This complements the previous subsection: neither temperature nor quantization alone fully explains the phenomenon, and the strongest discrepancies arise when both dimensions of configuration are allowed to vary.

The model-disaggregated results in Tables~\ref{tab:exp3_commonsense_qa}--\ref{tab:exp3_mobile_actions} again matter here. In several cases the same broad type of mismatch produces noticeably different effects across models, including different degrees of undercoverage and conservativeness. The appendix therefore supports the main-paper claim that matched-only reporting can be materially misleading even under a fixed benchmark task distribution.

\section{Calibration--Deployment Mismatch Structure}
\label{app:mismatch}

We next visualize the full pairwise vanilla calibration$\rightarrow$deployment structure under LAC with $\alpha=0.1$. Each figure is model-specific and contains four panels, one for each dataset. The left subfigure reports empirical coverage, while the right subfigure reports the corresponding coverage gap relative to the nominal target of $0.9$. Rows index calibration configuration and columns index deployment configuration.

These figures are intentionally reported at the model level rather than after averaging across models. This is important because the direction and magnitude of mismatch effects vary across model families.

\begin{figure*}[t]
    \centering
    \begin{subfigure}[t]{0.99\textwidth}
        \centering
        \includegraphics[width=\textwidth]{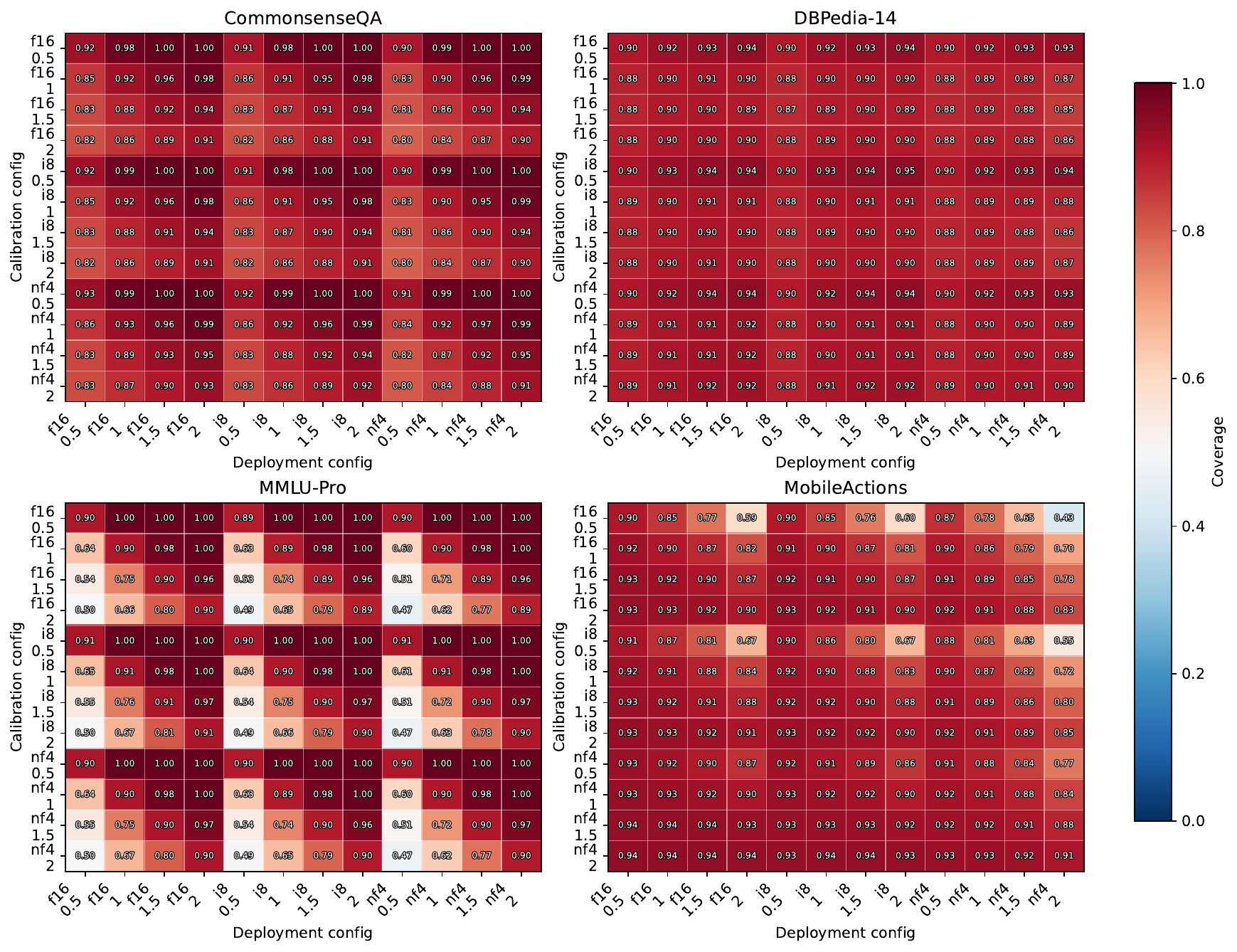}
        \caption{Coverage.}
    \end{subfigure}
    \hfill
    \begin{subfigure}[t]{0.99\textwidth}
        \centering
        \includegraphics[width=\textwidth]{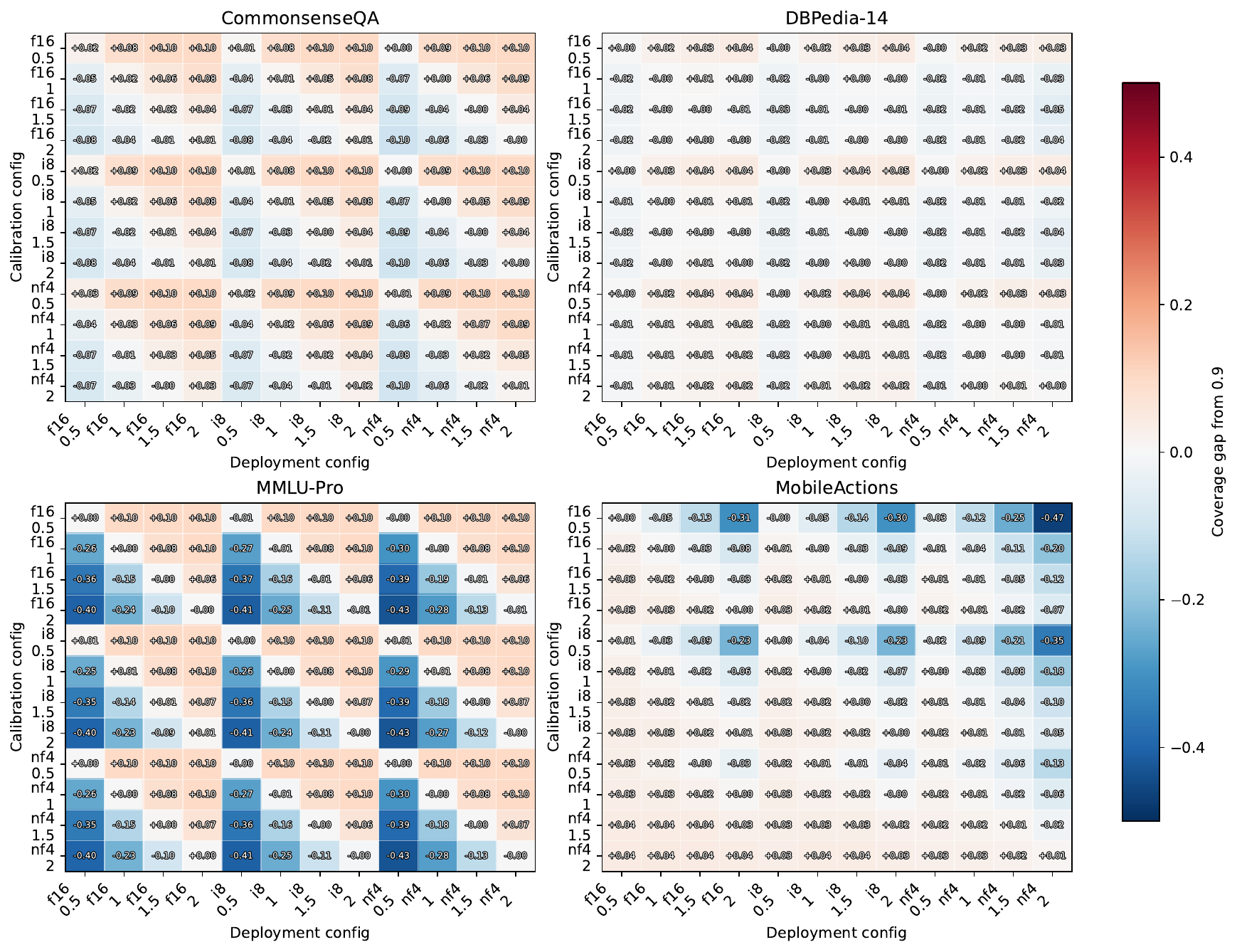}
        \caption{Coverage gap from target.}
    \end{subfigure}
    \caption{Pairwise vanilla calibration$\rightarrow$deployment structure for Qwen2.5-1.5B under LAC with $\alpha=0.1$. Each panel corresponds to one dataset. White cell annotations report the exact values used to construct the heatmap.}
    \label{fig:pairwise_qwen15}
\end{figure*}

\begin{figure*}[t]
    \centering
    \begin{subfigure}[t]{0.99\textwidth}
        \centering
        \includegraphics[width=\textwidth]{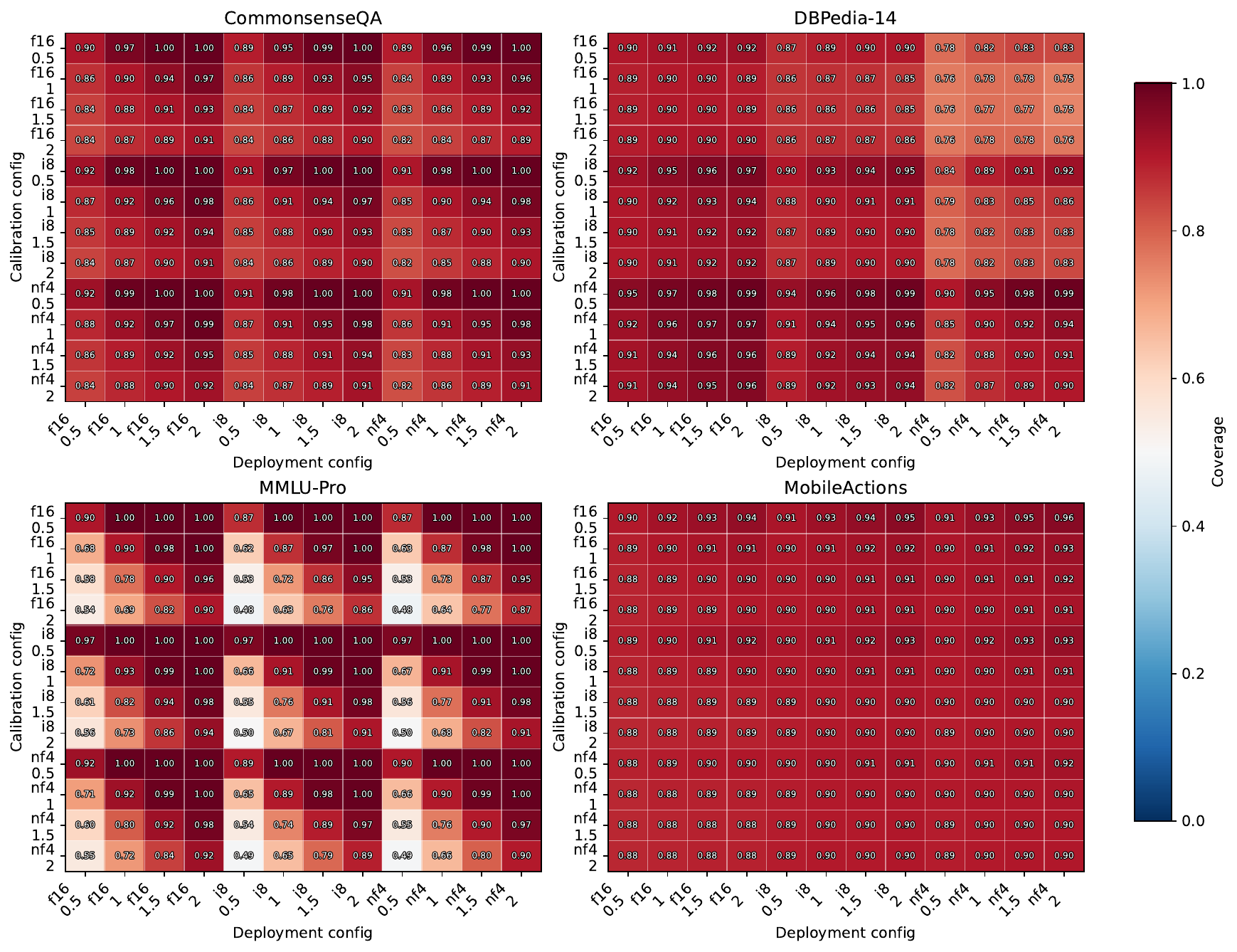}
        \caption{Coverage.}
    \end{subfigure}
    \hfill
    \begin{subfigure}[t]{0.99\textwidth}
        \centering
        \includegraphics[width=\textwidth]{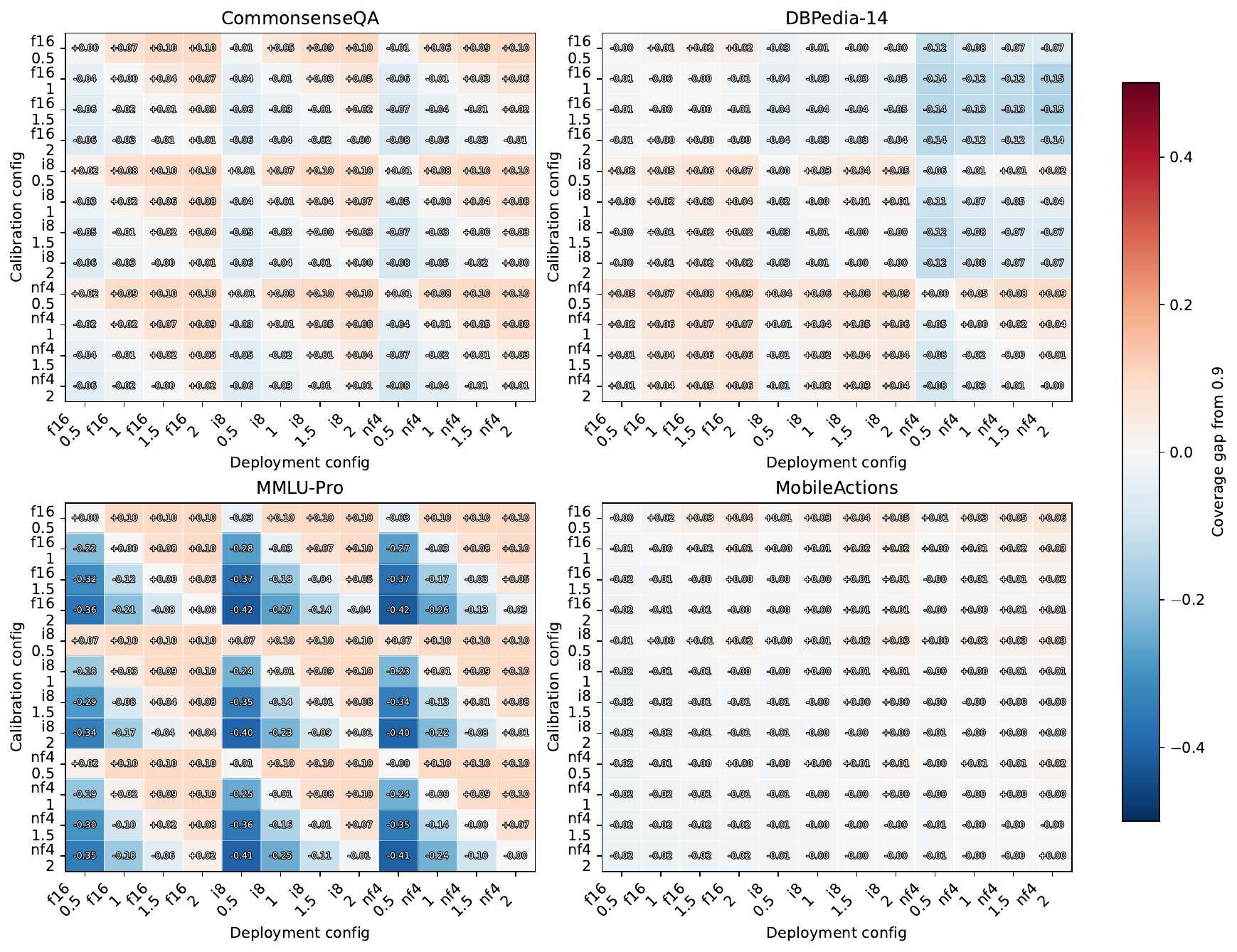}
        \caption{Coverage gap from target.}
    \end{subfigure}
    \caption{Pairwise vanilla calibration$\rightarrow$deployment structure for Qwen2.5-3B under LAC with $\alpha=0.1$.}
    \label{fig:pairwise_qwen3}
\end{figure*}

\begin{figure*}[t]
    \centering
    \begin{subfigure}[t]{0.99\textwidth}
        \centering
        \includegraphics[width=\textwidth]{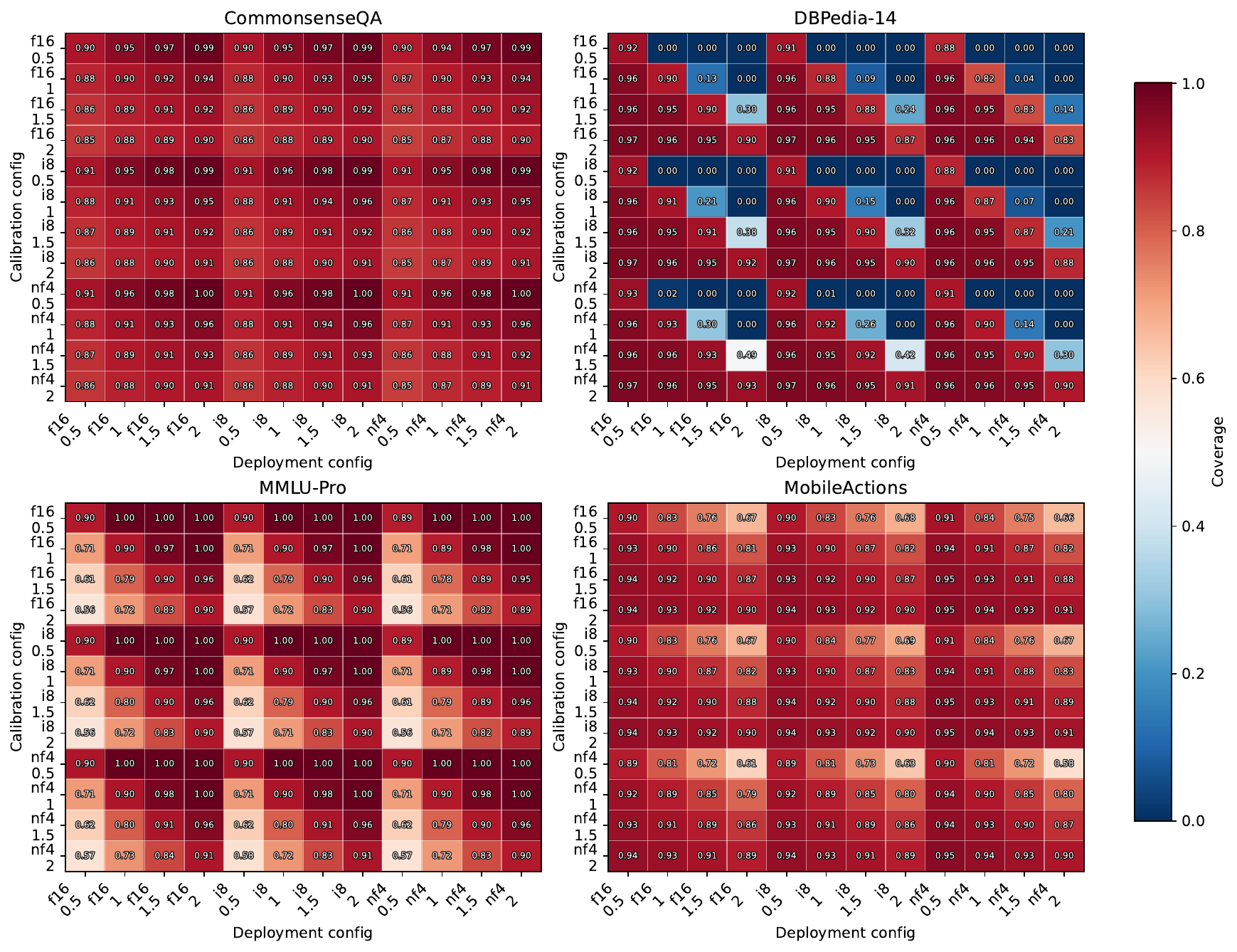}
        \caption{Coverage.}
    \end{subfigure}
    \hfill
    \begin{subfigure}[t]{0.99\textwidth}
        \centering
        \includegraphics[width=\textwidth]{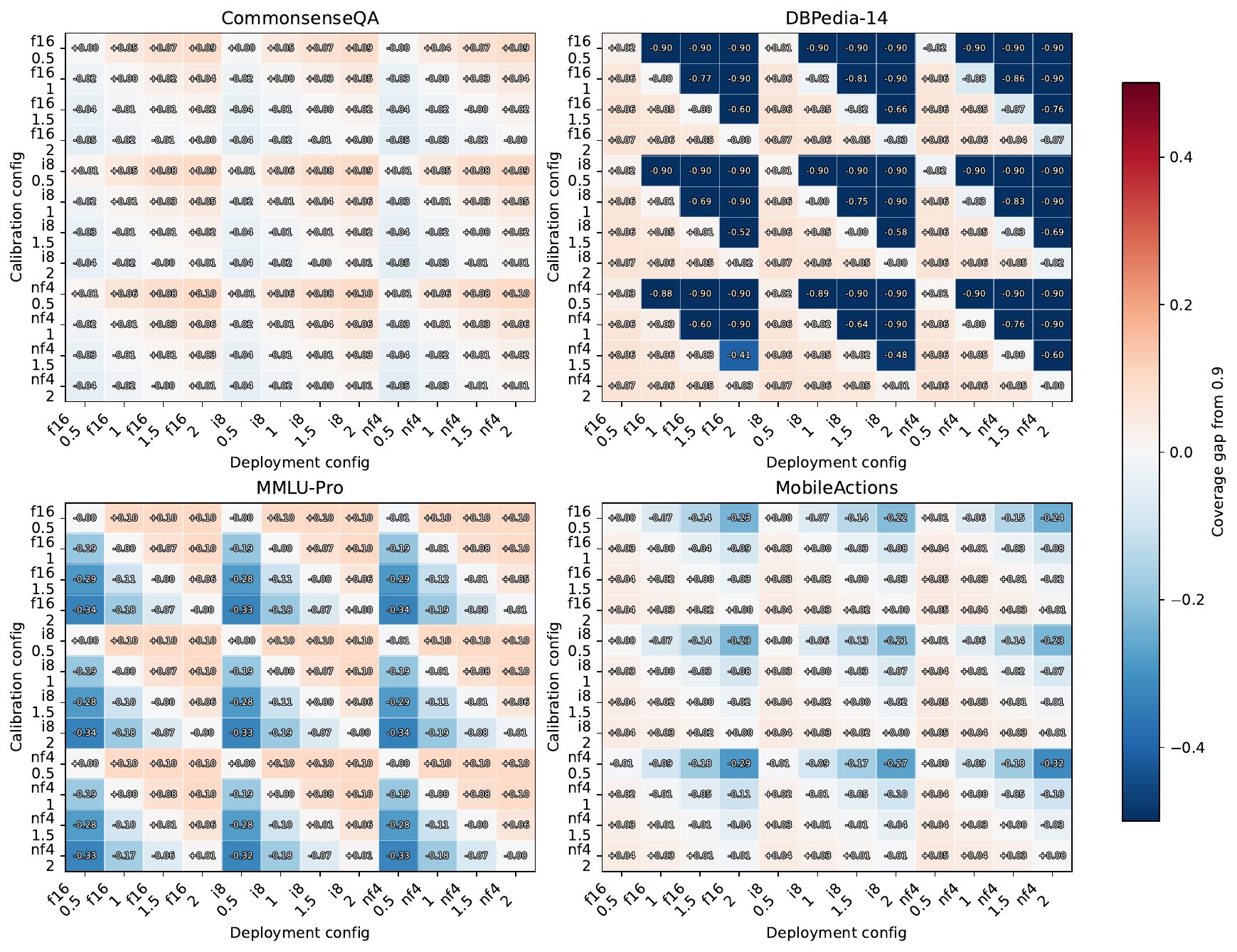}
        \caption{Coverage gap from target.}
    \end{subfigure}
    \caption{Pairwise vanilla calibration$\rightarrow$deployment structure for Qwen2.5-7B under LAC with $\alpha=0.1$.}
    \label{fig:pairwise_qwen7}
\end{figure*}

\begin{figure*}[t]
    \centering
    \begin{subfigure}[t]{0.99\textwidth}
        \centering
        \includegraphics[width=\textwidth]{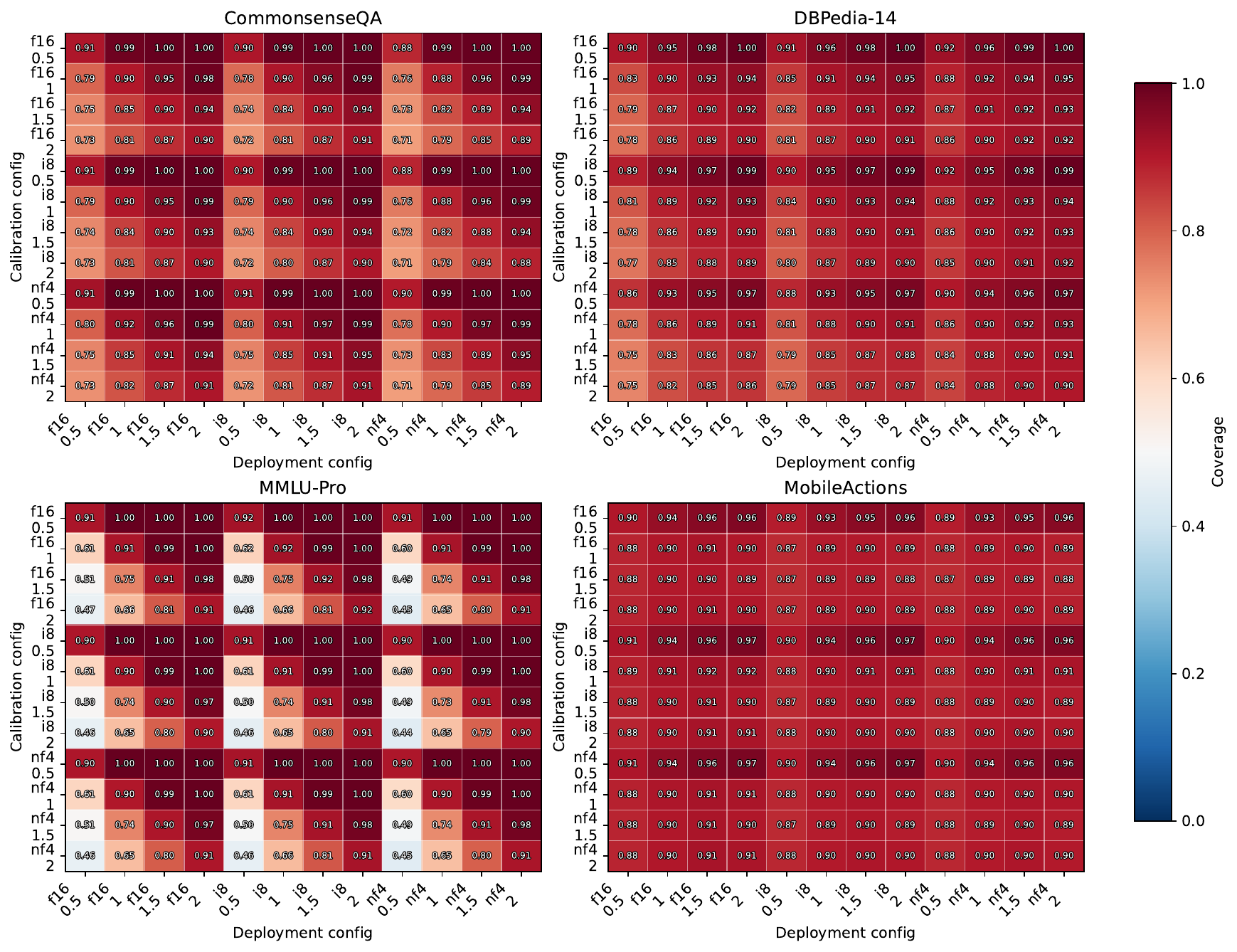}
        \caption{Coverage.}
    \end{subfigure}
    \hfill
    \begin{subfigure}[t]{0.99\textwidth}
        \centering
        \includegraphics[width=\textwidth]{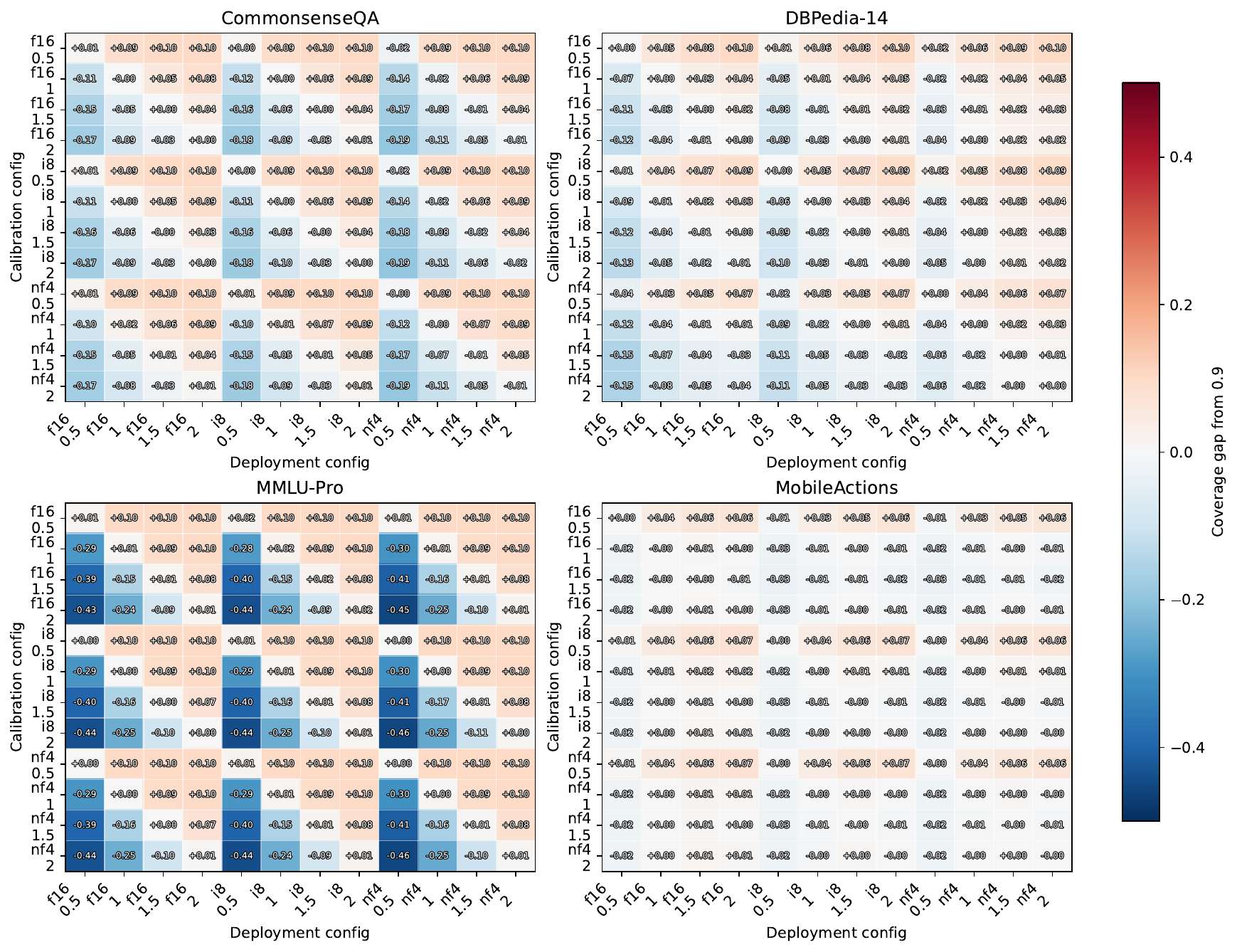}
        \caption{Coverage gap from target.}
    \end{subfigure}
    \caption{Pairwise vanilla calibration$\rightarrow$deployment structure for DeepSeek-R1-Distill-Qwen-7B under LAC with $\alpha=0.1$.}
    \label{fig:pairwise_deepseek}
\end{figure*}

\begin{figure*}[t]
    \centering
    \begin{subfigure}[t]{0.99\textwidth}
        \centering
        \includegraphics[width=\textwidth]{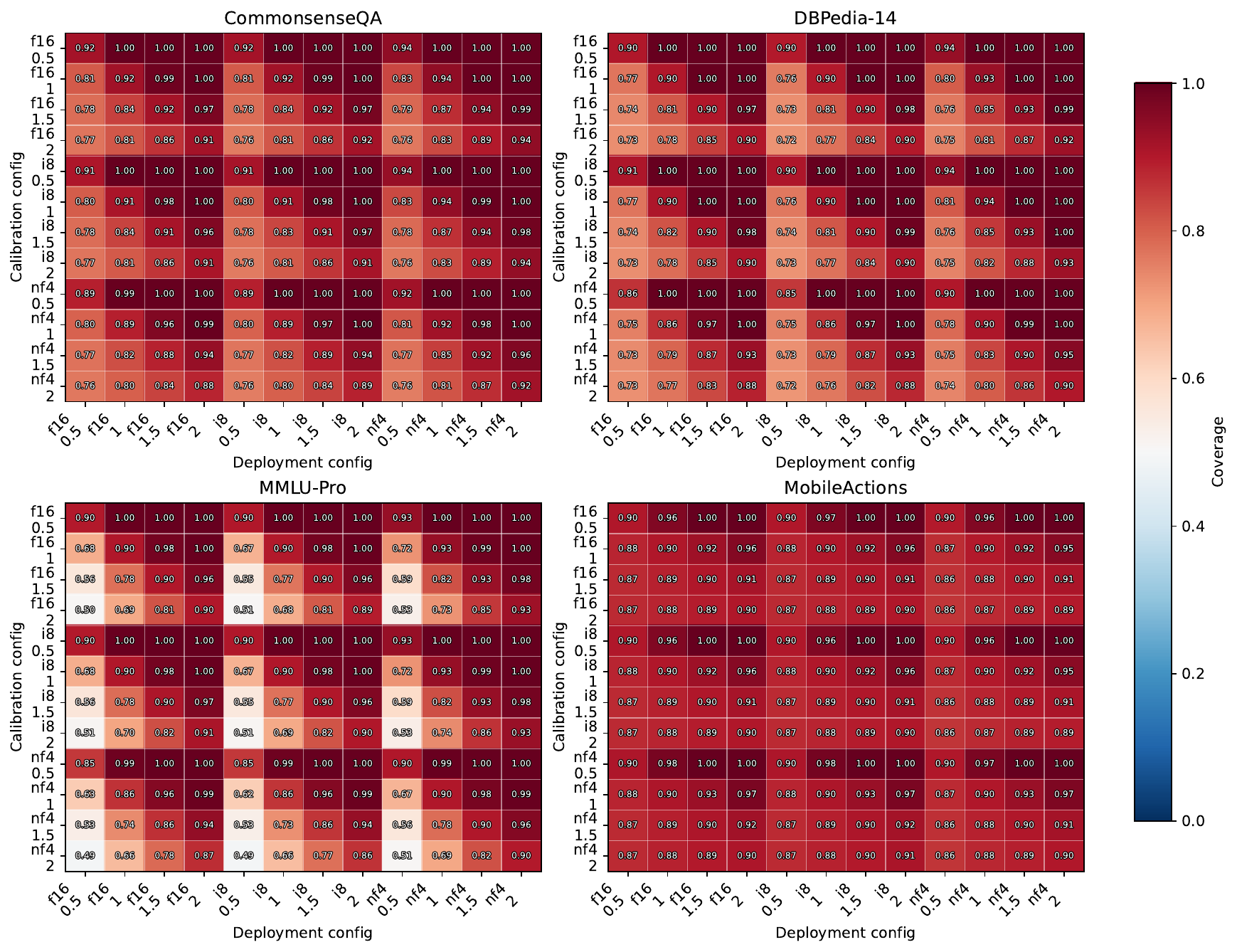}
        \caption{Coverage.}
    \end{subfigure}
    \hfill
    \begin{subfigure}[t]{0.99\textwidth}
        \centering
        \includegraphics[width=\textwidth]{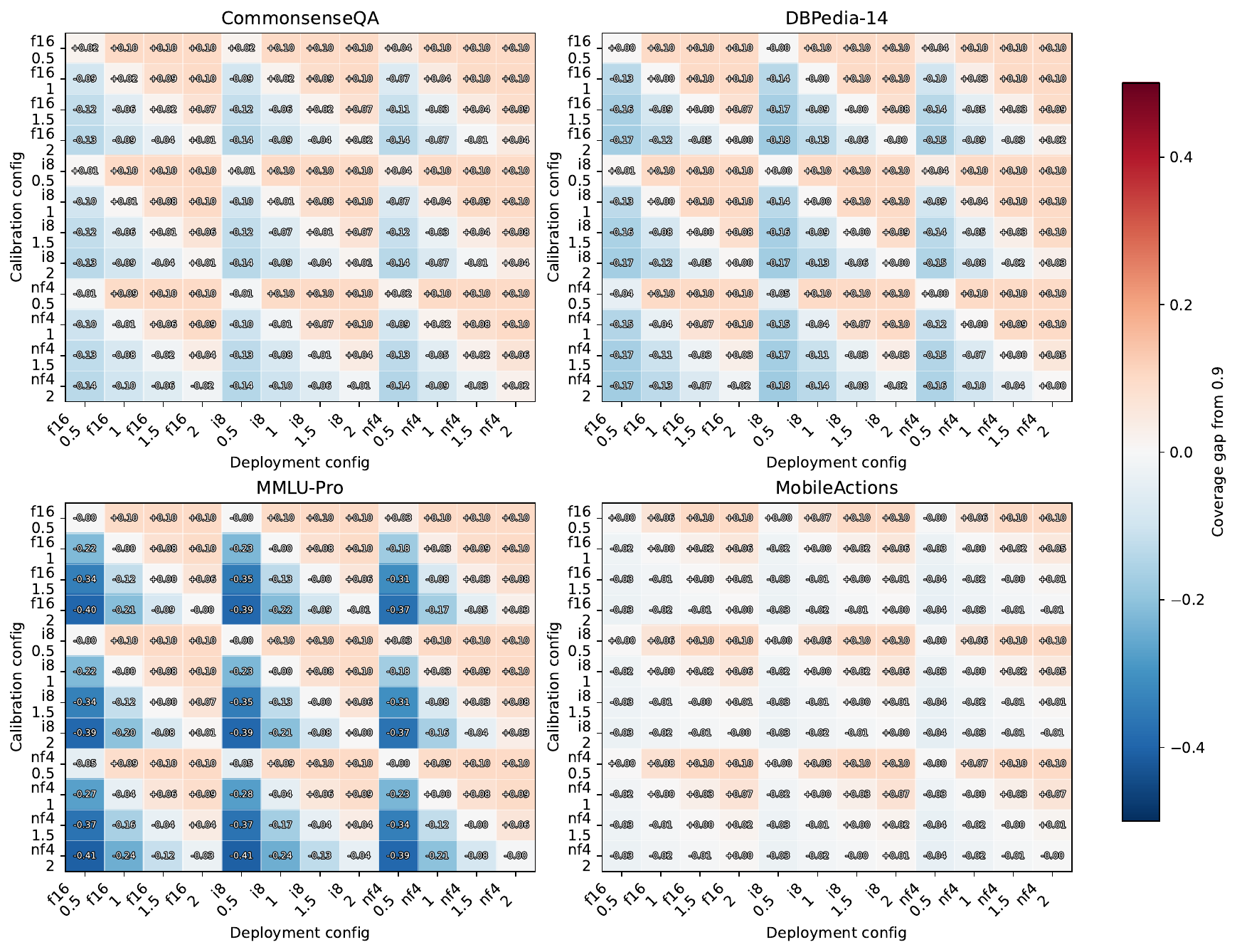}
        \caption{Coverage gap from target.}
    \end{subfigure}
    \caption{Pairwise vanilla calibration$\rightarrow$deployment structure for Gemma-2-2B under LAC with $\alpha=0.1$.}
    \label{fig:pairwise_gemma}
\end{figure*}

\begin{figure*}[t]
    \centering
    \begin{subfigure}[t]{0.99\textwidth}
        \centering
        \includegraphics[width=\textwidth]{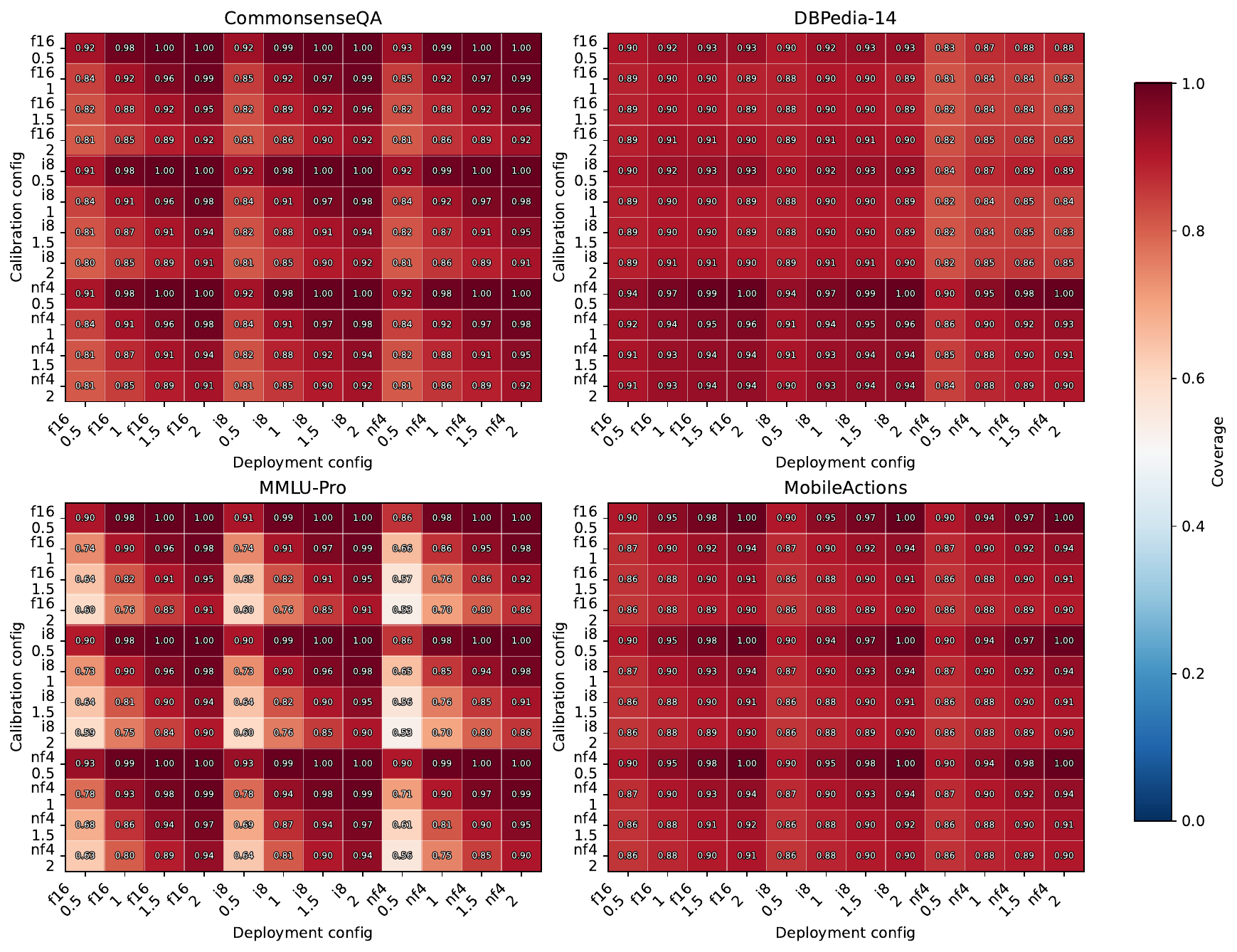}
        \caption{Coverage.}
    \end{subfigure}
    \hfill
    \begin{subfigure}[t]{0.99\textwidth}
        \centering
        \includegraphics[width=\textwidth]{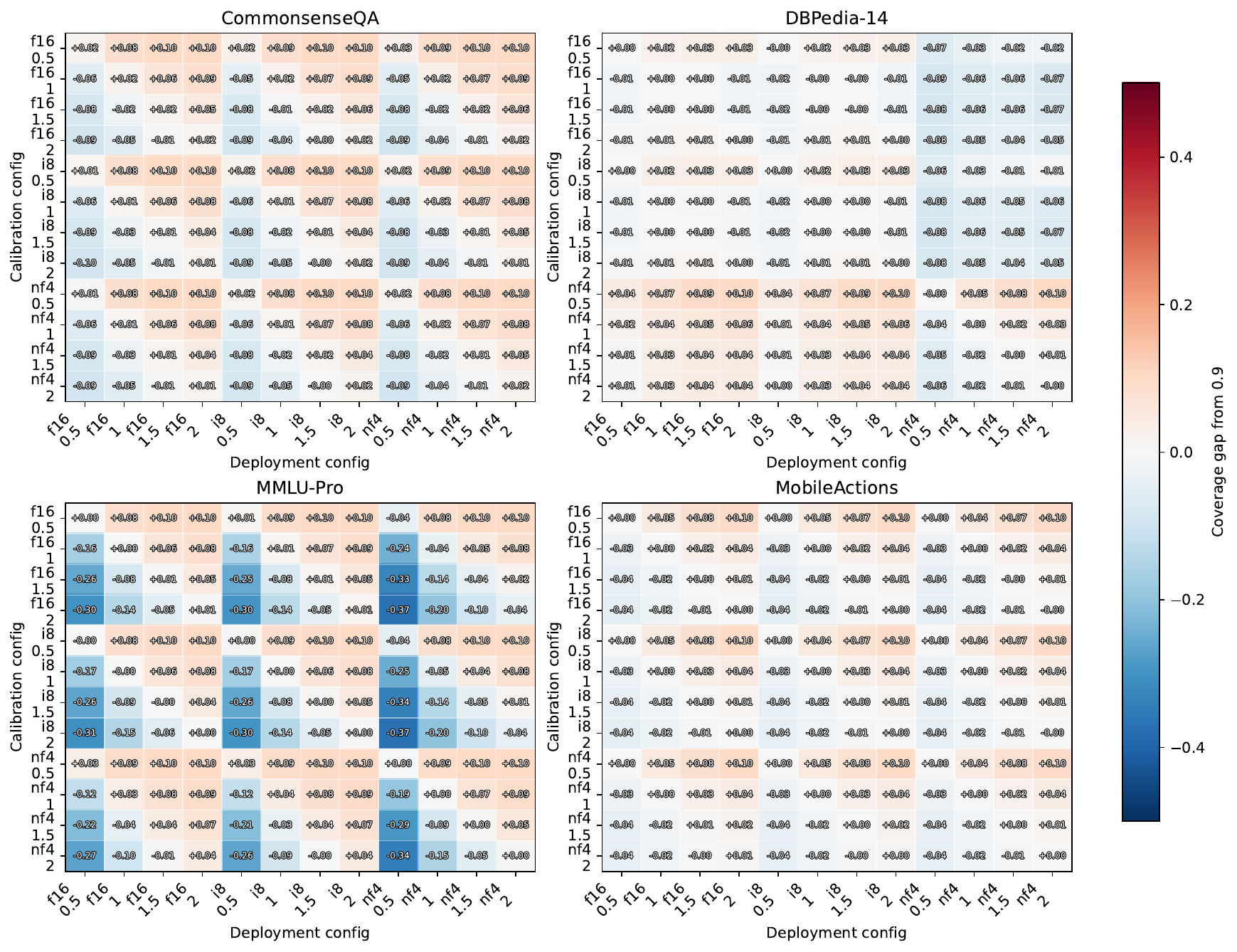}
        \caption{Coverage gap from target.}
    \end{subfigure}
    \caption{Pairwise vanilla calibration$\rightarrow$deployment structure for Llama-3.2-3B under LAC with $\alpha=0.1$.}
    \label{fig:pairwise_llama}
\end{figure*}

\begin{figure*}[t]
    \centering
    \begin{subfigure}[t]{0.99\textwidth}
        \centering
        \includegraphics[width=\textwidth]{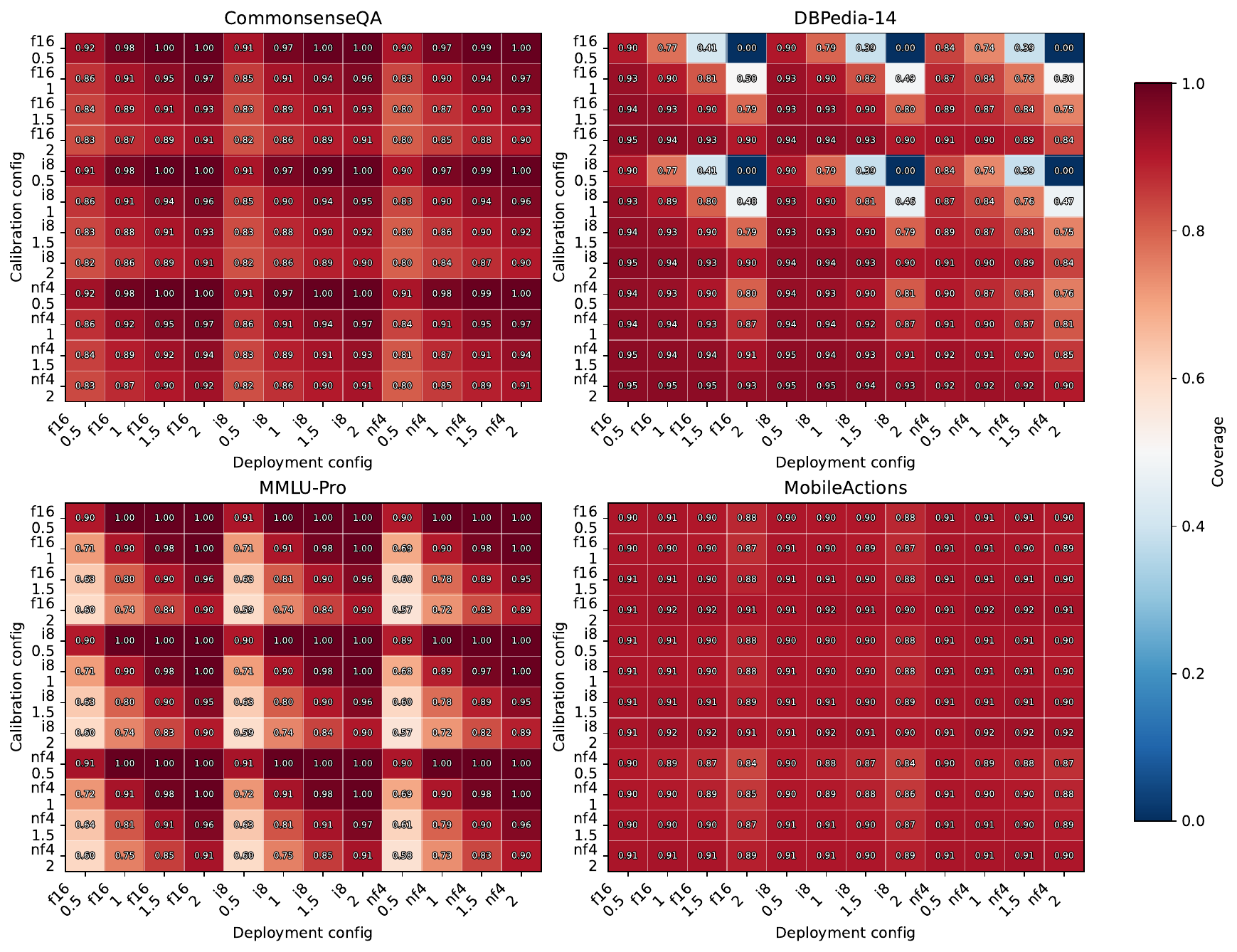}
        \caption{Coverage.}
    \end{subfigure}
    \hfill
    \begin{subfigure}[t]{0.99\textwidth}
        \centering
        \includegraphics[width=\textwidth]{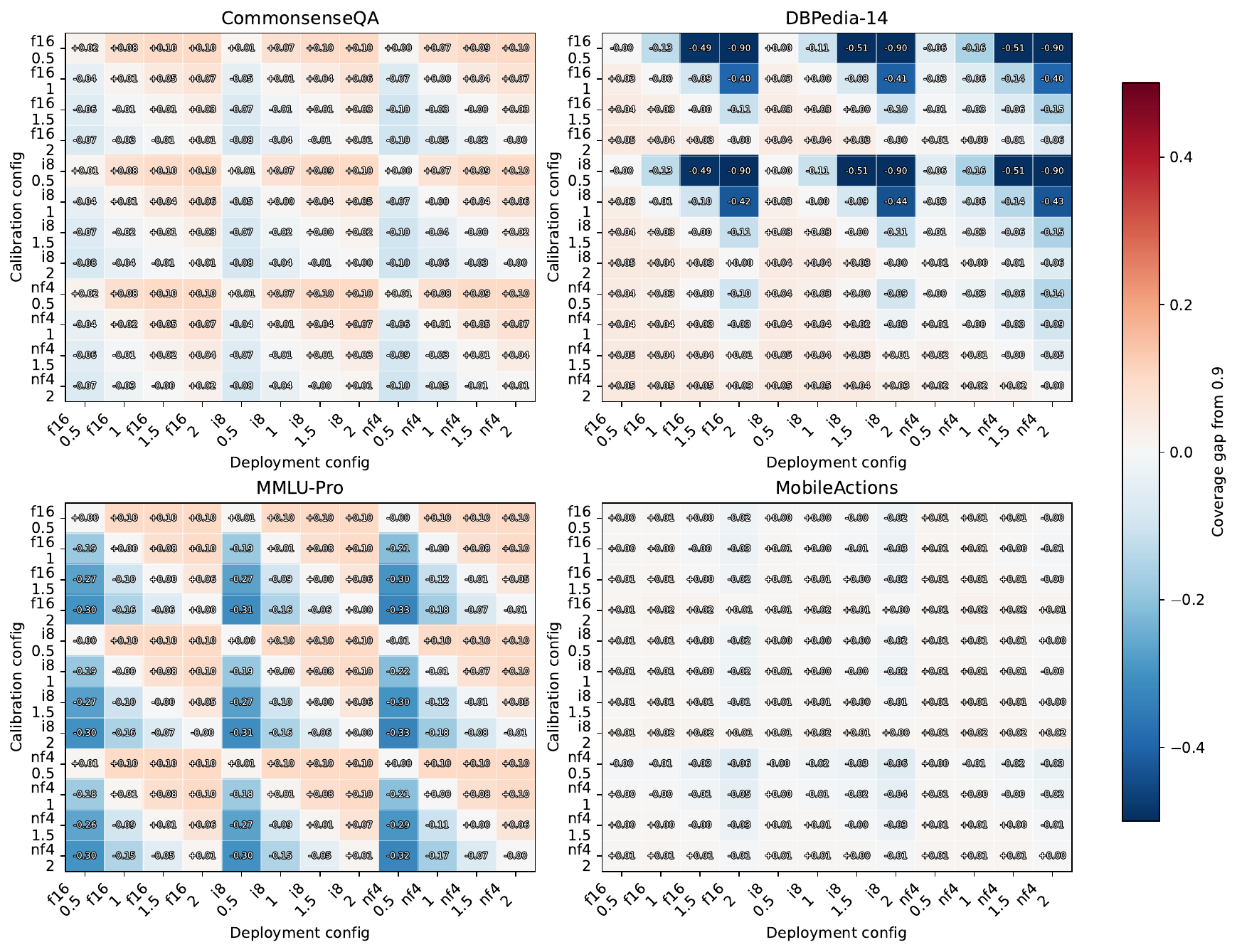}
        \caption{Coverage gap from target.}
    \end{subfigure}
    \caption{Pairwise vanilla calibration$\rightarrow$deployment structure for Phi-3-mini under LAC with $\alpha=0.1$.}
    \label{fig:pairwise_phi}
\end{figure*}

\begin{figure*}[t]
    \centering
    \begin{subfigure}[t]{0.99\textwidth}
        \centering
        \includegraphics[width=\textwidth]{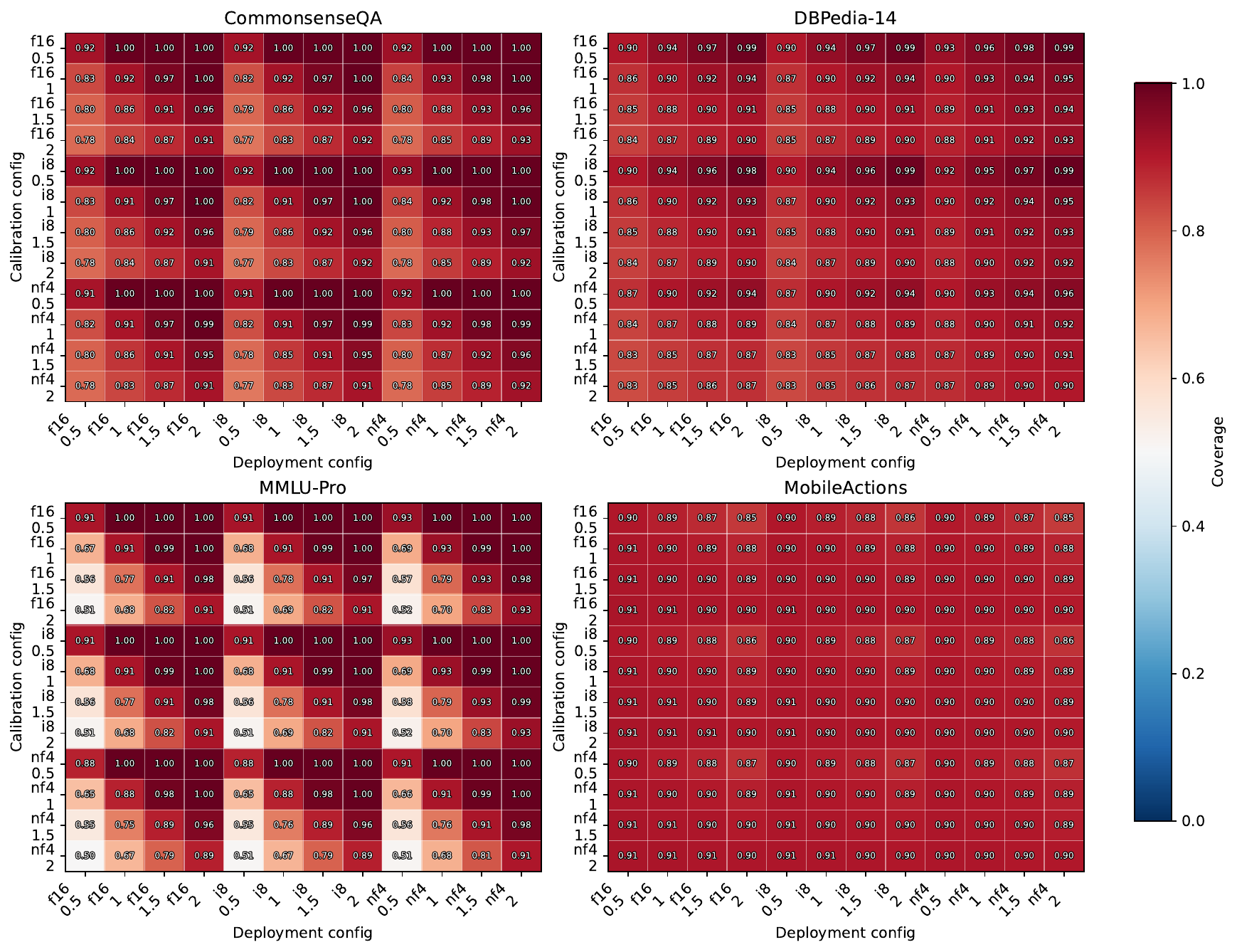}
        \caption{Coverage.}
    \end{subfigure}
    \hfill
    \begin{subfigure}[t]{0.99\textwidth}
        \centering
        \includegraphics[width=\textwidth]{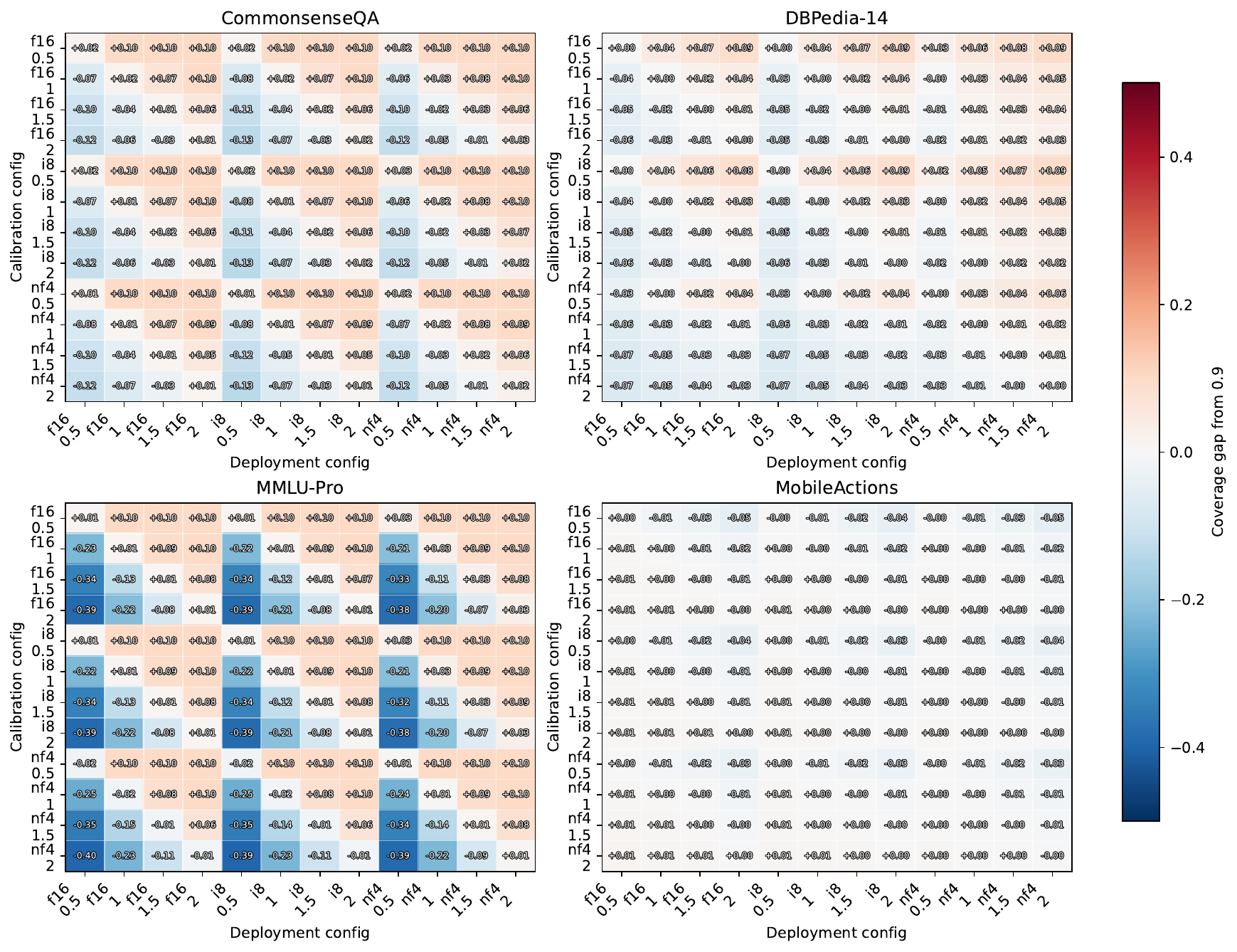}
        \caption{Coverage gap from target.}
    \end{subfigure}
    \caption{Pairwise vanilla calibration$\rightarrow$deployment structure for Mistral-7B under LAC with $\alpha=0.1$.}
    \label{fig:pairwise_mistral}
\end{figure*}

Figures~\ref{fig:pairwise_qwen15}--\ref{fig:pairwise_mistral} reveal structure that is difficult to see from summary statistics alone. First, the diagonal matched cells are generally close to target, as expected under the standard split-conformal setting. Second, off-diagonal behavior is highly configuration-dependent, with localized regions of severe undercoverage and other regions that are conservative. Third, these patterns differ materially across models. The same broad type of mismatch can therefore look relatively benign in one model family and clearly unreliable in another. This is exactly why model-level reporting is necessary for the present evaluation question.

Because each heatmap cell is numerically annotated, Figures~\ref{fig:pairwise_qwen15}--\ref{fig:pairwise_mistral} also support direct inspection of specific calibration$\rightarrow$deployment pairs without sacrificing the visual overview of the mismatch structure. The coverage-gap panels are especially useful for isolating undercoverage regions, while the raw coverage panels make the absolute magnitude of the deviations easy to interpret.

\section{Additional Robustness Analyses}
\label{app:robustness}

\subsection{Robustness to the score function}
\label{app:score_robustness}

\begin{table*}[t]
\centering
\small
\textbf{Panel A: APS}
\vspace{0.3em}
\centerline{
\resizebox{1.4\textwidth}{!}{%
\begin{tabular}{lcccccccccccc}
\toprule
Model & \multicolumn{2}{c}{Matched} & \multicolumn{2}{c}{Mismatch mean} & \multicolumn{2}{c}{Mismatch worst} & \multicolumn{2}{c}{Pooled} & \multicolumn{2}{c}{MaxThr} & \multicolumn{2}{c}{Latent-robust} \\
\cmidrule(lr){2-3} \cmidrule(lr){4-5} \cmidrule(lr){6-7} \cmidrule(lr){8-9} \cmidrule(lr){10-11} \cmidrule(lr){12-13}
 & Cov. & Size & Cov. & Size & Cov. & Size & Cov. & Size & Cov. & Size & Cov. & Size \\
\midrule
Qwen2.5-1.5B & 0.995 $\pm$ 0.002 & 4.193 $\pm$ 0.029 & 0.895 $\pm$ 0.008 & 2.097 $\pm$ 0.018 & 0.825 $\pm$ 0.006 & 1.277 $\pm$ 0.013 & 0.984 $\pm$ 0.009 & 3.256 $\pm$ 0.120 & 0.997 $\pm$ 0.001 & 4.427 $\pm$ 0.035 & 0.999 $\pm$ 0.001 & 4.732 $\pm$ 0.008 \\
Qwen2.5-3B & 0.943 $\pm$ 0.003 & 1.942 $\pm$ 0.010 & 0.895 $\pm$ 0.005 & 1.463 $\pm$ 0.006 & 0.854 $\pm$ 0.006 & 1.213 $\pm$ 0.013 & 0.943 $\pm$ 0.003 & 1.942 $\pm$ 0.010 & 0.943 $\pm$ 0.003 & 1.942 $\pm$ 0.010 & 0.943 $\pm$ 0.003 & 1.927 $\pm$ 0.012 \\
Qwen2.5-7B & 0.970 $\pm$ 0.005 & 2.455 $\pm$ 0.009 & 0.920 $\pm$ 0.004 & 1.604 $\pm$ 0.009 & 0.883 $\pm$ 0.006 & 1.226 $\pm$ 0.010 & 0.969 $\pm$ 0.008 & 2.373 $\pm$ 0.143 & 0.970 $\pm$ 0.005 & 2.455 $\pm$ 0.009 & 0.970 $\pm$ 0.005 & 2.438 $\pm$ 0.016 \\
DeepSeek-R1-7B & 0.980 $\pm$ 0.002 & 4.123 $\pm$ 0.044 & 0.830 $\pm$ 0.009 & 2.285 $\pm$ 0.031 & 0.738 $\pm$ 0.010 & 1.546 $\pm$ 0.032 & 0.926 $\pm$ 0.008 & 2.983 $\pm$ 0.076 & 0.981 $\pm$ 0.004 & 4.178 $\pm$ 0.029 & 1.000 $\pm$ 0.000 & 4.879 $\pm$ 0.052 \\
Gemma-2-2B & 0.995 $\pm$ 0.002 & 4.350 $\pm$ 0.047 & 0.859 $\pm$ 0.008 & 2.217 $\pm$ 0.050 & 0.775 $\pm$ 0.008 & 1.304 $\pm$ 0.031 & 0.984 $\pm$ 0.008 & 3.937 $\pm$ 0.130 & 0.996 $\pm$ 0.002 & 4.607 $\pm$ 0.018 & 0.998 $\pm$ 0.002 & 4.874 $\pm$ 0.010 \\
Llama-3.2-3B & 0.997 $\pm$ 0.002 & 4.230 $\pm$ 0.053 & 0.918 $\pm$ 0.008 & 2.377 $\pm$ 0.004 & 0.849 $\pm$ 0.011 & 1.491 $\pm$ 0.005 & 0.979 $\pm$ 0.004 & 3.220 $\pm$ 0.136 & 1.000 $\pm$ 0.000 & 4.392 $\pm$ 0.037 & 1.000 $\pm$ 0.000 & 4.652 $\pm$ 0.114 \\
\bottomrule
\end{tabular}%
}
}
\vspace{0.8em}
\textbf{Panel B: LAC}
\vspace{0.3em}
\centerline{
\resizebox{1.4\textwidth}{!}{%
\begin{tabular}{lcccccccccccc}
\toprule
Model & \multicolumn{2}{c}{Matched} & \multicolumn{2}{c}{Mismatch mean} & \multicolumn{2}{c}{Mismatch worst} & \multicolumn{2}{c}{Pooled} & \multicolumn{2}{c}{MaxThr} & \multicolumn{2}{c}{Latent-robust} \\
\cmidrule(lr){2-3} \cmidrule(lr){4-5} \cmidrule(lr){6-7} \cmidrule(lr){8-9} \cmidrule(lr){10-11} \cmidrule(lr){12-13}
 & Cov. & Size & Cov. & Size & Cov. & Size & Cov. & Size & Cov. & Size & Cov. & Size \\
\midrule
Qwen2.5-1.5B & 0.915 $\pm$ 0.023 & 1.876 $\pm$ 0.149 & 0.832 $\pm$ 0.014 & 1.339 $\pm$ 0.044 & 0.797 $\pm$ 0.016 & 1.166 $\pm$ 0.020 & 0.823 $\pm$ 0.006 & 1.255 $\pm$ 0.010 & 0.911 $\pm$ 0.024 & 1.850 $\pm$ 0.149 & 0.962 $\pm$ 0.023 & 2.652 $\pm$ 0.402 \\
Qwen2.5-3B & 0.913 $\pm$ 0.013 & 1.489 $\pm$ 0.049 & 0.846 $\pm$ 0.011 & 1.187 $\pm$ 0.011 & 0.821 $\pm$ 0.008 & 1.086 $\pm$ 0.008 & 0.838 $\pm$ 0.015 & 1.154 $\pm$ 0.019 & 0.910 $\pm$ 0.014 & 1.474 $\pm$ 0.048 & 0.998 $\pm$ 0.002 & 4.957 $\pm$ 0.012 \\
Qwen2.5-7B & 0.907 $\pm$ 0.005 & 1.369 $\pm$ 0.010 & 0.866 $\pm$ 0.012 & 1.153 $\pm$ 0.037 & 0.844 $\pm$ 0.009 & 1.074 $\pm$ 0.020 & 0.861 $\pm$ 0.011 & 1.120 $\pm$ 0.041 & 0.905 $\pm$ 0.015 & 1.353 $\pm$ 0.072 & 0.954 $\pm$ 0.015 & 2.007 $\pm$ 0.198 \\
DeepSeek-R1-7B & 0.895 $\pm$ 0.003 & 2.637 $\pm$ 0.029 & 0.759 $\pm$ 0.008 & 1.675 $\pm$ 0.036 & 0.708 $\pm$ 0.007 & 1.374 $\pm$ 0.022 & 0.752 $\pm$ 0.008 & 1.591 $\pm$ 0.005 & 0.892 $\pm$ 0.004 & 2.616 $\pm$ 0.027 & 0.952 $\pm$ 0.013 & 3.404 $\pm$ 0.198 \\
Gemma-2-2B & 0.919 $\pm$ 0.017 & 2.544 $\pm$ 0.193 & 0.816 $\pm$ 0.013 & 1.655 $\pm$ 0.092 & 0.755 $\pm$ 0.011 & 1.208 $\pm$ 0.037 & 0.799 $\pm$ 0.014 & 1.430 $\pm$ 0.080 & 0.943 $\pm$ 0.014 & 2.956 $\pm$ 0.196 & 0.973 $\pm$ 0.008 & 3.573 $\pm$ 0.139 \\
Llama-3.2-3B & 0.919 $\pm$ 0.012 & 1.993 $\pm$ 0.045 & 0.841 $\pm$ 0.014 & 1.470 $\pm$ 0.020 & 0.805 $\pm$ 0.014 & 1.252 $\pm$ 0.003 & 0.830 $\pm$ 0.013 & 1.372 $\pm$ 0.027 & 0.923 $\pm$ 0.016 & 2.064 $\pm$ 0.116 & 0.967 $\pm$ 0.013 & 2.826 $\pm$ 0.223 \\
\bottomrule
\end{tabular}%
}
}
\vspace{0.8em}
\textbf{Panel C: RAPS}
\vspace{0.3em}
\centerline{
\resizebox{1.4\textwidth}{!}{%
\begin{tabular}{lcccccccccccc}
\toprule
Model & \multicolumn{2}{c}{Matched} & \multicolumn{2}{c}{Mismatch mean} & \multicolumn{2}{c}{Mismatch worst} & \multicolumn{2}{c}{Pooled} & \multicolumn{2}{c}{MaxThr} & \multicolumn{2}{c}{Latent-robust} \\
\cmidrule(lr){2-3} \cmidrule(lr){4-5} \cmidrule(lr){6-7} \cmidrule(lr){8-9} \cmidrule(lr){10-11} \cmidrule(lr){12-13}
 & Cov. & Size & Cov. & Size & Cov. & Size & Cov. & Size & Cov. & Size & Cov. & Size \\
\midrule
Qwen2.5-1.5B & 0.967 $\pm$ 0.006 & 3.130 $\pm$ 0.053 & 0.893 $\pm$ 0.005 & 1.927 $\pm$ 0.015 & 0.827 $\pm$ 0.007 & 1.285 $\pm$ 0.006 & 0.930 $\pm$ 0.007 & 2.193 $\pm$ 0.031 & 0.967 $\pm$ 0.006 & 3.113 $\pm$ 0.038 & 0.983 $\pm$ 0.008 & 4.097 $\pm$ 0.105 \\
Qwen2.5-3B & 0.952 $\pm$ 0.005 & 2.235 $\pm$ 0.005 & 0.937 $\pm$ 0.008 & 2.084 $\pm$ 0.002 & 0.928 $\pm$ 0.009 & 2.017 $\pm$ 0.003 & 0.941 $\pm$ 0.014 & 2.124 $\pm$ 0.100 & 0.952 $\pm$ 0.005 & 2.235 $\pm$ 0.005 & 0.974 $\pm$ 0.006 & 3.417 $\pm$ 0.507 \\
Qwen2.5-7B & 0.963 $\pm$ 0.011 & 2.375 $\pm$ 0.003 & 0.932 $\pm$ 0.007 & 1.901 $\pm$ 0.006 & 0.891 $\pm$ 0.005 & 1.254 $\pm$ 0.008 & 0.940 $\pm$ 0.009 & 2.049 $\pm$ 0.011 & 0.963 $\pm$ 0.011 & 2.375 $\pm$ 0.003 & 0.982 $\pm$ 0.009 & 3.467 $\pm$ 0.462 \\
DeepSeek-R1-7B & 0.967 $\pm$ 0.008 & 4.046 $\pm$ 0.022 & 0.865 $\pm$ 0.019 & 2.645 $\pm$ 0.131 & 0.759 $\pm$ 0.024 & 1.625 $\pm$ 0.066 & 0.914 $\pm$ 0.006 & 3.081 $\pm$ 0.094 & 0.970 $\pm$ 0.006 & 4.104 $\pm$ 0.079 & 0.999 $\pm$ 0.002 & 4.885 $\pm$ 0.199 \\
Gemma-2-2B & 0.964 $\pm$ 0.003 & 3.490 $\pm$ 0.058 & 0.864 $\pm$ 0.007 & 2.237 $\pm$ 0.075 & 0.779 $\pm$ 0.008 & 1.304 $\pm$ 0.033 & 0.905 $\pm$ 0.010 & 2.509 $\pm$ 0.109 & 0.970 $\pm$ 0.003 & 3.725 $\pm$ 0.110 & 0.998 $\pm$ 0.002 & 4.833 $\pm$ 0.153 \\
Llama-3.2-3B & 0.969 $\pm$ 0.009 & 3.096 $\pm$ 0.115 & 0.905 $\pm$ 0.010 & 2.027 $\pm$ 0.018 & 0.849 $\pm$ 0.011 & 1.483 $\pm$ 0.007 & 0.936 $\pm$ 0.008 & 2.381 $\pm$ 0.020 & 0.967 $\pm$ 0.007 & 3.081 $\pm$ 0.102 & 0.989 $\pm$ 0.010 & 3.976 $\pm$ 0.424 \\
\bottomrule
\end{tabular}%
}
}
\caption{Score-function robustness on CommonsenseQA at deployment $\mathrm{nf4}@T=0.5$ with $\alpha=0.1$. Each panel reports mean $\pm$ standard deviation across runs.}
\label{tab:score_robustness_commonsense_qa}
\end{table*}
\begin{table*}[t]
\centering
\small
\textbf{Panel A: APS}
\vspace{0.3em}
\centerline{
\resizebox{1.4\textwidth}{!}{%
\begin{tabular}{lcccccccccccc}
\toprule
Model & \multicolumn{2}{c}{Matched} & \multicolumn{2}{c}{Mismatch mean} & \multicolumn{2}{c}{Mismatch worst} & \multicolumn{2}{c}{Pooled} & \multicolumn{2}{c}{MaxThr} & \multicolumn{2}{c}{Latent-robust} \\
\cmidrule(lr){2-3} \cmidrule(lr){4-5} \cmidrule(lr){6-7} \cmidrule(lr){8-9} \cmidrule(lr){10-11} \cmidrule(lr){12-13}
 & Cov. & Size & Cov. & Size & Cov. & Size & Cov. & Size & Cov. & Size & Cov. & Size \\
\midrule
Qwen2.5-1.5B & 0.985 $\pm$ 0.000 & 8.044 $\pm$ 0.103 & 0.944 $\pm$ 0.001 & 3.761 $\pm$ 0.008 & 0.903 $\pm$ 0.001 & 1.290 $\pm$ 0.003 & 0.973 $\pm$ 0.002 & 5.044 $\pm$ 0.303 & 0.992 $\pm$ 0.000 & 10.781 $\pm$ 0.002 & 0.996 $\pm$ 0.000 & 11.534 $\pm$ 0.148 \\
Qwen2.5-3B & 0.995 $\pm$ 0.000 & 8.189 $\pm$ 0.005 & 0.941 $\pm$ 0.002 & 3.292 $\pm$ 0.034 & 0.887 $\pm$ 0.002 & 1.571 $\pm$ 0.011 & 0.985 $\pm$ 0.001 & 3.850 $\pm$ 0.122 & 0.995 $\pm$ 0.000 & 8.189 $\pm$ 0.005 & 0.994 $\pm$ 0.000 & 8.124 $\pm$ 0.016 \\
Qwen2.5-7B & 0.990 $\pm$ 0.000 & 1.419 $\pm$ 0.001 & 0.981 $\pm$ 0.000 & 1.115 $\pm$ 0.001 & 0.976 $\pm$ 0.000 & 1.019 $\pm$ 0.000 & 0.990 $\pm$ 0.000 & 1.419 $\pm$ 0.001 & 0.990 $\pm$ 0.000 & 1.419 $\pm$ 0.001 & 0.988 $\pm$ 0.001 & 1.244 $\pm$ 0.059 \\
DeepSeek-R1-7B & 0.968 $\pm$ 0.001 & 3.904 $\pm$ 0.051 & 0.900 $\pm$ 0.001 & 2.004 $\pm$ 0.008 & 0.847 $\pm$ 0.004 & 1.294 $\pm$ 0.014 & 0.934 $\pm$ 0.001 & 2.230 $\pm$ 0.016 & 0.968 $\pm$ 0.001 & 3.923 $\pm$ 0.029 & 0.976 $\pm$ 0.001 & 4.888 $\pm$ 0.111 \\
Gemma-2-2B & 0.953 $\pm$ 0.001 & 6.924 $\pm$ 0.055 & 0.793 $\pm$ 0.000 & 2.601 $\pm$ 0.013 & 0.724 $\pm$ 0.000 & 1.104 $\pm$ 0.002 & 0.920 $\pm$ 0.001 & 4.358 $\pm$ 0.105 & 0.966 $\pm$ 0.001 & 8.233 $\pm$ 0.070 & 0.989 $\pm$ 0.002 & 10.921 $\pm$ 0.287 \\
Llama-3.2-3B & 0.965 $\pm$ 0.000 & 5.031 $\pm$ 0.057 & 0.879 $\pm$ 0.000 & 2.171 $\pm$ 0.010 & 0.819 $\pm$ 0.000 & 1.167 $\pm$ 0.002 & 0.930 $\pm$ 0.001 & 2.521 $\pm$ 0.055 & 0.969 $\pm$ 0.000 & 5.742 $\pm$ 0.065 & 0.976 $\pm$ 0.000 & 7.836 $\pm$ 0.139 \\
\bottomrule
\end{tabular}%
}
}
\vspace{0.8em}
\textbf{Panel B: LAC}
\vspace{0.3em}
\centerline{
\resizebox{1.4\textwidth}{!}{%
\begin{tabular}{lcccccccccccc}
\toprule
Model & \multicolumn{2}{c}{Matched} & \multicolumn{2}{c}{Mismatch mean} & \multicolumn{2}{c}{Mismatch worst} & \multicolumn{2}{c}{Pooled} & \multicolumn{2}{c}{MaxThr} & \multicolumn{2}{c}{Latent-robust} \\
\cmidrule(lr){2-3} \cmidrule(lr){4-5} \cmidrule(lr){6-7} \cmidrule(lr){8-9} \cmidrule(lr){10-11} \cmidrule(lr){12-13}
 & Cov. & Size & Cov. & Size & Cov. & Size & Cov. & Size & Cov. & Size & Cov. & Size \\
\midrule
Qwen2.5-1.5B & 0.899 $\pm$ 0.006 & 1.231 $\pm$ 0.042 & 0.884 $\pm$ 0.002 & 1.139 $\pm$ 0.013 & 0.876 $\pm$ 0.002 & 1.098 $\pm$ 0.007 & 0.883 $\pm$ 0.001 & 1.130 $\pm$ 0.004 & 0.901 $\pm$ 0.004 & 1.254 $\pm$ 0.033 & 0.947 $\pm$ 0.008 & 2.505 $\pm$ 0.319 \\
Qwen2.5-3B & 0.900 $\pm$ 0.002 & 1.659 $\pm$ 0.019 & 0.795 $\pm$ 0.003 & 1.168 $\pm$ 0.007 & 0.757 $\pm$ 0.001 & 1.061 $\pm$ 0.003 & 0.803 $\pm$ 0.006 & 1.185 $\pm$ 0.016 & 0.900 $\pm$ 0.002 & 1.658 $\pm$ 0.018 & 0.946 $\pm$ 0.002 & 2.342 $\pm$ 0.034 \\
Qwen2.5-7B & 0.908 $\pm$ 0.000 & 0.916 $\pm$ 0.000 & 0.946 $\pm$ 0.001 & 0.966 $\pm$ 0.000 & 0.881 $\pm$ 0.001 & 0.888 $\pm$ 0.000 & 0.964 $\pm$ 0.000 & 0.988 $\pm$ 0.000 & 0.965 $\pm$ 0.000 & 0.991 $\pm$ 0.000 & 0.966 $\pm$ 0.000 & 0.993 $\pm$ 0.000 \\
DeepSeek-R1-7B & 0.902 $\pm$ 0.004 & 1.701 $\pm$ 0.036 & 0.871 $\pm$ 0.002 & 1.472 $\pm$ 0.012 & 0.838 $\pm$ 0.002 & 1.251 $\pm$ 0.004 & 0.875 $\pm$ 0.001 & 1.452 $\pm$ 0.009 & 0.924 $\pm$ 0.002 & 2.024 $\pm$ 0.025 & 0.956 $\pm$ 0.003 & 3.026 $\pm$ 0.200 \\
Gemma-2-2B & 0.902 $\pm$ 0.004 & 3.485 $\pm$ 0.134 & 0.799 $\pm$ 0.001 & 2.251 $\pm$ 0.046 & 0.744 $\pm$ 0.000 & 1.232 $\pm$ 0.002 & 0.782 $\pm$ 0.000 & 1.611 $\pm$ 0.003 & 0.940 $\pm$ 0.004 & 5.953 $\pm$ 0.237 & 0.987 $\pm$ 0.005 & 10.787 $\pm$ 0.672 \\
Llama-3.2-3B & 0.899 $\pm$ 0.002 & 1.750 $\pm$ 0.024 & 0.829 $\pm$ 0.001 & 1.216 $\pm$ 0.005 & 0.815 $\pm$ 0.000 & 1.152 $\pm$ 0.002 & 0.833 $\pm$ 0.000 & 1.226 $\pm$ 0.002 & 0.898 $\pm$ 0.002 & 1.748 $\pm$ 0.026 & 0.947 $\pm$ 0.004 & 3.576 $\pm$ 0.266 \\
\bottomrule
\end{tabular}%
}
}
\vspace{0.8em}
\textbf{Panel C: RAPS}
\vspace{0.3em}
\centerline{
\resizebox{1.4\textwidth}{!}{%
\begin{tabular}{lcccccccccccc}
\toprule
Model & \multicolumn{2}{c}{Matched} & \multicolumn{2}{c}{Mismatch mean} & \multicolumn{2}{c}{Mismatch worst} & \multicolumn{2}{c}{Pooled} & \multicolumn{2}{c}{MaxThr} & \multicolumn{2}{c}{Latent-robust} \\
\cmidrule(lr){2-3} \cmidrule(lr){4-5} \cmidrule(lr){6-7} \cmidrule(lr){8-9} \cmidrule(lr){10-11} \cmidrule(lr){12-13}
 & Cov. & Size & Cov. & Size & Cov. & Size & Cov. & Size & Cov. & Size & Cov. & Size \\
\midrule
Qwen2.5-1.5B & 0.940 $\pm$ 0.001 & 2.548 $\pm$ 0.024 & 0.927 $\pm$ 0.001 & 1.866 $\pm$ 0.011 & 0.903 $\pm$ 0.001 & 1.287 $\pm$ 0.004 & 0.937 $\pm$ 0.001 & 2.310 $\pm$ 0.014 & 0.941 $\pm$ 0.001 & 2.653 $\pm$ 0.047 & 0.961 $\pm$ 0.005 & 4.588 $\pm$ 0.384 \\
Qwen2.5-3B & 0.933 $\pm$ 0.001 & 3.177 $\pm$ 0.007 & 0.904 $\pm$ 0.001 & 1.865 $\pm$ 0.011 & 0.884 $\pm$ 0.001 & 1.531 $\pm$ 0.003 & 0.909 $\pm$ 0.000 & 2.087 $\pm$ 0.001 & 0.933 $\pm$ 0.001 & 3.176 $\pm$ 0.007 & 0.980 $\pm$ 0.000 & 4.417 $\pm$ 0.013 \\
Qwen2.5-7B & 0.985 $\pm$ 0.000 & 1.113 $\pm$ 0.000 & 0.980 $\pm$ 0.000 & 1.049 $\pm$ 0.000 & 0.976 $\pm$ 0.000 & 1.018 $\pm$ 0.001 & 0.985 $\pm$ 0.000 & 1.113 $\pm$ 0.000 & 0.985 $\pm$ 0.000 & 1.113 $\pm$ 0.000 & 0.985 $\pm$ 0.000 & 1.113 $\pm$ 0.000 \\
DeepSeek-R1-7B & 0.952 $\pm$ 0.000 & 2.850 $\pm$ 0.024 & 0.903 $\pm$ 0.001 & 1.944 $\pm$ 0.005 & 0.851 $\pm$ 0.002 & 1.313 $\pm$ 0.002 & 0.937 $\pm$ 0.000 & 2.235 $\pm$ 0.004 & 0.955 $\pm$ 0.000 & 3.123 $\pm$ 0.016 & 0.967 $\pm$ 0.002 & 4.166 $\pm$ 0.123 \\
Gemma-2-2B & 0.924 $\pm$ 0.002 & 5.172 $\pm$ 0.077 & 0.825 $\pm$ 0.004 & 2.900 $\pm$ 0.069 & 0.727 $\pm$ 0.001 & 1.120 $\pm$ 0.003 & 0.892 $\pm$ 0.005 & 3.625 $\pm$ 0.251 & 0.928 $\pm$ 0.002 & 5.630 $\pm$ 0.147 & 0.976 $\pm$ 0.003 & 10.606 $\pm$ 0.241 \\
Llama-3.2-3B & 0.938 $\pm$ 0.001 & 2.746 $\pm$ 0.018 & 0.873 $\pm$ 0.000 & 1.599 $\pm$ 0.003 & 0.822 $\pm$ 0.001 & 1.179 $\pm$ 0.006 & 0.920 $\pm$ 0.000 & 2.006 $\pm$ 0.003 & 0.938 $\pm$ 0.001 & 2.745 $\pm$ 0.018 & 0.957 $\pm$ 0.004 & 4.506 $\pm$ 0.374 \\
\bottomrule
\end{tabular}%
}
}
\caption{Score-function robustness on DBPedia-14 at deployment $\mathrm{nf4}@T=0.5$ with $\alpha=0.1$. Each panel reports mean $\pm$ standard deviation across runs.}
\label{tab:score_robustness_dbpedia_14}
\end{table*}
\begin{table*}[t]
\centering
\small
\textbf{Panel A: APS}
\vspace{0.3em}
\centerline{
\resizebox{1.4\textwidth}{!}{%
\begin{tabular}{lcccccccccccc}
\toprule
Model & \multicolumn{2}{c}{Matched} & \multicolumn{2}{c}{Mismatch mean} & \multicolumn{2}{c}{Mismatch worst} & \multicolumn{2}{c}{Pooled} & \multicolumn{2}{c}{MaxThr} & \multicolumn{2}{c}{Latent-robust} \\
\cmidrule(lr){2-3} \cmidrule(lr){4-5} \cmidrule(lr){6-7} \cmidrule(lr){8-9} \cmidrule(lr){10-11} \cmidrule(lr){12-13}
 & Cov. & Size & Cov. & Size & Cov. & Size & Cov. & Size & Cov. & Size & Cov. & Size \\
\midrule
Qwen2.5-1.5B & 0.981 $\pm$ 0.001 & 9.011 $\pm$ 0.015 & 0.686 $\pm$ 0.001 & 5.315 $\pm$ 0.035 & 0.519 $\pm$ 0.003 & 3.411 $\pm$ 0.025 & 0.917 $\pm$ 0.006 & 7.964 $\pm$ 0.128 & 0.981 $\pm$ 0.001 & 9.011 $\pm$ 0.015 & 0.989 $\pm$ 0.000 & 9.196 $\pm$ 0.016 \\
Qwen2.5-3B & 0.937 $\pm$ 0.001 & 7.769 $\pm$ 0.028 & 0.694 $\pm$ 0.006 & 4.552 $\pm$ 0.050 & 0.530 $\pm$ 0.004 & 2.741 $\pm$ 0.020 & 0.888 $\pm$ 0.004 & 6.821 $\pm$ 0.085 & 0.937 $\pm$ 0.001 & 7.769 $\pm$ 0.028 & 0.937 $\pm$ 0.001 & 7.769 $\pm$ 0.028 \\
Qwen2.5-7B & 0.956 $\pm$ 0.001 & 7.573 $\pm$ 0.014 & 0.740 $\pm$ 0.001 & 4.250 $\pm$ 0.024 & 0.602 $\pm$ 0.002 & 2.611 $\pm$ 0.007 & 0.878 $\pm$ 0.013 & 5.806 $\pm$ 0.179 & 0.956 $\pm$ 0.001 & 7.573 $\pm$ 0.014 & 0.956 $\pm$ 0.001 & 7.573 $\pm$ 0.014 \\
DeepSeek-R1-7B & 0.974 $\pm$ 0.005 & 8.806 $\pm$ 0.058 & 0.669 $\pm$ 0.010 & 4.883 $\pm$ 0.085 & 0.511 $\pm$ 0.010 & 3.126 $\pm$ 0.060 & 0.885 $\pm$ 0.013 & 7.349 $\pm$ 0.140 & 0.978 $\pm$ 0.006 & 8.875 $\pm$ 0.083 & 0.999 $\pm$ 0.000 & 9.444 $\pm$ 0.017 \\
Gemma-2-2B & 0.984 $\pm$ 0.002 & 8.948 $\pm$ 0.027 & 0.796 $\pm$ 0.004 & 6.277 $\pm$ 0.034 & 0.667 $\pm$ 0.006 & 4.688 $\pm$ 0.071 & 0.956 $\pm$ 0.005 & 8.365 $\pm$ 0.112 & 0.990 $\pm$ 0.002 & 9.082 $\pm$ 0.024 & 1.000 $\pm$ 0.000 & 9.432 $\pm$ 0.013 \\
Llama-3.2-3B & 0.980 $\pm$ 0.002 & 8.677 $\pm$ 0.061 & 0.780 $\pm$ 0.005 & 5.773 $\pm$ 0.066 & 0.642 $\pm$ 0.009 & 4.127 $\pm$ 0.086 & 0.899 $\pm$ 0.007 & 7.204 $\pm$ 0.105 & 0.981 $\pm$ 0.001 & 8.686 $\pm$ 0.032 & 0.998 $\pm$ 0.000 & 9.311 $\pm$ 0.011 \\
\bottomrule
\end{tabular}%
}
}
\vspace{0.8em}
\textbf{Panel B: LAC}
\vspace{0.3em}
\centerline{
\resizebox{1.4\textwidth}{!}{%
\begin{tabular}{lcccccccccccc}
\toprule
Model & \multicolumn{2}{c}{Matched} & \multicolumn{2}{c}{Mismatch mean} & \multicolumn{2}{c}{Mismatch worst} & \multicolumn{2}{c}{Pooled} & \multicolumn{2}{c}{MaxThr} & \multicolumn{2}{c}{Latent-robust} \\
\cmidrule(lr){2-3} \cmidrule(lr){4-5} \cmidrule(lr){6-7} \cmidrule(lr){8-9} \cmidrule(lr){10-11} \cmidrule(lr){12-13}
 & Cov. & Size & Cov. & Size & Cov. & Size & Cov. & Size & Cov. & Size & Cov. & Size \\
\midrule
Qwen2.5-1.5B & 0.901 $\pm$ 0.009 & 7.765 $\pm$ 0.072 & 0.597 $\pm$ 0.005 & 4.293 $\pm$ 0.019 & 0.468 $\pm$ 0.003 & 2.876 $\pm$ 0.004 & 0.650 $\pm$ 0.025 & 4.829 $\pm$ 0.236 & 0.908 $\pm$ 0.007 & 7.860 $\pm$ 0.043 & 0.954 $\pm$ 0.003 & 8.533 $\pm$ 0.022 \\
Qwen2.5-3B & 0.899 $\pm$ 0.011 & 6.963 $\pm$ 0.139 & 0.628 $\pm$ 0.009 & 3.870 $\pm$ 0.175 & 0.480 $\pm$ 0.004 & 2.261 $\pm$ 0.033 & 0.670 $\pm$ 0.004 & 4.095 $\pm$ 0.063 & 0.967 $\pm$ 0.056 & 8.655 $\pm$ 1.406 & 1.000 $\pm$ 0.000 & 9.464 $\pm$ 0.017 \\
Qwen2.5-7B & 0.899 $\pm$ 0.007 & 6.167 $\pm$ 0.092 & 0.678 $\pm$ 0.002 & 3.454 $\pm$ 0.021 & 0.563 $\pm$ 0.002 & 2.286 $\pm$ 0.005 & 0.706 $\pm$ 0.009 & 3.586 $\pm$ 0.082 & 0.898 $\pm$ 0.008 & 6.148 $\pm$ 0.092 & 1.000 $\pm$ 0.001 & 9.461 $\pm$ 0.018 \\
DeepSeek-R1-7B & 0.904 $\pm$ 0.006 & 7.643 $\pm$ 0.050 & 0.583 $\pm$ 0.007 & 3.949 $\pm$ 0.056 & 0.443 $\pm$ 0.005 & 2.508 $\pm$ 0.032 & 0.651 $\pm$ 0.017 & 4.462 $\pm$ 0.165 & 0.908 $\pm$ 0.009 & 7.719 $\pm$ 0.099 & 0.958 $\pm$ 0.007 & 8.491 $\pm$ 0.095 \\
Gemma-2-2B & 0.899 $\pm$ 0.003 & 7.495 $\pm$ 0.062 & 0.663 $\pm$ 0.002 & 4.717 $\pm$ 0.028 & 0.507 $\pm$ 0.003 & 3.037 $\pm$ 0.020 & 0.727 $\pm$ 0.001 & 5.293 $\pm$ 0.020 & 0.935 $\pm$ 0.005 & 8.028 $\pm$ 0.074 & 0.972 $\pm$ 0.001 & 8.656 $\pm$ 0.033 \\
Llama-3.2-3B & 0.902 $\pm$ 0.003 & 7.239 $\pm$ 0.037 & 0.644 $\pm$ 0.004 & 4.213 $\pm$ 0.048 & 0.525 $\pm$ 0.006 & 2.997 $\pm$ 0.050 & 0.694 $\pm$ 0.009 & 4.629 $\pm$ 0.100 & 0.902 $\pm$ 0.003 & 7.232 $\pm$ 0.043 & 0.957 $\pm$ 0.004 & 8.152 $\pm$ 0.075 \\
\bottomrule
\end{tabular}%
}
}
\vspace{0.8em}
\textbf{Panel C: RAPS}
\vspace{0.3em}
\centerline{
\resizebox{1.4\textwidth}{!}{%
\begin{tabular}{lcccccccccccc}
\toprule
Model & \multicolumn{2}{c}{Matched} & \multicolumn{2}{c}{Mismatch mean} & \multicolumn{2}{c}{Mismatch worst} & \multicolumn{2}{c}{Pooled} & \multicolumn{2}{c}{MaxThr} & \multicolumn{2}{c}{Latent-robust} \\
\cmidrule(lr){2-3} \cmidrule(lr){4-5} \cmidrule(lr){6-7} \cmidrule(lr){8-9} \cmidrule(lr){10-11} \cmidrule(lr){12-13}
 & Cov. & Size & Cov. & Size & Cov. & Size & Cov. & Size & Cov. & Size & Cov. & Size \\
\midrule
Qwen2.5-1.5B & 0.964 $\pm$ 0.006 & 8.928 $\pm$ 0.067 & 0.887 $\pm$ 0.004 & 7.979 $\pm$ 0.019 & 0.801 $\pm$ 0.007 & 6.907 $\pm$ 0.018 & 0.912 $\pm$ 0.011 & 8.222 $\pm$ 0.160 & 0.963 $\pm$ 0.006 & 8.925 $\pm$ 0.064 & 1.000 $\pm$ 0.000 & 9.465 $\pm$ 0.016 \\
Qwen2.5-3B & 0.958 $\pm$ 0.005 & 8.668 $\pm$ 0.053 & 0.920 $\pm$ 0.015 & 8.024 $\pm$ 0.223 & 0.868 $\pm$ 0.032 & 7.164 $\pm$ 0.500 & 0.921 $\pm$ 0.016 & 8.010 $\pm$ 0.178 & 0.964 $\pm$ 0.008 & 8.748 $\pm$ 0.102 & 1.000 $\pm$ 0.000 & 9.461 $\pm$ 0.010 \\
Qwen2.5-7B & 0.942 $\pm$ 0.009 & 7.931 $\pm$ 0.096 & 0.892 $\pm$ 0.010 & 6.959 $\pm$ 0.156 & 0.821 $\pm$ 0.024 & 5.684 $\pm$ 0.427 & 0.903 $\pm$ 0.020 & 7.108 $\pm$ 0.303 & 0.943 $\pm$ 0.008 & 7.935 $\pm$ 0.080 & 0.987 $\pm$ 0.002 & 8.972 $\pm$ 0.004 \\
DeepSeek-R1-7B & 0.961 $\pm$ 0.003 & 8.696 $\pm$ 0.018 & 0.871 $\pm$ 0.012 & 7.339 $\pm$ 0.156 & 0.773 $\pm$ 0.015 & 5.994 $\pm$ 0.192 & 0.910 $\pm$ 0.008 & 7.834 $\pm$ 0.082 & 0.962 $\pm$ 0.005 & 8.712 $\pm$ 0.049 & 1.000 $\pm$ 0.000 & 9.458 $\pm$ 0.005 \\
Gemma-2-2B & 0.959 $\pm$ 0.005 & 8.783 $\pm$ 0.078 & 0.869 $\pm$ 0.009 & 7.300 $\pm$ 0.107 & 0.804 $\pm$ 0.008 & 6.252 $\pm$ 0.094 & 0.906 $\pm$ 0.011 & 7.861 $\pm$ 0.122 & 0.961 $\pm$ 0.005 & 8.817 $\pm$ 0.050 & 1.000 $\pm$ 0.000 & 9.454 $\pm$ 0.033 \\
Llama-3.2-3B & 0.954 $\pm$ 0.003 & 8.590 $\pm$ 0.106 & 0.853 $\pm$ 0.006 & 6.890 $\pm$ 0.081 & 0.775 $\pm$ 0.007 & 5.639 $\pm$ 0.098 & 0.878 $\pm$ 0.009 & 7.276 $\pm$ 0.126 & 0.954 $\pm$ 0.003 & 8.555 $\pm$ 0.078 & 0.998 $\pm$ 0.003 & 9.387 $\pm$ 0.118 \\
\bottomrule
\end{tabular}%
}
}
\caption{Score-function robustness on MMLU-Pro at deployment $\mathrm{nf4}@T=0.5$ with $\alpha=0.1$. Each panel reports mean $\pm$ standard deviation across runs.}
\label{tab:score_robustness_mmlu_pro}
\end{table*}
\begin{table*}[t]
\centering
\small
\textbf{Panel A: APS}
\vspace{0.3em}
\centerline{
\resizebox{1.4\textwidth}{!}{%
\begin{tabular}{lcccccccccccc}
\toprule
Model & \multicolumn{2}{c}{Matched} & \multicolumn{2}{c}{Mismatch mean} & \multicolumn{2}{c}{Mismatch worst} & \multicolumn{2}{c}{Pooled} & \multicolumn{2}{c}{MaxThr} & \multicolumn{2}{c}{Latent-robust} \\
\cmidrule(lr){2-3} \cmidrule(lr){4-5} \cmidrule(lr){6-7} \cmidrule(lr){8-9} \cmidrule(lr){10-11} \cmidrule(lr){12-13}
 & Cov. & Size & Cov. & Size & Cov. & Size & Cov. & Size & Cov. & Size & Cov. & Size \\
\midrule
Qwen2.5-1.5B & 1.000 $\pm$ 0.000 & 4.327 $\pm$ 0.155 & 0.987 $\pm$ 0.001 & 2.130 $\pm$ 0.008 & 0.967 $\pm$ 0.002 & 1.144 $\pm$ 0.005 & 1.000 $\pm$ 0.000 & 3.502 $\pm$ 0.030 & 1.000 $\pm$ 0.000 & 5.463 $\pm$ 0.012 & 1.000 $\pm$ 0.000 & 5.366 $\pm$ 0.040 \\
Qwen2.5-3B & 0.985 $\pm$ 0.001 & 2.025 $\pm$ 0.013 & 0.962 $\pm$ 0.001 & 1.429 $\pm$ 0.011 & 0.937 $\pm$ 0.002 & 1.133 $\pm$ 0.006 & 0.985 $\pm$ 0.001 & 2.025 $\pm$ 0.013 & 0.985 $\pm$ 0.001 & 2.025 $\pm$ 0.013 & 0.985 $\pm$ 0.001 & 1.964 $\pm$ 0.036 \\
Qwen2.5-7B & 0.999 $\pm$ 0.000 & 1.665 $\pm$ 0.010 & 0.995 $\pm$ 0.000 & 1.306 $\pm$ 0.003 & 0.989 $\pm$ 0.001 & 1.124 $\pm$ 0.001 & 0.999 $\pm$ 0.000 & 1.665 $\pm$ 0.010 & 0.999 $\pm$ 0.000 & 1.665 $\pm$ 0.010 & 0.999 $\pm$ 0.000 & 1.615 $\pm$ 0.024 \\
DeepSeek-R1-7B & 0.999 $\pm$ 0.000 & 3.390 $\pm$ 0.066 & 0.950 $\pm$ 0.001 & 1.677 $\pm$ 0.015 & 0.909 $\pm$ 0.003 & 1.149 $\pm$ 0.005 & 0.993 $\pm$ 0.001 & 2.169 $\pm$ 0.099 & 0.999 $\pm$ 0.000 & 3.431 $\pm$ 0.060 & 0.999 $\pm$ 0.000 & 3.864 $\pm$ 0.051 \\
Gemma-2-2B & 0.999 $\pm$ 0.000 & 4.344 $\pm$ 0.016 & 0.908 $\pm$ 0.001 & 1.730 $\pm$ 0.004 & 0.871 $\pm$ 0.002 & 1.064 $\pm$ 0.000 & 0.991 $\pm$ 0.001 & 3.191 $\pm$ 0.082 & 0.999 $\pm$ 0.000 & 4.344 $\pm$ 0.016 & 0.998 $\pm$ 0.000 & 4.275 $\pm$ 0.008 \\
Llama-3.2-3B & 0.997 $\pm$ 0.001 & 4.203 $\pm$ 0.059 & 0.925 $\pm$ 0.000 & 1.818 $\pm$ 0.008 & 0.883 $\pm$ 0.000 & 1.109 $\pm$ 0.003 & 0.981 $\pm$ 0.001 & 2.606 $\pm$ 0.078 & 0.998 $\pm$ 0.000 & 4.511 $\pm$ 0.008 & 0.999 $\pm$ 0.001 & 5.008 $\pm$ 0.020 \\
\bottomrule
\end{tabular}%
}
}
\vspace{0.8em}
\textbf{Panel B: LAC}
\vspace{0.3em}
\centerline{
\resizebox{1.4\textwidth}{!}{%
\begin{tabular}{lcccccccccccc}
\toprule
Model & \multicolumn{2}{c}{Matched} & \multicolumn{2}{c}{Mismatch mean} & \multicolumn{2}{c}{Mismatch worst} & \multicolumn{2}{c}{Pooled} & \multicolumn{2}{c}{MaxThr} & \multicolumn{2}{c}{Latent-robust} \\
\cmidrule(lr){2-3} \cmidrule(lr){4-5} \cmidrule(lr){6-7} \cmidrule(lr){8-9} \cmidrule(lr){10-11} \cmidrule(lr){12-13}
 & Cov. & Size & Cov. & Size & Cov. & Size & Cov. & Size & Cov. & Size & Cov. & Size \\
\midrule
Qwen2.5-1.5B & 0.907 $\pm$ 0.000 & 0.975 $\pm$ 0.002 & 0.905 $\pm$ 0.002 & 0.974 $\pm$ 0.001 & 0.866 $\pm$ 0.003 & 0.909 $\pm$ 0.005 & 0.915 $\pm$ 0.001 & 0.990 $\pm$ 0.004 & 0.927 $\pm$ 0.000 & 1.012 $\pm$ 0.002 & 0.940 $\pm$ 0.010 & 1.049 $\pm$ 0.029 \\
Qwen2.5-3B & 0.898 $\pm$ 0.006 & 1.017 $\pm$ 0.010 & 0.899 $\pm$ 0.002 & 1.019 $\pm$ 0.003 & 0.893 $\pm$ 0.001 & 1.006 $\pm$ 0.001 & 0.897 $\pm$ 0.002 & 1.014 $\pm$ 0.003 & 0.913 $\pm$ 0.004 & 1.052 $\pm$ 0.008 & 0.961 $\pm$ 0.001 & 1.265 $\pm$ 0.008 \\
Qwen2.5-7B & 0.902 $\pm$ 0.003 & 0.919 $\pm$ 0.003 & 0.939 $\pm$ 0.002 & 0.969 $\pm$ 0.003 & 0.911 $\pm$ 0.009 & 0.930 $\pm$ 0.012 & 0.944 $\pm$ 0.001 & 0.976 $\pm$ 0.000 & 0.950 $\pm$ 0.001 & 0.986 $\pm$ 0.002 & 0.955 $\pm$ 0.001 & 0.998 $\pm$ 0.000 \\
DeepSeek-R1-7B & 0.899 $\pm$ 0.003 & 1.119 $\pm$ 0.003 & 0.880 $\pm$ 0.000 & 1.072 $\pm$ 0.005 & 0.871 $\pm$ 0.001 & 1.049 $\pm$ 0.003 & 0.879 $\pm$ 0.002 & 1.067 $\pm$ 0.001 & 0.900 $\pm$ 0.001 & 1.121 $\pm$ 0.003 & 0.946 $\pm$ 0.006 & 1.277 $\pm$ 0.035 \\
Gemma-2-2B & 0.900 $\pm$ 0.001 & 1.202 $\pm$ 0.009 & 0.870 $\pm$ 0.002 & 1.070 $\pm$ 0.002 & 0.857 $\pm$ 0.002 & 1.030 $\pm$ 0.002 & 0.868 $\pm$ 0.003 & 1.053 $\pm$ 0.005 & 0.899 $\pm$ 0.001 & 1.195 $\pm$ 0.008 & 0.942 $\pm$ 0.003 & 1.674 $\pm$ 0.014 \\
Llama-3.2-3B & 0.902 $\pm$ 0.001 & 1.169 $\pm$ 0.007 & 0.871 $\pm$ 0.001 & 1.081 $\pm$ 0.007 & 0.858 $\pm$ 0.001 & 1.048 $\pm$ 0.008 & 0.865 $\pm$ 0.001 & 1.066 $\pm$ 0.004 & 0.903 $\pm$ 0.001 & 1.170 $\pm$ 0.006 & 0.943 $\pm$ 0.007 & 1.451 $\pm$ 0.084 \\
\bottomrule
\end{tabular}%
}
}
\vspace{0.8em}
\textbf{Panel C: RAPS}
\vspace{0.3em}
\centerline{
\resizebox{1.4\textwidth}{!}{%
\begin{tabular}{lcccccccccccc}
\toprule
Model & \multicolumn{2}{c}{Matched} & \multicolumn{2}{c}{Mismatch mean} & \multicolumn{2}{c}{Mismatch worst} & \multicolumn{2}{c}{Pooled} & \multicolumn{2}{c}{MaxThr} & \multicolumn{2}{c}{Latent-robust} \\
\cmidrule(lr){2-3} \cmidrule(lr){4-5} \cmidrule(lr){6-7} \cmidrule(lr){8-9} \cmidrule(lr){10-11} \cmidrule(lr){12-13}
 & Cov. & Size & Cov. & Size & Cov. & Size & Cov. & Size & Cov. & Size & Cov. & Size \\
\midrule
Qwen2.5-1.5B & 0.996 $\pm$ 0.000 & 1.998 $\pm$ 0.002 & 0.986 $\pm$ 0.001 & 1.404 $\pm$ 0.006 & 0.967 $\pm$ 0.002 & 1.141 $\pm$ 0.005 & 0.996 $\pm$ 0.000 & 1.877 $\pm$ 0.018 & 0.996 $\pm$ 0.000 & 1.998 $\pm$ 0.002 & 0.997 $\pm$ 0.000 & 2.172 $\pm$ 0.005 \\
Qwen2.5-3B & 0.981 $\pm$ 0.002 & 2.027 $\pm$ 0.012 & 0.962 $\pm$ 0.002 & 1.513 $\pm$ 0.055 & 0.937 $\pm$ 0.001 & 1.129 $\pm$ 0.003 & 0.974 $\pm$ 0.001 & 1.438 $\pm$ 0.008 & 0.984 $\pm$ 0.001 & 2.107 $\pm$ 0.024 & 0.986 $\pm$ 0.001 & 2.212 $\pm$ 0.006 \\
Qwen2.5-7B & 0.999 $\pm$ 0.000 & 1.339 $\pm$ 0.003 & 0.996 $\pm$ 0.000 & 1.231 $\pm$ 0.002 & 0.989 $\pm$ 0.000 & 1.126 $\pm$ 0.002 & 0.999 $\pm$ 0.000 & 1.339 $\pm$ 0.003 & 0.999 $\pm$ 0.000 & 1.339 $\pm$ 0.003 & 0.999 $\pm$ 0.000 & 1.339 $\pm$ 0.003 \\
DeepSeek-R1-7B & 0.993 $\pm$ 0.001 & 2.214 $\pm$ 0.039 & 0.950 $\pm$ 0.001 & 1.461 $\pm$ 0.011 & 0.911 $\pm$ 0.005 & 1.157 $\pm$ 0.010 & 0.988 $\pm$ 0.000 & 1.821 $\pm$ 0.029 & 0.993 $\pm$ 0.001 & 2.217 $\pm$ 0.033 & 0.995 $\pm$ 0.000 & 2.537 $\pm$ 0.011 \\
Gemma-2-2B & 0.964 $\pm$ 0.002 & 2.161 $\pm$ 0.013 & 0.915 $\pm$ 0.002 & 1.536 $\pm$ 0.000 & 0.871 $\pm$ 0.002 & 1.064 $\pm$ 0.000 & 0.957 $\pm$ 0.002 & 1.865 $\pm$ 0.037 & 0.965 $\pm$ 0.002 & 2.183 $\pm$ 0.002 & 0.972 $\pm$ 0.001 & 2.435 $\pm$ 0.021 \\
Llama-3.2-3B & 0.975 $\pm$ 0.002 & 2.235 $\pm$ 0.013 & 0.929 $\pm$ 0.000 & 1.589 $\pm$ 0.002 & 0.885 $\pm$ 0.002 & 1.114 $\pm$ 0.003 & 0.964 $\pm$ 0.002 & 1.928 $\pm$ 0.006 & 0.976 $\pm$ 0.002 & 2.270 $\pm$ 0.000 & 0.981 $\pm$ 0.002 & 2.656 $\pm$ 0.023 \\
\bottomrule
\end{tabular}%
}
}
\caption{Score-function robustness on MobileActions at deployment $\mathrm{nf4}@T=0.5$ with $\alpha=0.1$. Each panel reports mean $\pm$ standard deviation across runs.}
\label{tab:score_robustness_mobile_actions}
\end{table*}

Tables~\ref{tab:score_robustness_commonsense_qa}--\ref{tab:score_robustness_mobile_actions} compare APS, LAC, and RAPS at deployment $\mathrm{nf4}@T=0.5$ with $\alpha=0.1$. The qualitative mismatch pattern persists across score functions, although the absolute level of conservativeness differs across APS, LAC, and RAPS. In particular, APS and RAPS are often more conservative than LAC, but neither removes the main evaluative issue: matched-only reporting can still overstate deployment-time reliability under calibration--deployment mismatch. This robustness check therefore strengthens the claim that the phenomenon is not an artifact of one particular nonconformity score.

The integrated mean $\pm$ standard deviation reporting in Tables~\ref{tab:score_robustness_commonsense_qa}--\ref{tab:score_robustness_mobile_actions} also shows that the score-function differences are typically larger than the run-to-run variability. This is consistent with the view that the observed changes reflect systematic properties of the score construction rather than instability of the benchmark.

\subsection{Robustness to the target miscoverage level}
\label{app:alpha_robustness}

\begin{table*}[t]
\centering
\small
\textbf{Panel A: $\alpha=0.05$}
\vspace{0.3em}
\centerline{
\resizebox{1.4\textwidth}{!}{%
\begin{tabular}{lcccccccccccc}
\toprule
Model & \multicolumn{2}{c}{Matched} & \multicolumn{2}{c}{Mismatch mean} & \multicolumn{2}{c}{Mismatch worst} & \multicolumn{2}{c}{Pooled} & \multicolumn{2}{c}{MaxThr} & \multicolumn{2}{c}{Latent-robust} \\
\cmidrule(lr){2-3} \cmidrule(lr){4-5} \cmidrule(lr){6-7} \cmidrule(lr){8-9} \cmidrule(lr){10-11} \cmidrule(lr){12-13}
 & Gap & Size & Gap & Size & Gap & Size & Gap & Size & Gap & Size & Gap & Size \\
\midrule
Qwen2.5-1.5B & 0.009 $\pm$ 0.014 & 2.597 $\pm$ 0.237 & -0.084 $\pm$ 0.010 & 1.630 $\pm$ 0.058 & -0.133 $\pm$ 0.008 & 1.240 $\pm$ 0.013 & -0.085 $\pm$ 0.018 & 1.503 $\pm$ 0.082 & 0.023 $\pm$ 0.013 & 2.813 $\pm$ 0.167 & 0.037 $\pm$ 0.007 & 3.469 $\pm$ 0.212 \\
Qwen2.5-3B & 0.050 $\pm$ 0.000 & 5.000 $\pm$ 0.000 & -0.069 $\pm$ 0.010 & 1.723 $\pm$ 0.170 & -0.117 $\pm$ 0.009 & 1.137 $\pm$ 0.003 & -0.075 $\pm$ 0.006 & 1.291 $\pm$ 0.008 & 0.050 $\pm$ 0.000 & 5.000 $\pm$ 0.000 & 0.049 $\pm$ 0.001 & 4.994 $\pm$ 0.005 \\
Qwen2.5-7B & 0.007 $\pm$ 0.009 & 2.051 $\pm$ 0.084 & -0.051 $\pm$ 0.006 & 1.390 $\pm$ 0.024 & -0.078 $\pm$ 0.006 & 1.169 $\pm$ 0.009 & -0.054 $\pm$ 0.012 & 1.301 $\pm$ 0.061 & 0.006 $\pm$ 0.008 & 2.034 $\pm$ 0.063 & 0.049 $\pm$ 0.001 & 4.986 $\pm$ 0.018 \\
DeepSeek-R1-7B & -0.003 $\pm$ 0.013 & 3.353 $\pm$ 0.107 & -0.148 $\pm$ 0.008 & 2.006 $\pm$ 0.045 & -0.223 $\pm$ 0.005 & 1.475 $\pm$ 0.030 & -0.124 $\pm$ 0.009 & 2.008 $\pm$ 0.081 & 0.001 $\pm$ 0.013 & 3.399 $\pm$ 0.076 & 0.028 $\pm$ 0.001 & 4.078 $\pm$ 0.169 \\
Gemma-2-2B & 0.006 $\pm$ 0.017 & 3.225 $\pm$ 0.152 & -0.110 $\pm$ 0.012 & 1.904 $\pm$ 0.079 & -0.183 $\pm$ 0.009 & 1.267 $\pm$ 0.041 & -0.092 $\pm$ 0.019 & 1.903 $\pm$ 0.202 & 0.028 $\pm$ 0.005 & 3.695 $\pm$ 0.076 & 0.041 $\pm$ 0.003 & 4.160 $\pm$ 0.075 \\
Llama-3.2-3B & 0.017 $\pm$ 0.004 & 2.736 $\pm$ 0.019 & -0.076 $\pm$ 0.009 & 1.757 $\pm$ 0.025 & -0.126 $\pm$ 0.011 & 1.340 $\pm$ 0.002 & -0.079 $\pm$ 0.005 & 1.630 $\pm$ 0.070 & 0.023 $\pm$ 0.001 & 2.908 $\pm$ 0.057 & 0.039 $\pm$ 0.003 & 3.566 $\pm$ 0.077 \\
\bottomrule
\end{tabular}%
}
}
\vspace{0.8em}
\textbf{Panel B: $\alpha=0.1$}
\vspace{0.3em}
\centerline{
\resizebox{1.4\textwidth}{!}{%
\begin{tabular}{lcccccccccccc}
\toprule
Model & \multicolumn{2}{c}{Matched} & \multicolumn{2}{c}{Mismatch mean} & \multicolumn{2}{c}{Mismatch worst} & \multicolumn{2}{c}{Pooled} & \multicolumn{2}{c}{MaxThr} & \multicolumn{2}{c}{Latent-robust} \\
\cmidrule(lr){2-3} \cmidrule(lr){4-5} \cmidrule(lr){6-7} \cmidrule(lr){8-9} \cmidrule(lr){10-11} \cmidrule(lr){12-13}
 & Gap & Size & Gap & Size & Gap & Size & Gap & Size & Gap & Size & Gap & Size \\
\midrule
Qwen2.5-1.5B & 0.015 $\pm$ 0.023 & 1.876 $\pm$ 0.149 & -0.068 $\pm$ 0.014 & 1.339 $\pm$ 0.044 & -0.103 $\pm$ 0.016 & 1.166 $\pm$ 0.020 & -0.077 $\pm$ 0.006 & 1.255 $\pm$ 0.010 & 0.011 $\pm$ 0.024 & 1.850 $\pm$ 0.149 & 0.062 $\pm$ 0.023 & 2.652 $\pm$ 0.402 \\
Qwen2.5-3B & 0.013 $\pm$ 0.013 & 1.489 $\pm$ 0.049 & -0.054 $\pm$ 0.011 & 1.187 $\pm$ 0.011 & -0.079 $\pm$ 0.008 & 1.086 $\pm$ 0.008 & -0.062 $\pm$ 0.015 & 1.154 $\pm$ 0.019 & 0.010 $\pm$ 0.014 & 1.474 $\pm$ 0.048 & 0.098 $\pm$ 0.002 & 4.957 $\pm$ 0.012 \\
Qwen2.5-7B & 0.007 $\pm$ 0.005 & 1.369 $\pm$ 0.010 & -0.034 $\pm$ 0.012 & 1.153 $\pm$ 0.037 & -0.056 $\pm$ 0.009 & 1.074 $\pm$ 0.020 & -0.039 $\pm$ 0.011 & 1.120 $\pm$ 0.041 & 0.005 $\pm$ 0.015 & 1.353 $\pm$ 0.072 & 0.054 $\pm$ 0.015 & 2.007 $\pm$ 0.198 \\
DeepSeek-R1-7B & -0.005 $\pm$ 0.003 & 2.637 $\pm$ 0.029 & -0.141 $\pm$ 0.008 & 1.675 $\pm$ 0.036 & -0.192 $\pm$ 0.007 & 1.374 $\pm$ 0.022 & -0.148 $\pm$ 0.008 & 1.591 $\pm$ 0.005 & -0.008 $\pm$ 0.004 & 2.616 $\pm$ 0.027 & 0.052 $\pm$ 0.013 & 3.404 $\pm$ 0.198 \\
Gemma-2-2B & 0.019 $\pm$ 0.017 & 2.544 $\pm$ 0.193 & -0.084 $\pm$ 0.013 & 1.655 $\pm$ 0.092 & -0.145 $\pm$ 0.011 & 1.208 $\pm$ 0.037 & -0.101 $\pm$ 0.014 & 1.430 $\pm$ 0.080 & 0.043 $\pm$ 0.014 & 2.956 $\pm$ 0.196 & 0.073 $\pm$ 0.008 & 3.573 $\pm$ 0.139 \\
Llama-3.2-3B & 0.019 $\pm$ 0.012 & 1.993 $\pm$ 0.045 & -0.059 $\pm$ 0.014 & 1.470 $\pm$ 0.020 & -0.095 $\pm$ 0.014 & 1.252 $\pm$ 0.003 & -0.070 $\pm$ 0.013 & 1.372 $\pm$ 0.027 & 0.023 $\pm$ 0.016 & 2.064 $\pm$ 0.116 & 0.067 $\pm$ 0.013 & 2.826 $\pm$ 0.223 \\
\bottomrule
\end{tabular}%
}
}
\vspace{0.8em}
\textbf{Panel C: $\alpha=0.2$}
\vspace{0.3em}
\centerline{
\resizebox{1.4\textwidth}{!}{%
\begin{tabular}{lcccccccccccc}
\toprule
Model & \multicolumn{2}{c}{Matched} & \multicolumn{2}{c}{Mismatch mean} & \multicolumn{2}{c}{Mismatch worst} & \multicolumn{2}{c}{Pooled} & \multicolumn{2}{c}{MaxThr} & \multicolumn{2}{c}{Latent-robust} \\
\cmidrule(lr){2-3} \cmidrule(lr){4-5} \cmidrule(lr){6-7} \cmidrule(lr){8-9} \cmidrule(lr){10-11} \cmidrule(lr){12-13}
 & Gap & Size & Gap & Size & Gap & Size & Gap & Size & Gap & Size & Gap & Size \\
\midrule
Qwen2.5-1.5B & 0.010 $\pm$ 0.015 & 1.206 $\pm$ 0.029 & -0.028 $\pm$ 0.014 & 1.081 $\pm$ 0.010 & -0.038 $\pm$ 0.011 & 1.046 $\pm$ 0.005 & -0.027 $\pm$ 0.016 & 1.080 $\pm$ 0.020 & 0.010 $\pm$ 0.015 & 1.205 $\pm$ 0.027 & 0.101 $\pm$ 0.014 & 1.739 $\pm$ 0.078 \\
Qwen2.5-3B & 0.014 $\pm$ 0.012 & 1.072 $\pm$ 0.016 & -0.008 $\pm$ 0.007 & 1.010 $\pm$ 0.008 & -0.018 $\pm$ 0.006 & 0.981 $\pm$ 0.003 & -0.009 $\pm$ 0.009 & 1.008 $\pm$ 0.017 & 0.014 $\pm$ 0.012 & 1.070 $\pm$ 0.016 & 0.110 $\pm$ 0.010 & 1.455 $\pm$ 0.030 \\
Qwen2.5-7B & -0.007 $\pm$ 0.012 & 0.961 $\pm$ 0.022 & 0.000 $\pm$ 0.005 & 0.972 $\pm$ 0.006 & -0.017 $\pm$ 0.006 & 0.930 $\pm$ 0.012 & 0.001 $\pm$ 0.008 & 0.975 $\pm$ 0.010 & 0.010 $\pm$ 0.003 & 0.993 $\pm$ 0.004 & 0.103 $\pm$ 0.029 & 1.368 $\pm$ 0.170 \\
DeepSeek-R1-7B & 0.004 $\pm$ 0.017 & 1.874 $\pm$ 0.073 & -0.095 $\pm$ 0.012 & 1.379 $\pm$ 0.047 & -0.132 $\pm$ 0.007 & 1.229 $\pm$ 0.022 & -0.096 $\pm$ 0.010 & 1.352 $\pm$ 0.032 & 0.003 $\pm$ 0.019 & 1.868 $\pm$ 0.076 & 0.081 $\pm$ 0.022 & 2.556 $\pm$ 0.201 \\
Gemma-2-2B & 0.010 $\pm$ 0.019 & 1.497 $\pm$ 0.107 & -0.039 $\pm$ 0.011 & 1.250 $\pm$ 0.045 & -0.066 $\pm$ 0.006 & 1.120 $\pm$ 0.017 & -0.057 $\pm$ 0.010 & 1.155 $\pm$ 0.042 & 0.034 $\pm$ 0.023 & 1.642 $\pm$ 0.128 & 0.127 $\pm$ 0.022 & 2.684 $\pm$ 0.294 \\
Llama-3.2-3B & 0.015 $\pm$ 0.017 & 1.295 $\pm$ 0.035 & -0.015 $\pm$ 0.020 & 1.159 $\pm$ 0.019 & -0.028 $\pm$ 0.020 & 1.115 $\pm$ 0.009 & -0.015 $\pm$ 0.015 & 1.147 $\pm$ 0.007 & 0.015 $\pm$ 0.018 & 1.295 $\pm$ 0.040 & 0.118 $\pm$ 0.029 & 2.027 $\pm$ 0.199 \\
\bottomrule
\end{tabular}%
}
}
\vspace{0.8em}
\textbf{Panel D: $\alpha=0.3$}
\vspace{0.3em}
\centerline{
\resizebox{1.4\textwidth}{!}{%
\begin{tabular}{lcccccccccccc}
\toprule
Model & \multicolumn{2}{c}{Matched} & \multicolumn{2}{c}{Mismatch mean} & \multicolumn{2}{c}{Mismatch worst} & \multicolumn{2}{c}{Pooled} & \multicolumn{2}{c}{MaxThr} & \multicolumn{2}{c}{Latent-robust} \\
\cmidrule(lr){2-3} \cmidrule(lr){4-5} \cmidrule(lr){6-7} \cmidrule(lr){8-9} \cmidrule(lr){10-11} \cmidrule(lr){12-13}
 & Gap & Size & Gap & Size & Gap & Size & Gap & Size & Gap & Size & Gap & Size \\
\midrule
Qwen2.5-1.5B & 0.015 $\pm$ 0.027 & 0.932 $\pm$ 0.027 & 0.032 $\pm$ 0.011 & 0.971 $\pm$ 0.004 & 0.003 $\pm$ 0.017 & 0.904 $\pm$ 0.008 & 0.037 $\pm$ 0.008 & 0.982 $\pm$ 0.009 & 0.046 $\pm$ 0.012 & 1.009 $\pm$ 0.010 & 0.157 $\pm$ 0.011 & 1.443 $\pm$ 0.051 \\
Qwen2.5-3B & -0.001 $\pm$ 0.026 & 0.816 $\pm$ 0.028 & 0.048 $\pm$ 0.013 & 0.909 $\pm$ 0.013 & -0.009 $\pm$ 0.021 & 0.800 $\pm$ 0.021 & 0.057 $\pm$ 0.011 & 0.926 $\pm$ 0.011 & 0.069 $\pm$ 0.010 & 0.952 $\pm$ 0.007 & 0.151 $\pm$ 0.035 & 1.199 $\pm$ 0.103 \\
Qwen2.5-7B & -0.001 $\pm$ 0.022 & 0.794 $\pm$ 0.024 & 0.047 $\pm$ 0.005 & 0.873 $\pm$ 0.005 & -0.078 $\pm$ 0.017 & 0.691 $\pm$ 0.014 & 0.075 $\pm$ 0.006 & 0.909 $\pm$ 0.008 & 0.094 $\pm$ 0.005 & 0.957 $\pm$ 0.006 & 0.138 $\pm$ 0.030 & 1.067 $\pm$ 0.082 \\
DeepSeek-R1-7B & 0.011 $\pm$ 0.012 & 1.384 $\pm$ 0.026 & -0.040 $\pm$ 0.009 & 1.199 $\pm$ 0.018 & -0.055 $\pm$ 0.009 & 1.150 $\pm$ 0.014 & -0.044 $\pm$ 0.006 & 1.182 $\pm$ 0.014 & 0.010 $\pm$ 0.012 & 1.376 $\pm$ 0.032 & 0.133 $\pm$ 0.012 & 2.101 $\pm$ 0.120 \\
Gemma-2-2B & 0.009 $\pm$ 0.009 & 1.016 $\pm$ 0.008 & 0.007 $\pm$ 0.009 & 1.009 $\pm$ 0.005 & -0.018 $\pm$ 0.019 & 0.953 $\pm$ 0.030 & 0.010 $\pm$ 0.012 & 1.018 $\pm$ 0.008 & 0.021 $\pm$ 0.005 & 1.044 $\pm$ 0.009 & 0.165 $\pm$ 0.024 & 1.983 $\pm$ 0.202 \\
Llama-3.2-3B & 0.020 $\pm$ 0.033 & 0.964 $\pm$ 0.041 & 0.032 $\pm$ 0.021 & 1.003 $\pm$ 0.011 & -0.007 $\pm$ 0.033 & 0.911 $\pm$ 0.032 & 0.041 $\pm$ 0.021 & 1.019 $\pm$ 0.024 & 0.056 $\pm$ 0.016 & 1.064 $\pm$ 0.002 & 0.154 $\pm$ 0.039 & 1.516 $\pm$ 0.176 \\
\bottomrule
\end{tabular}%
}
}
\caption{Robustness to the target miscoverage level on CommonsenseQA at deployment $\mathrm{nf4}@T=0.5$ using LAC. Each panel reports coverage gap from target and average prediction-set size as mean $\pm$ standard deviation across runs.}
\label{tab:alpha_robustness_commonsense_qa}
\end{table*}
\begin{table*}[t]
\centering
\small
\textbf{Panel A: $\alpha=0.05$}
\vspace{0.3em}
\centerline{
\resizebox{1.4\textwidth}{!}{%
\begin{tabular}{lcccccccccccc}
\toprule
Model & \multicolumn{2}{c}{Matched} & \multicolumn{2}{c}{Mismatch mean} & \multicolumn{2}{c}{Mismatch worst} & \multicolumn{2}{c}{Pooled} & \multicolumn{2}{c}{MaxThr} & \multicolumn{2}{c}{Latent-robust} \\
\cmidrule(lr){2-3} \cmidrule(lr){4-5} \cmidrule(lr){6-7} \cmidrule(lr){8-9} \cmidrule(lr){10-11} \cmidrule(lr){12-13}
 & Gap & Size & Gap & Size & Gap & Size & Gap & Size & Gap & Size & Gap & Size \\
\midrule
Qwen2.5-1.5B & -0.001 $\pm$ 0.003 & 2.522 $\pm$ 0.121 & -0.030 $\pm$ 0.003 & 1.686 $\pm$ 0.079 & -0.046 $\pm$ 0.003 & 1.282 $\pm$ 0.025 & -0.032 $\pm$ 0.003 & 1.501 $\pm$ 0.048 & 0.006 $\pm$ 0.004 & 2.951 $\pm$ 0.275 & 0.029 $\pm$ 0.002 & 5.784 $\pm$ 0.616 \\
Qwen2.5-3B & -0.001 $\pm$ 0.005 & 2.454 $\pm$ 0.099 & -0.091 $\pm$ 0.004 & 1.439 $\pm$ 0.023 & -0.147 $\pm$ 0.003 & 1.180 $\pm$ 0.009 & -0.077 $\pm$ 0.007 & 1.462 $\pm$ 0.037 & -0.001 $\pm$ 0.005 & 2.452 $\pm$ 0.098 & 0.023 $\pm$ 0.001 & 3.073 $\pm$ 0.049 \\
Qwen2.5-7B & -0.001 $\pm$ 0.000 & 0.965 $\pm$ 0.000 & 0.010 $\pm$ 0.001 & 0.983 $\pm$ 0.001 & -0.008 $\pm$ 0.003 & 0.956 $\pm$ 0.004 & 0.015 $\pm$ 0.000 & 0.991 $\pm$ 0.000 & 0.016 $\pm$ 0.000 & 0.994 $\pm$ 0.000 & 0.022 $\pm$ 0.002 & 1.006 $\pm$ 0.005 \\
DeepSeek-R1-7B & 0.002 $\pm$ 0.001 & 2.781 $\pm$ 0.039 & -0.038 $\pm$ 0.001 & 1.967 $\pm$ 0.006 & -0.067 $\pm$ 0.001 & 1.516 $\pm$ 0.002 & -0.036 $\pm$ 0.002 & 1.857 $\pm$ 0.019 & 0.009 $\pm$ 0.000 & 3.197 $\pm$ 0.018 & 0.030 $\pm$ 0.002 & 5.661 $\pm$ 0.283 \\
Gemma-2-2B & -0.001 $\pm$ 0.003 & 6.736 $\pm$ 0.205 & -0.127 $\pm$ 0.001 & 3.273 $\pm$ 0.066 & -0.197 $\pm$ 0.000 & 1.317 $\pm$ 0.002 & -0.105 $\pm$ 0.002 & 2.257 $\pm$ 0.029 & 0.041 $\pm$ 0.005 & 10.984 $\pm$ 0.683 & 0.046 $\pm$ 0.003 & 12.206 $\pm$ 1.525 \\
Llama-3.2-3B & 0.002 $\pm$ 0.002 & 3.793 $\pm$ 0.161 & -0.072 $\pm$ 0.001 & 1.585 $\pm$ 0.012 & -0.099 $\pm$ 0.003 & 1.321 $\pm$ 0.019 & -0.071 $\pm$ 0.000 & 1.525 $\pm$ 0.003 & 0.002 $\pm$ 0.002 & 3.779 $\pm$ 0.161 & 0.024 $\pm$ 0.004 & 6.973 $\pm$ 0.899 \\
\bottomrule
\end{tabular}%
}
}
\vspace{0.8em}
\textbf{Panel B: $\alpha=0.1$}
\vspace{0.3em}
\centerline{
\resizebox{1.4\textwidth}{!}{%
\begin{tabular}{lcccccccccccc}
\toprule
Model & \multicolumn{2}{c}{Matched} & \multicolumn{2}{c}{Mismatch mean} & \multicolumn{2}{c}{Mismatch worst} & \multicolumn{2}{c}{Pooled} & \multicolumn{2}{c}{MaxThr} & \multicolumn{2}{c}{Latent-robust} \\
\cmidrule(lr){2-3} \cmidrule(lr){4-5} \cmidrule(lr){6-7} \cmidrule(lr){8-9} \cmidrule(lr){10-11} \cmidrule(lr){12-13}
 & Gap & Size & Gap & Size & Gap & Size & Gap & Size & Gap & Size & Gap & Size \\
\midrule
Qwen2.5-1.5B & -0.001 $\pm$ 0.006 & 1.231 $\pm$ 0.042 & -0.016 $\pm$ 0.002 & 1.139 $\pm$ 0.013 & -0.024 $\pm$ 0.002 & 1.098 $\pm$ 0.007 & -0.017 $\pm$ 0.001 & 1.130 $\pm$ 0.004 & 0.001 $\pm$ 0.004 & 1.254 $\pm$ 0.033 & 0.047 $\pm$ 0.008 & 2.505 $\pm$ 0.319 \\
Qwen2.5-3B & 0.000 $\pm$ 0.002 & 1.659 $\pm$ 0.019 & -0.105 $\pm$ 0.003 & 1.168 $\pm$ 0.007 & -0.143 $\pm$ 0.001 & 1.061 $\pm$ 0.003 & -0.097 $\pm$ 0.006 & 1.185 $\pm$ 0.016 & 0.000 $\pm$ 0.002 & 1.658 $\pm$ 0.018 & 0.046 $\pm$ 0.002 & 2.342 $\pm$ 0.034 \\
Qwen2.5-7B & 0.008 $\pm$ 0.000 & 0.916 $\pm$ 0.000 & 0.046 $\pm$ 0.001 & 0.966 $\pm$ 0.000 & -0.019 $\pm$ 0.001 & 0.888 $\pm$ 0.000 & 0.064 $\pm$ 0.000 & 0.988 $\pm$ 0.000 & 0.065 $\pm$ 0.000 & 0.991 $\pm$ 0.000 & 0.066 $\pm$ 0.000 & 0.993 $\pm$ 0.000 \\
DeepSeek-R1-7B & 0.002 $\pm$ 0.004 & 1.701 $\pm$ 0.036 & -0.029 $\pm$ 0.002 & 1.472 $\pm$ 0.012 & -0.062 $\pm$ 0.002 & 1.251 $\pm$ 0.004 & -0.025 $\pm$ 0.001 & 1.452 $\pm$ 0.009 & 0.024 $\pm$ 0.002 & 2.024 $\pm$ 0.025 & 0.056 $\pm$ 0.003 & 3.026 $\pm$ 0.200 \\
Gemma-2-2B & 0.002 $\pm$ 0.004 & 3.485 $\pm$ 0.134 & -0.101 $\pm$ 0.001 & 2.251 $\pm$ 0.046 & -0.156 $\pm$ 0.000 & 1.232 $\pm$ 0.002 & -0.118 $\pm$ 0.000 & 1.611 $\pm$ 0.003 & 0.040 $\pm$ 0.004 & 5.953 $\pm$ 0.237 & 0.087 $\pm$ 0.005 & 10.787 $\pm$ 0.672 \\
Llama-3.2-3B & -0.001 $\pm$ 0.002 & 1.750 $\pm$ 0.024 & -0.071 $\pm$ 0.001 & 1.216 $\pm$ 0.005 & -0.085 $\pm$ 0.000 & 1.152 $\pm$ 0.002 & -0.067 $\pm$ 0.000 & 1.226 $\pm$ 0.002 & -0.002 $\pm$ 0.002 & 1.748 $\pm$ 0.026 & 0.047 $\pm$ 0.004 & 3.576 $\pm$ 0.266 \\
\bottomrule
\end{tabular}%
}
}
\vspace{0.8em}
\textbf{Panel C: $\alpha=0.2$}
\vspace{0.3em}
\centerline{
\resizebox{1.4\textwidth}{!}{%
\begin{tabular}{lcccccccccccc}
\toprule
Model & \multicolumn{2}{c}{Matched} & \multicolumn{2}{c}{Mismatch mean} & \multicolumn{2}{c}{Mismatch worst} & \multicolumn{2}{c}{Pooled} & \multicolumn{2}{c}{MaxThr} & \multicolumn{2}{c}{Latent-robust} \\
\cmidrule(lr){2-3} \cmidrule(lr){4-5} \cmidrule(lr){6-7} \cmidrule(lr){8-9} \cmidrule(lr){10-11} \cmidrule(lr){12-13}
 & Gap & Size & Gap & Size & Gap & Size & Gap & Size & Gap & Size & Gap & Size \\
\midrule
Qwen2.5-1.5B & -0.001 $\pm$ 0.005 & 0.906 $\pm$ 0.009 & 0.029 $\pm$ 0.004 & 0.973 $\pm$ 0.006 & -0.045 $\pm$ 0.007 & 0.840 $\pm$ 0.010 & 0.050 $\pm$ 0.002 & 1.012 $\pm$ 0.003 & 0.069 $\pm$ 0.001 & 1.072 $\pm$ 0.001 & 0.087 $\pm$ 0.004 & 1.152 $\pm$ 0.018 \\
Qwen2.5-3B & 0.001 $\pm$ 0.002 & 1.175 $\pm$ 0.003 & -0.078 $\pm$ 0.002 & 0.966 $\pm$ 0.003 & -0.157 $\pm$ 0.001 & 0.744 $\pm$ 0.001 & -0.061 $\pm$ 0.005 & 1.013 $\pm$ 0.011 & 0.001 $\pm$ 0.002 & 1.175 $\pm$ 0.003 & 0.088 $\pm$ 0.002 & 1.534 $\pm$ 0.016 \\
Qwen2.5-7B & 0.081 $\pm$ 0.001 & 0.888 $\pm$ 0.000 & 0.145 $\pm$ 0.001 & 0.964 $\pm$ 0.000 & 0.081 $\pm$ 0.001 & 0.888 $\pm$ 0.000 & 0.162 $\pm$ 0.001 & 0.986 $\pm$ 0.001 & 0.164 $\pm$ 0.000 & 0.989 $\pm$ 0.000 & 0.165 $\pm$ 0.000 & 0.991 $\pm$ 0.000 \\
DeepSeek-R1-7B & 0.002 $\pm$ 0.005 & 1.123 $\pm$ 0.012 & 0.021 $\pm$ 0.003 & 1.198 $\pm$ 0.008 & -0.010 $\pm$ 0.002 & 1.090 $\pm$ 0.003 & 0.020 $\pm$ 0.003 & 1.184 $\pm$ 0.004 & 0.072 $\pm$ 0.004 & 1.428 $\pm$ 0.022 & 0.118 $\pm$ 0.002 & 1.895 $\pm$ 0.033 \\
Gemma-2-2B & 0.000 $\pm$ 0.005 & 1.770 $\pm$ 0.049 & -0.041 $\pm$ 0.002 & 1.404 $\pm$ 0.022 & -0.075 $\pm$ 0.001 & 1.111 $\pm$ 0.001 & -0.058 $\pm$ 0.002 & 1.217 $\pm$ 0.010 & 0.050 $\pm$ 0.006 & 2.338 $\pm$ 0.091 & 0.127 $\pm$ 0.006 & 5.339 $\pm$ 0.355 \\
Llama-3.2-3B & -0.002 $\pm$ 0.002 & 1.098 $\pm$ 0.004 & -0.039 $\pm$ 0.002 & 1.012 $\pm$ 0.002 & -0.112 $\pm$ 0.002 & 0.867 $\pm$ 0.004 & -0.020 $\pm$ 0.000 & 1.047 $\pm$ 0.000 & 0.005 $\pm$ 0.000 & 1.119 $\pm$ 0.000 & 0.092 $\pm$ 0.007 & 1.672 $\pm$ 0.081 \\
\bottomrule
\end{tabular}%
}
}
\vspace{0.8em}
\textbf{Panel D: $\alpha=0.3$}
\vspace{0.3em}
\centerline{
\resizebox{1.4\textwidth}{!}{%
\begin{tabular}{lcccccccccccc}
\toprule
Model & \multicolumn{2}{c}{Matched} & \multicolumn{2}{c}{Mismatch mean} & \multicolumn{2}{c}{Mismatch worst} & \multicolumn{2}{c}{Pooled} & \multicolumn{2}{c}{MaxThr} & \multicolumn{2}{c}{Latent-robust} \\
\cmidrule(lr){2-3} \cmidrule(lr){4-5} \cmidrule(lr){6-7} \cmidrule(lr){8-9} \cmidrule(lr){10-11} \cmidrule(lr){12-13}
 & Gap & Size & Gap & Size & Gap & Size & Gap & Size & Gap & Size & Gap & Size \\
\midrule
Qwen2.5-1.5B & -0.001 $\pm$ 0.005 & 0.765 $\pm$ 0.006 & 0.080 $\pm$ 0.003 & 0.894 $\pm$ 0.004 & -0.096 $\pm$ 0.006 & 0.652 $\pm$ 0.007 & 0.120 $\pm$ 0.001 & 0.947 $\pm$ 0.002 & 0.161 $\pm$ 0.000 & 1.043 $\pm$ 0.001 & 0.156 $\pm$ 0.004 & 1.033 $\pm$ 0.009 \\
Qwen2.5-3B & 0.001 $\pm$ 0.001 & 0.911 $\pm$ 0.003 & -0.020 $\pm$ 0.002 & 0.861 $\pm$ 0.003 & -0.138 $\pm$ 0.004 & 0.594 $\pm$ 0.004 & -0.003 $\pm$ 0.002 & 0.901 $\pm$ 0.003 & 0.048 $\pm$ 0.004 & 1.037 $\pm$ 0.009 & 0.141 $\pm$ 0.003 & 1.291 $\pm$ 0.009 \\
Qwen2.5-7B & 0.181 $\pm$ 0.001 & 0.888 $\pm$ 0.000 & 0.244 $\pm$ 0.001 & 0.963 $\pm$ 0.000 & 0.181 $\pm$ 0.001 & 0.888 $\pm$ 0.000 & 0.261 $\pm$ 0.001 & 0.982 $\pm$ 0.001 & 0.264 $\pm$ 0.000 & 0.989 $\pm$ 0.000 & 0.264 $\pm$ 0.000 & 0.990 $\pm$ 0.000 \\
DeepSeek-R1-7B & 0.001 $\pm$ 0.009 & 0.894 $\pm$ 0.013 & 0.082 $\pm$ 0.003 & 1.073 $\pm$ 0.003 & 0.036 $\pm$ 0.003 & 0.964 $\pm$ 0.003 & 0.084 $\pm$ 0.007 & 1.076 $\pm$ 0.015 & 0.110 $\pm$ 0.002 & 1.150 $\pm$ 0.003 & 0.182 $\pm$ 0.001 & 1.510 $\pm$ 0.003 \\
Gemma-2-2B & 0.002 $\pm$ 0.005 & 1.028 $\pm$ 0.014 & 0.003 $\pm$ 0.006 & 1.029 $\pm$ 0.016 & -0.007 $\pm$ 0.012 & 1.002 $\pm$ 0.029 & 0.003 $\pm$ 0.004 & 1.030 $\pm$ 0.012 & 0.015 $\pm$ 0.012 & 1.072 $\pm$ 0.040 & 0.185 $\pm$ 0.006 & 2.941 $\pm$ 0.154 \\
Llama-3.2-3B & -0.005 $\pm$ 0.003 & 0.878 $\pm$ 0.005 & 0.006 $\pm$ 0.001 & 0.914 $\pm$ 0.002 & -0.157 $\pm$ 0.002 & 0.664 $\pm$ 0.004 & 0.050 $\pm$ 0.001 & 0.979 $\pm$ 0.001 & 0.092 $\pm$ 0.000 & 1.079 $\pm$ 0.001 & 0.137 $\pm$ 0.012 & 1.249 $\pm$ 0.062 \\
\bottomrule
\end{tabular}%
}
}
\caption{Robustness to the target miscoverage level on DBPedia-14 at deployment $\mathrm{nf4}@T=0.5$ using LAC. Each panel reports coverage gap from target and average prediction-set size as mean $\pm$ standard deviation across runs.}
\label{tab:alpha_robustness_dbpedia_14}
\end{table*}
\begin{table*}[t]
\centering
\small
\textbf{Panel A: $\alpha=0.05$}
\vspace{0.3em}
\centerline{
\resizebox{1.4\textwidth}{!}{%
\begin{tabular}{lcccccccccccc}
\toprule
Model & \multicolumn{2}{c}{Matched} & \multicolumn{2}{c}{Mismatch mean} & \multicolumn{2}{c}{Mismatch worst} & \multicolumn{2}{c}{Pooled} & \multicolumn{2}{c}{MaxThr} & \multicolumn{2}{c}{Latent-robust} \\
\cmidrule(lr){2-3} \cmidrule(lr){4-5} \cmidrule(lr){6-7} \cmidrule(lr){8-9} \cmidrule(lr){10-11} \cmidrule(lr){12-13}
 & Gap & Size & Gap & Size & Gap & Size & Gap & Size & Gap & Size & Gap & Size \\
\midrule
Qwen2.5-1.5B & 0.001 $\pm$ 0.004 & 8.495 $\pm$ 0.039 & -0.309 $\pm$ 0.004 & 4.789 $\pm$ 0.010 & -0.453 $\pm$ 0.004 & 3.157 $\pm$ 0.007 & -0.137 $\pm$ 0.027 & 6.657 $\pm$ 0.273 & 0.002 $\pm$ 0.003 & 8.501 $\pm$ 0.011 & 0.030 $\pm$ 0.003 & 8.978 $\pm$ 0.043 \\
Qwen2.5-3B & 0.050 $\pm$ 0.000 & 9.465 $\pm$ 0.016 & -0.269 $\pm$ 0.003 & 4.591 $\pm$ 0.039 & -0.434 $\pm$ 0.003 & 2.573 $\pm$ 0.030 & -0.143 $\pm$ 0.002 & 5.660 $\pm$ 0.039 & 0.050 $\pm$ 0.000 & 9.465 $\pm$ 0.016 & 0.050 $\pm$ 0.000 & 9.465 $\pm$ 0.016 \\
Qwen2.5-7B & 0.050 $\pm$ 0.000 & 9.465 $\pm$ 0.016 & -0.210 $\pm$ 0.001 & 4.496 $\pm$ 0.013 & -0.346 $\pm$ 0.002 & 2.605 $\pm$ 0.008 & -0.136 $\pm$ 0.001 & 4.821 $\pm$ 0.027 & 0.050 $\pm$ 0.000 & 9.465 $\pm$ 0.016 & 0.050 $\pm$ 0.000 & 9.465 $\pm$ 0.016 \\
DeepSeek-R1-7B & 0.005 $\pm$ 0.008 & 8.440 $\pm$ 0.091 & -0.326 $\pm$ 0.007 & 4.397 $\pm$ 0.056 & -0.484 $\pm$ 0.006 & 2.699 $\pm$ 0.031 & -0.141 $\pm$ 0.021 & 6.327 $\pm$ 0.277 & 0.009 $\pm$ 0.007 & 8.500 $\pm$ 0.081 & 0.034 $\pm$ 0.004 & 8.982 $\pm$ 0.070 \\
Gemma-2-2B & -0.003 $\pm$ 0.004 & 8.215 $\pm$ 0.055 & -0.231 $\pm$ 0.005 & 5.362 $\pm$ 0.063 & -0.404 $\pm$ 0.001 & 3.380 $\pm$ 0.014 & -0.089 $\pm$ 0.015 & 6.948 $\pm$ 0.186 & 0.019 $\pm$ 0.004 & 8.610 $\pm$ 0.060 & 0.037 $\pm$ 0.002 & 8.984 $\pm$ 0.033 \\
Llama-3.2-3B & 0.002 $\pm$ 0.001 & 8.069 $\pm$ 0.012 & -0.236 $\pm$ 0.004 & 4.976 $\pm$ 0.044 & -0.378 $\pm$ 0.007 & 3.436 $\pm$ 0.063 & -0.141 $\pm$ 0.007 & 5.951 $\pm$ 0.119 & 0.002 $\pm$ 0.001 & 8.060 $\pm$ 0.013 & 0.027 $\pm$ 0.002 & 8.583 $\pm$ 0.060 \\
\bottomrule
\end{tabular}%
}
}
\vspace{0.8em}
\textbf{Panel B: $\alpha=0.1$}
\vspace{0.3em}
\centerline{
\resizebox{1.4\textwidth}{!}{%
\begin{tabular}{lcccccccccccc}
\toprule
Model & \multicolumn{2}{c}{Matched} & \multicolumn{2}{c}{Mismatch mean} & \multicolumn{2}{c}{Mismatch worst} & \multicolumn{2}{c}{Pooled} & \multicolumn{2}{c}{MaxThr} & \multicolumn{2}{c}{Latent-robust} \\
\cmidrule(lr){2-3} \cmidrule(lr){4-5} \cmidrule(lr){6-7} \cmidrule(lr){8-9} \cmidrule(lr){10-11} \cmidrule(lr){12-13}
 & Gap & Size & Gap & Size & Gap & Size & Gap & Size & Gap & Size & Gap & Size \\
\midrule
Qwen2.5-1.5B & 0.001 $\pm$ 0.009 & 7.765 $\pm$ 0.072 & -0.303 $\pm$ 0.005 & 4.293 $\pm$ 0.019 & -0.432 $\pm$ 0.003 & 2.876 $\pm$ 0.004 & -0.250 $\pm$ 0.025 & 4.829 $\pm$ 0.236 & 0.008 $\pm$ 0.007 & 7.860 $\pm$ 0.043 & 0.054 $\pm$ 0.003 & 8.533 $\pm$ 0.022 \\
Qwen2.5-3B & -0.001 $\pm$ 0.011 & 6.963 $\pm$ 0.139 & -0.272 $\pm$ 0.009 & 3.870 $\pm$ 0.175 & -0.420 $\pm$ 0.004 & 2.261 $\pm$ 0.033 & -0.230 $\pm$ 0.004 & 4.095 $\pm$ 0.063 & 0.067 $\pm$ 0.056 & 8.655 $\pm$ 1.406 & 0.100 $\pm$ 0.000 & 9.464 $\pm$ 0.017 \\
Qwen2.5-7B & -0.001 $\pm$ 0.007 & 6.167 $\pm$ 0.092 & -0.222 $\pm$ 0.002 & 3.454 $\pm$ 0.021 & -0.337 $\pm$ 0.002 & 2.286 $\pm$ 0.005 & -0.194 $\pm$ 0.009 & 3.586 $\pm$ 0.082 & -0.002 $\pm$ 0.008 & 6.148 $\pm$ 0.092 & 0.100 $\pm$ 0.001 & 9.461 $\pm$ 0.018 \\
DeepSeek-R1-7B & 0.004 $\pm$ 0.006 & 7.643 $\pm$ 0.050 & -0.317 $\pm$ 0.007 & 3.949 $\pm$ 0.056 & -0.457 $\pm$ 0.005 & 2.508 $\pm$ 0.032 & -0.249 $\pm$ 0.017 & 4.462 $\pm$ 0.165 & 0.008 $\pm$ 0.009 & 7.719 $\pm$ 0.099 & 0.058 $\pm$ 0.007 & 8.491 $\pm$ 0.095 \\
Gemma-2-2B & -0.001 $\pm$ 0.003 & 7.495 $\pm$ 0.062 & -0.237 $\pm$ 0.002 & 4.717 $\pm$ 0.028 & -0.393 $\pm$ 0.003 & 3.037 $\pm$ 0.020 & -0.173 $\pm$ 0.001 & 5.293 $\pm$ 0.020 & 0.035 $\pm$ 0.005 & 8.028 $\pm$ 0.074 & 0.072 $\pm$ 0.001 & 8.656 $\pm$ 0.033 \\
Llama-3.2-3B & 0.002 $\pm$ 0.003 & 7.239 $\pm$ 0.037 & -0.256 $\pm$ 0.004 & 4.213 $\pm$ 0.048 & -0.375 $\pm$ 0.006 & 2.997 $\pm$ 0.050 & -0.206 $\pm$ 0.009 & 4.629 $\pm$ 0.100 & 0.002 $\pm$ 0.003 & 7.232 $\pm$ 0.043 & 0.057 $\pm$ 0.004 & 8.152 $\pm$ 0.075 \\
\bottomrule
\end{tabular}%
}
}
\vspace{0.8em}
\textbf{Panel C: $\alpha=0.2$}
\vspace{0.3em}
\centerline{
\resizebox{1.4\textwidth}{!}{%
\begin{tabular}{lcccccccccccc}
\toprule
Model & \multicolumn{2}{c}{Matched} & \multicolumn{2}{c}{Mismatch mean} & \multicolumn{2}{c}{Mismatch worst} & \multicolumn{2}{c}{Pooled} & \multicolumn{2}{c}{MaxThr} & \multicolumn{2}{c}{Latent-robust} \\
\cmidrule(lr){2-3} \cmidrule(lr){4-5} \cmidrule(lr){6-7} \cmidrule(lr){8-9} \cmidrule(lr){10-11} \cmidrule(lr){12-13}
 & Gap & Size & Gap & Size & Gap & Size & Gap & Size & Gap & Size & Gap & Size \\
\midrule
Qwen2.5-1.5B & -0.004 $\pm$ 0.008 & 6.465 $\pm$ 0.062 & -0.259 $\pm$ 0.006 & 3.667 $\pm$ 0.031 & -0.364 $\pm$ 0.003 & 2.561 $\pm$ 0.005 & -0.276 $\pm$ 0.013 & 3.445 $\pm$ 0.078 & 0.003 $\pm$ 0.014 & 6.539 $\pm$ 0.108 & 0.108 $\pm$ 0.006 & 7.845 $\pm$ 0.043 \\
Qwen2.5-3B & -0.001 $\pm$ 0.008 & 5.560 $\pm$ 0.104 & -0.245 $\pm$ 0.005 & 3.015 $\pm$ 0.048 & -0.355 $\pm$ 0.004 & 1.945 $\pm$ 0.033 & -0.255 $\pm$ 0.007 & 2.855 $\pm$ 0.034 & 0.009 $\pm$ 0.005 & 5.692 $\pm$ 0.067 & 0.166 $\pm$ 0.058 & 8.630 $\pm$ 1.427 \\
Qwen2.5-7B & -0.006 $\pm$ 0.008 & 4.585 $\pm$ 0.085 & -0.195 $\pm$ 0.003 & 2.695 $\pm$ 0.019 & -0.282 $\pm$ 0.002 & 1.918 $\pm$ 0.005 & -0.202 $\pm$ 0.009 & 2.564 $\pm$ 0.051 & -0.004 $\pm$ 0.005 & 4.602 $\pm$ 0.057 & 0.091 $\pm$ 0.007 & 6.048 $\pm$ 0.082 \\
DeepSeek-R1-7B & 0.013 $\pm$ 0.004 & 6.361 $\pm$ 0.009 & -0.276 $\pm$ 0.005 & 3.328 $\pm$ 0.031 & -0.387 $\pm$ 0.004 & 2.261 $\pm$ 0.023 & -0.293 $\pm$ 0.006 & 3.069 $\pm$ 0.046 & 0.012 $\pm$ 0.004 & 6.355 $\pm$ 0.012 & 0.105 $\pm$ 0.013 & 7.658 $\pm$ 0.164 \\
Gemma-2-2B & 0.002 $\pm$ 0.004 & 6.179 $\pm$ 0.044 & -0.211 $\pm$ 0.002 & 3.906 $\pm$ 0.018 & -0.327 $\pm$ 0.002 & 2.691 $\pm$ 0.015 & -0.222 $\pm$ 0.006 & 3.707 $\pm$ 0.048 & 0.052 $\pm$ 0.004 & 6.828 $\pm$ 0.056 & 0.132 $\pm$ 0.008 & 7.997 $\pm$ 0.116 \\
Llama-3.2-3B & 0.010 $\pm$ 0.004 & 5.963 $\pm$ 0.050 & -0.235 $\pm$ 0.003 & 3.397 $\pm$ 0.045 & -0.320 $\pm$ 0.003 & 2.590 $\pm$ 0.021 & -0.238 $\pm$ 0.009 & 3.338 $\pm$ 0.082 & 0.009 $\pm$ 0.003 & 5.953 $\pm$ 0.045 & 0.102 $\pm$ 0.005 & 7.232 $\pm$ 0.082 \\
\bottomrule
\end{tabular}%
%
}
}
\caption{Robustness to the target miscoverage level on MMLU-Pro at deployment $\mathrm{nf4}@T=0.5$ using LAC. Each panel reports coverage gap from target and average prediction-set size as mean $\pm$ standard deviation across runs.}
\label{tab:alpha_robustness_mmlu_pro}
\end{table*}
\begin{table*}[t]
\centering
\small
\textbf{Panel A: $\alpha=0.05$}
\vspace{0.3em}
\centerline{
\resizebox{1.4\textwidth}{!}{%
%
%
}
}
\caption{Robustness to the target miscoverage level on MobileActions at deployment $\mathrm{nf4}@T=0.5$ using LAC. Each panel reports coverage gap from target and average prediction-set size as mean $\pm$ standard deviation across runs.}
\label{tab:alpha_robustness_mobile_actions}
\end{table*}

Tables~\ref{tab:alpha_robustness_commonsense_qa}--\ref{tab:alpha_robustness_mobile_actions} report the coverage gap from target, defined as empirical coverage minus $(1-\alpha)$, together with average prediction-set size, for $\alpha\in\{0.05,0.1,0.2,0.3\}$ in the deployment setting $\mathrm{nf4}@T=0.5$ using LAC. The main qualitative ranking of methods remains stable across target levels. Across datasets and models, \textsc{Matched} tends to remain near target, mismatch baselines can substantially under-cover, and the more conservative robust methods recover coverage more reliably. This indicates that the central phenomenon is not tied to a single nominal target level.

The size columns in Tables~\ref{tab:alpha_robustness_commonsense_qa}--\ref{tab:alpha_robustness_mobile_actions} are also informative: as the target level changes, the efficiency cost of robust recovery changes in tandem. This again reinforces the point that deployment robustness should be evaluated jointly in terms of reliability and prediction-set size, not by coverage alone.

\section{\textsc{Latent-Robust} Diagnostics}
\label{app:latent_diag}

\begin{table*}[t]
\centering
\small
\resizebox{\textwidth}{!}{%
\begin{tabular}{lccccc}
\toprule
Model & Coverage & Set size & $\Gamma$ coverage & Avg. $|\Gamma|$ \\
\midrule
Qwen2.5-1.5B & 0.962 $\pm$ 0.023 & 2.652 $\pm$ 0.402 & 0.896 $\pm$ 0.062 & 4.582 $\pm$ 0.605  \\
Qwen2.5-3B & 0.998 $\pm$ 0.002 & 4.957 $\pm$ 0.012 & 0.945 $\pm$ 0.046 & 5.120 $\pm$ 1.795 \\
Qwen2.5-7B & 0.954 $\pm$ 0.015 & 2.007 $\pm$ 0.198 & 0.946 $\pm$ 0.027 & 3.762 $\pm$ 0.635 \\
DeepSeek-R1-7B & 0.952 $\pm$ 0.013 & 3.404 $\pm$ 0.198 & 0.917 $\pm$ 0.018 & 5.082 $\pm$ 0.376 \\
Gemma-2-2B & 0.973 $\pm$ 0.008 & 3.573 $\pm$ 0.139 & 0.878 $\pm$ 0.038 & 4.728 $\pm$ 0.135 \\
Llama-3.2-3B & 0.967 $\pm$ 0.013 & 2.826 $\pm$ 0.223 & 0.903 $\pm$ 0.067 & 6.965 $\pm$ 0.382 \\
Phi-3-mini & 0.957 $\pm$ 0.016 & 2.533 $\pm$ 0.273 & 0.952 $\pm$ 0.027 & 5.237 $\pm$ 0.884 \\
Mistral-7B & 0.970 $\pm$ 0.007 & 3.429 $\pm$ 0.047 & 0.944 $\pm$ 0.043 & 4.892 $\pm$ 0.174 \\
\bottomrule
\end{tabular}%
}
\caption{Diagnostics for \textsc{Latent-Robust} on CommonsenseQA in the main-paper setting: LAC, $\alpha=0.1$, deployment $\mathrm{nf4}@T=0.5$, $\alpha_{\mathrm{stage1}}=\beta_{\mathrm{stage2}}=0.05$, and $d_1=0.5$. Each entry reports mean $\pm$ standard deviation across runs.}
\label{tab:latent_diag_commonsense_qa}
\end{table*}
\begin{table*}[t]
\centering
\small
\resizebox{\textwidth}{!}{%
\begin{tabular}{lccccc}
\toprule
Model & Coverage & Set size & $\Gamma$ coverage & Avg. $|\Gamma|$ \\
\midrule
Qwen2.5-1.5B & 0.947 $\pm$ 0.008 & 2.505 $\pm$ 0.319 & 0.967 $\pm$ 0.011 & 5.595 $\pm$ 0.595 \\
Qwen2.5-3B & 0.946 $\pm$ 0.002 & 2.342 $\pm$ 0.034 & 0.975 $\pm$ 0.004 & 6.142 $\pm$ 0.183 \\
Qwen2.5-7B & 0.966 $\pm$ 0.000 & 0.993 $\pm$ 0.000 & 0.955 $\pm$ 0.008 & 4.703 $\pm$ 0.548 \\
DeepSeek-R1-7B & 0.956 $\pm$ 0.003 & 3.026 $\pm$ 0.200 & 0.993 $\pm$ 0.002 & 3.032 $\pm$ 0.022 \\
Gemma-2-2B & 0.987 $\pm$ 0.005 & 10.787 $\pm$ 0.672 & 0.972 $\pm$ 0.001 & 3.042 $\pm$ 0.030 \\
Llama-3.2-3B & 0.947 $\pm$ 0.004 & 3.576 $\pm$ 0.266 & 0.984 $\pm$ 0.000 & 3.076 $\pm$ 0.028 \\
Phi-3-mini & 0.937 $\pm$ 0.004 & 1.060 $\pm$ 0.013 & 0.954 $\pm$ 0.009 & 11.356 $\pm$ 0.049 \\
Mistral-7B & 0.961 $\pm$ 0.004 & 2.088 $\pm$ 0.186 & 0.987 $\pm$ 0.007 & 2.989 $\pm$ 0.014 \\
\bottomrule
\end{tabular}%
}
\caption{Diagnostics for \textsc{Latent-Robust} on DBPedia-14 in the main-paper setting: LAC, $\alpha=0.1$, deployment $\mathrm{nf4}@T=0.5$, $\alpha_{\mathrm{stage1}}=\beta_{\mathrm{stage2}}=0.05$, and $d_1=0.5$. Each entry reports mean $\pm$ standard deviation across runs.}
\label{tab:latent_diag_dbpedia_14}
\end{table*}
\begin{table*}[t]
\centering
\small
\resizebox{\textwidth}{!}{%
\begin{tabular}{lccccc}
\toprule
Model & Coverage & Set size & $\Gamma$ coverage & Avg. $|\Gamma|$ \\
\midrule
Qwen2.5-1.5B & 0.954 $\pm$ 0.003 & 8.533 $\pm$ 0.022 & 0.963 $\pm$ 0.007 & 6.411 $\pm$ 0.306 \\
Qwen2.5-3B & 1.000 $\pm$ 0.000 & 9.464 $\pm$ 0.017 & 0.926 $\pm$ 0.013 & 6.281 $\pm$ 0.342 \\
Qwen2.5-7B & 1.000 $\pm$ 0.001 & 9.461 $\pm$ 0.018 & 0.942 $\pm$ 0.008 & 6.511 $\pm$ 0.041 \\
DeepSeek-R1-7B & 0.958 $\pm$ 0.007 & 8.491 $\pm$ 0.095 & 0.969 $\pm$ 0.003 & 5.068 $\pm$ 0.033 \\
Gemma-2-2B & 0.972 $\pm$ 0.001 & 8.656 $\pm$ 0.033 & 0.953 $\pm$ 0.026 & 8.149 $\pm$ 0.143 \\
Llama-3.2-3B & 0.957 $\pm$ 0.004 & 8.152 $\pm$ 0.075 & 0.950 $\pm$ 0.028 & 8.069 $\pm$ 0.190 \\
Phi-3-mini & 0.957 $\pm$ 0.004 & 7.425 $\pm$ 0.112 & 0.939 $\pm$ 0.041 & 6.379 $\pm$ 0.133 \\
Mistral-7B & 0.968 $\pm$ 0.001 & 8.425 $\pm$ 0.061 & 0.911 $\pm$ 0.031 & 5.966 $\pm$ 0.098 \\
\bottomrule
\end{tabular}%
}
\caption{Diagnostics for \textsc{Latent-Robust} on MMLU-Pro in the main-paper setting: LAC, $\alpha=0.1$, deployment $\mathrm{nf4}@T=0.5$, $\alpha_{\mathrm{stage1}}=\beta_{\mathrm{stage2}}=0.05$, and $d_1=0.5$. Each entry reports mean $\pm$ standard deviation across runs.}
\label{tab:latent_diag_mmlu_pro}
\end{table*}
\begin{table*}[t]
\centering
\small
\resizebox{\textwidth}{!}{%
\begin{tabular}{lccccc}
\toprule
Model & Coverage & Set size & $\Gamma$ coverage & Avg. $|\Gamma|$ \\
\midrule
Qwen2.5-1.5B & 0.940 $\pm$ 0.010 & 1.049 $\pm$ 0.029 & 0.947 $\pm$ 0.008 & 11.652 $\pm$ 0.050 \\
Qwen2.5-3B & 0.961 $\pm$ 0.001 & 1.265 $\pm$ 0.008 & 0.969 $\pm$ 0.004 & 6.037 $\pm$ 0.033 \\
Qwen2.5-7B & 0.955 $\pm$ 0.001 & 0.998 $\pm$ 0.000 & 0.975 $\pm$ 0.004 & 11.808 $\pm$ 0.019 \\
DeepSeek-R1-7B & 0.946 $\pm$ 0.006 & 1.277 $\pm$ 0.035 & 0.952 $\pm$ 0.011 & 8.161 $\pm$ 1.169 \\
Gemma-2-2B & 0.942 $\pm$ 0.003 & 1.674 $\pm$ 0.014 & 0.968 $\pm$ 0.004 & 3.115 $\pm$ 0.012 \\
Llama-3.2-3B & 0.943 $\pm$ 0.007 & 1.451 $\pm$ 0.084 & 0.978 $\pm$ 0.008 & 4.453 $\pm$ 0.624 \\
Phi-3-mini & 0.940 $\pm$ 0.007 & 1.086 $\pm$ 0.015 & 0.940 $\pm$ 0.020 & 11.527 $\pm$ 0.096 \\
Mistral-7B & 0.951 $\pm$ 0.005 & 1.166 $\pm$ 0.016 & 0.946 $\pm$ 0.031 & 5.259 $\pm$ 0.441 \\
\bottomrule
\end{tabular}%
}
\caption{Diagnostics for \textsc{Latent-Robust} on MobileActions in the main-paper setting: LAC, $\alpha=0.1$, deployment $\mathrm{nf4}@T=0.5$, $\alpha_{\mathrm{stage1}}=\beta_{\mathrm{stage2}}=0.05$, and $d_1=0.5$. Each entry reports mean $\pm$ standard deviation across runs.}
\label{tab:latent_diag_mobile_actions}
\end{table*}

Tables~\ref{tab:latent_diag_commonsense_qa}--\ref{tab:latent_diag_mobile_actions} clarify how \textsc{Latent-Robust} achieves coverage recovery. Besides final coverage and average prediction-set size, they report the empirical coverage of the latent configuration-set predictor $\widehat{\Gamma}(Z(X))$ and the average size of the predicted configuration set. Interpreting these diagnostics jointly is important. When the inferred configuration set is large, the method hedges more broadly across plausible latent deployment states and correspondingly pays a larger efficiency cost. Conversely, smaller inferred configuration sets allow a less conservative threshold.

The diagnostics in Tables~\ref{tab:latent_diag_commonsense_qa}--\ref{tab:latent_diag_mobile_actions} therefore show that the method's conservativeness is not arbitrary. It is mediated by how difficult it is to infer the latent deployment configuration from the observed predictive vector. Generally, the adaptive hedge is driven primarily by latent-configuration predictions rather than repeated reversion to the global worst-case threshold. The attached standard deviations also make clear that these diagnostic quantities are themselves reasonably stable across runs.

\clearpage
\newpage
\section*{NeurIPS Paper Checklist}

\begin{enumerate}

\item {\bf Claims}
    \item[] Question: Do the main claims made in the abstract and introduction accurately reflect the paper's contributions and scope?
    \item[] Answer: \answerYes{}.
    \item[] Justification: The abstract and introduction state the main contribution: identifying inference-time model shift as a source of non-exchangeability for conformal prediction in language models and evaluating its empirical consequences across models, datasets, and inference configurations. The scope is also clearly limited to finite-label benchmarks and post-training inference-configuration changes.

\item {\bf Limitations}
    \item[] Question: Does the paper discuss the limitations of the work performed by the authors?
    \item[] Answer: \answerYes{}.
    \item[] Justification: Section~6 discusses limitations, including that the work is primarily a characterization rather than a mechanistic theory, that the empirical study uses a finite deployment grid, and that it does not cover richer deployment variation such as continuous temperatures or temporal drift in serving policies.

\item {\bf Theory assumptions and proofs}
    \item[] Question: For each theoretical result, does the paper provide the full set of assumptions and a complete (and correct) proof?
    \item[] Answer: \answerYes{}.
    \item[] Justification: The main text states the assumptions and targets for inference-time model shift, MAXTHR, POOLED, and LATENT-ROBUST. Appendix~A provides the corresponding definitions, lemmas, full statements, and proofs.

\item {\bf Experimental result reproducibility}
    \item[] Question: Does the paper fully disclose all the information needed to reproduce the main experimental results of the paper to the extent that it affects the main claims and/or conclusions of the paper (regardless of whether the code and data are provided or not)?
    \item[] Answer: \answerYes{}.
    \item[] Justification: Section~5 and Appendix~C specify the datasets, models, inference-configuration grid, calibration/evaluation splits, score functions, robust baselines, quantization and temperature implementation, and the Stage-2 LATENT-ROBUST router construction. The supplementary material includes the source code, configuration files, and processed experimental artifacts sufficient to reproduce the main reported tables and figures.

\item {\bf Open access to data and code}
    \item[] Question: Does the paper provide open access to the data and code, with sufficient instructions to faithfully reproduce the main experimental results, as described in supplemental material?
    \item[] Answer: \answerYes{}.
    \item[] Justification: The supplementary material includes the source code and data package with scripts, configuration files, processed experimental artifacts, and instructions for reproducing the main coverage and prediction-set-size results. The package documents how to access the public datasets and model checkpoints used in the experiments.

\item {\bf Experimental setting/details}
    \item[] Question: Does the paper specify all the training and test details (e.g., data splits, hyperparameters, how they were chosen, type of optimizer) necessary to understand the results?
    \item[] Answer: \answerYes{}.
    \item[] Justification: Section~5 specifies the datasets, models, inference configurations, conformal score setup, data splits, compared methods, and metrics. Appendix~C further specifies split sizes, model versions, score functions, deployment grid, latent split settings, the logistic-regression router, quantization/temperature implementation, and compute resources.

\item {\bf Experiment statistical significance}
    \item[] Question: Does the paper report error bars suitably and correctly defined or other appropriate information about the statistical significance of the experiments?
    \item[] Answer: \answerYes{}.
    \item[] Justification: The experiments are averaged over three independent random splits, and the appendix tables report mean $\pm$ standard deviation across runs for coverage and average prediction-set size. The paper also states how aggregation over deployment axes is performed before summarizing across runs.

\item {\bf Experiments compute resources}
    \item[] Question: For each experiment, does the paper provide sufficient information on the computer resources (type of compute workers, memory, time of execution) needed to reproduce the experiments?
    \item[] Answer: \answerYes{}.
    \item[] Justification: Appendix~C reports the compute environment used for the experiments, including the CPU class, GPU type, GPU memory, and how CPU and GPU resources were used across inference, calibration, routing, and aggregation.

\item {\bf Code of ethics}
    \item[] Question: Does the research conducted in the paper conform, in every respect, with the NeurIPS Code of Ethics \url{https://neurips.cc/public/EthicsGuidelines}?
    \item[] Answer: \answerYes{}.
    \item[] Justification: The work evaluates uncertainty-quantification protocols for publicly available language models and benchmark datasets, does not introduce human-subject experiments, and does not release a high-risk model or scraped dataset. The study is intended to improve transparency and reliability reporting for deployed language-model systems.

\item {\bf Broader impacts}
    \item[] Question: Does the paper discuss both potential positive societal impacts and negative societal impacts of the work performed?
    \item[] Answer: \answerYes{}.
    \item[] Justification: The paper discusses implications for evaluation practice and reliability reporting for language-model systems. The main positive impact is improved transparency around conformal reliability claims; the main potential negative impact is that conformal guarantees could be overinterpreted if deployment assumptions and inference-configuration mismatch are not reported.

\item {\bf Safeguards}
    \item[] Question: Does the paper describe safeguards that have been put in place for responsible release of data or models that have a high risk for misuse (e.g., pre-trained language models, image generators, or scraped datasets)?
    \item[] Answer: \answerNA{}.
    \item[] Justification: The paper does not release a new pretrained model, image generator, scraped dataset, or other high-risk asset. The supplementary material contains the source code and processed experimental artifacts derived from existing public datasets and model outputs.

\item {\bf Licenses for existing assets}
    \item[] Question: Are the creators or original owners of assets (e.g., code, data, models), used in the paper, properly credited and are the license and terms of use explicitly mentioned and properly respected?
    \item[] Answer: \answerYes{}.
    \item[] Justification: The paper cites the original datasets and pretrained model checkpoints used in the experiments, and the supplementary material documents the corresponding sources, licenses, and terms of use where available. The released artifacts are derived from these public assets and are provided consistently with their usage terms.

\item {\bf New assets}
    \item[] Question: Are new assets introduced in the paper well documented and is the documentation provided alongside the assets?
    \item[] Answer: \answerYes{}.
    \item[] Justification: The supplementary material includes the code and processed experimental artifacts derived from public datasets and model outputs. The accompanying documentation describes the contents of the package, how the artifacts were generated, how to reproduce the reported results, and the intended use and limitations of the released materials.

\item {\bf Crowdsourcing and research with human subjects}
    \item[] Question: For crowdsourcing experiments and research with human subjects, does the paper include the full text of instructions given to participants and screenshots, if applicable, as well as details about compensation (if any)? 
    \item[] Answer: \answerNA{}.
    \item[] Justification: The paper does not involve crowdsourcing experiments or research with human subjects. All experiments are performed on existing benchmark datasets and pretrained language models.

\item {\bf Institutional review board (IRB) approvals or equivalent for research with human subjects}
    \item[] Question: Does the paper describe potential risks incurred by study participants, whether such risks were disclosed to the subjects, and whether Institutional Review Board (IRB) approvals (or an equivalent approval/review based on the requirements of your country or institution) were obtained?
    \item[] Answer: \answerNA{}.
    \item[] Justification: The paper does not involve crowdsourcing or research with human subjects, so IRB approval or equivalent human-subjects review is not applicable.

\item {\bf Declaration of LLM usage}
    \item[] Question: Does the paper describe the usage of LLMs if it is an important, original, or non-standard component of the core methods in this research? Note that if the LLM is used only for writing, editing, or formatting purposes and does \emph{not} impact the core methodology, scientific rigor, or originality of the research, declaration is not required.
    \item[] Answer: \answerYes{}.
    \item[] Justification: Language models are central experimental objects in this work. The paper describes the pretrained language models used, the inference configurations applied to them, and how the resulting predictive distributions are used to compute conformal scores; any LLM use for writing, editing, or formatting does not affect the scientific methodology or originality of the research.

\end{enumerate}

\end{document}